\documentclass{article}

\PassOptionsToPackage{numbers}{natbib}
\usepackage[main, final]{neurips_2026}

\usepackage[utf8]{inputenc} 
\usepackage[T1]{fontenc}    
\usepackage{url}            
\usepackage{booktabs}       
\usepackage{amsfonts}       
\usepackage{nicefrac}       
\usepackage{microtype}      
\usepackage{xcolor}         

\usepackage{graphicx}
\usepackage{amsmath, amsthm}
\usepackage{thm-restate}
\usepackage[most]{tcolorbox}
\usepackage{marginnote}
\usepackage{krish}

\usetikzlibrary{decorations.pathreplacing,calligraphy}
\allowdisplaybreaks

\newcommand{\cost}{\mathrm{Cost}}

\DeclareMathOperator{\objo}{\mathsf{MDV}}
\DeclareMathOperator{\objt}{\mathsf{ADV}}
\DeclareMathOperator{\objtx}{\objt_{\calX}}
\DeclareMathOperator{\objmix}{\mathsf{MixDV}}
\DeclareMathOperator{\objmixx}{\mathsf{MixDV}_{\calX}}
\newcommand{\objos}{\objo^s}
\DeclareMathOperator{\objth}{\mathsf{WDDV}}
\DeclareMathOperator{\objf}{\mathsf{ADDV}}
\DeclareMathOperator{\DASQ}{\smash \calD^{\mathsf{SSQ}}}
\DeclareMathOperator{\VarASQ}{\mathsf{VarSSQ}}
\DeclareMathOperator{\MSE}{\mathsf{MSE}}
\newcommand{\ardo}{\scriptscriptstyle \downarrow}
\newcommand{\arup}{\scriptscriptstyle \uparrow}
\newcommand{\wiup}{w^{\arup}_i}
\newcommand{\widown}{w_i^{\ardo}}
\newcommand{\wup}[1]{w_{#1}^{\arup}}
\newcommand{\wdown}[1]{w_{#1}^{\ardo}}
\newcommand{\varasq}{\VarASQ}
\newcommand{\dasq}{\DASQ}

\definecolor{amber}{rgb}{1.0, 0.49, 0.0}

\newcommand{\var}{\Var}
\newcommand{\bwh}{\widehat{\bw}}
\newcommand{\buh}{\widehat{\bu}}

\algdef{SE}[SUBALG]{Indent}{EndIndent}{}{\algorithmicend\ }
\algtext*{Indent}
\algtext*{EndIndent}

\newcommand{\algcomment}[1]{\textcolor{gray}{$\triangleright$ #1}}

\title{Inner Product Aware Quantization: Provably Fast, Accurate, and Adaptive Algorithms}

\author{%
  Nathan White \hypersetup{hidelinks}\thanks{Equal Contribution} \\
  University of Pennsylvania\\
  \texttt{nathanlw@cis.upenn.edu} \\
  \And
  Krish Singal \hypersetup{hidelinks}\footnotemark[1] \\
  University of Pennsylvania \\
  \texttt{ksingal@seas.upenn.edu} \\
}

\begin{document}

\maketitle

\addtocontents{toc}{\protect\setcounter{tocdepth}{-1}} 

\begin{abstract}
    Quantization is a fundamental tool used to compress datasets, neural network weights and reduce memory usage in a range of computational tasks. Many downstream applications of vector quantization perform inner products with arbitrary inputs. This motivates the study of \emph{inner product aware} quantization schemes that approximately preserve inner products with unseen vectors, in contrast to simply minimizing the mean-squared error. 
    
    In this work, we formulate objectives that capture natural desiderata and develop \emph{adaptive} and \emph{unbiased} quantization methods that approximately preserve inner products with worst-case and average-case inputs. An analysis of these objectives shows a tight connection with the well-studied notion of Adaptive Stochastic Quantization (ASQ).
    
    We develop provably fast exact and approximate algorithms for our objectives. Our theoretical results inspire efficient practical algorithms that perform well across a variety of workload distributions. They also lead to practical algorithms for standard ASQ which are 2--10$\times$ faster than prior state-of-the-art methods while maintaining quality. These theoretical and empirical results contribute towards making adaptive quantization techniques more efficient and tractable in practical settings. 
    
\end{abstract}

\section{Introduction}
Vector quantization is an incredibly important tool that is central to space and runtime optimization in a wide variety of computational and machine learning applications.
For example, quantization is used in dataset compression for vector search \cite{ge2013optimized, gao2025practical}, compression of large model weights and key-value caches \cite{sheng2023flexgen}, quantization-aware training \cite{micikevicius2017mixed}, and post-training quantization \cite{frantar2023optq, jeon2023frustratingly}.


Formally, we define vector quantization as follows.
Consider $w \in \R^d$ and let $Q \subset \R$ be a set of \emph{quantization values} such that $|Q| = s \ll d$.
Then, the vector $w$ can be \emph{quantized} via some \emph{rounding distribution} $\calD$ with $\text{Supp}(\calD) \subseteq Q^d$.
In this work, we consider \textit{unbiased} and \textit{adaptive} quantization schemes. A quantization scheme is unbiased if $\E_{\bwh \sim \calD}[\bwh] = w$,\footnote{We use boldface to denote random variables.} while an adaptive scheme is one where both $Q$ and $\calD$ may depend on $w$. In full generality, we wish to jointly optimize over the tuple $(Q, \calD)$ according to some suitable objective function. 

Much prior work on quantization considers schemes which are either non-adaptive/weakly adaptive or biased; for example, schemes such as QSGD \cite{alistarh2017qsgd} and NUQSGD \cite{ramezani2021nuqsgd} only use superficial properties of the vectors such as their norm or aspect ratio.
Other works \cite{fu2020don, faghri2020adaptive} are tailored to vectors which come from common, pre-specified distributions, but do not adapt to the specific input vector. For example, there is evidence that important and relevant data in ML applications follow LogNormal \cite{chmiel2020neural}, Normal \cite{banner2019post}, or sub-Weibull distributions \cite{vladimirova2018bayesian}. On the other hand, techniques such as Round-to-Nearest (RTN), which are biased and non-adaptive, are known to be outperformed by adaptive methods on post-training quantization \cite{nagel2020up}.  


A line of work on \emph{adaptive stochastic quantization} (ASQ) in \cite{ben2024optimal, ZLKALZ17, faghri2020adaptive} studies adaptive and unbiased quantization schemes, and notes substantial improvements over other types of quantization (see Figure 1 of \cite{ben2024optimal}).
In ASQ, the goal is to construct a quantization set which minimizes the mean-squared error
\[\MSE(\bwh,w) \triangleq \E \left[~\norm{\bwh - w}_2^2 \right].\]
The rounding distribution from which $\bwh$ is sampled is simple and natural: round each $w_i$ independently to either the closest quantization point larger than $w_i$ or smaller than $w_i$ in the unique unbiased manner. We refer to this rounding distribution as \textit{standard stochastic quantization}, and denote the resulting distribution $\dasq(w,Q)$;\footnote{When clear from context, we simply write $\dasq$.} we give a formal definition in Section \ref{sec: prelims}.
Standard stochastic quantization is a natural rounding distribution choice, as for each coordinate $i$, it minimizes $\var[\bwh_i]$ across all unbiased distributions.\footnote{See Theorem \ref{thm:objective-three-optimal-rounding} for a proof of this fact.}


While MSE is a natural objective to optimize, many downstream applications of quantization rely on more than simply low $\ell_2$ distance.
For instance, many modern applications take inner products of quantized vectors with other, potentially arbitrary vectors (as is the case in both neural network weight quantization and vector search).
Thus, the success of quantization in these applications heavily depends on its ability to preserve inner products well.

In this work, we introduce and study notions of \textit{inner product aware} quantization objectives which aim to approximately preserve inner products.
Because we consider unbiased quantization schemes where $\E[\bwh]=w$, linearity of expectation gives that for any vector $x\in \mathbb{R}^{d}$, $\E[\langle \bwh, x \rangle] = \langle w, x \rangle$.
So, minimizing the expected (squared) error of $\langle \bwh, x \rangle$ reduces to minimizing $\var[\langle \bwh,x \rangle]$. In our first objective, we minimize the variance over the worst-case choice of $x$.
\begin{definition}[Maximum Directional Variance ($\objo$)]
    For a vector $w\in \mathbb{R}^{d}$ and target quantization set size $s \in \N$, we define
    \[\objo(w,s) \triangleq \min_{\substack{ Q \subset \R \\ |Q| \leq s } } \max_{\substack{ x\in \mathbb{R}^{d} \\ \|x\|_2 \leq 1 }} \var_{\bwh\sim \dasq(w,Q)}[\langle \bwh,x \rangle]\]
\end{definition}

In practice, input vectors $x$ are often not worst-case, and instead come from some known (or estimated) distribution.
As such, we also consider optimizing to minimize the average variance over some (known) distribution of vectors $x$.
\begin{definition}[Average Directional Variance ($\objt$)]
    For a vector $w\in \mathbb{R}^{d}$, target quantization set size $s \in \N$, and input distribution $\calX$ over $\mathbb{R}^{d}$, define
    \begin{align}
        \objt_{\calX}(w, s) \triangleq \min_{\substack{ Q \subset \R \\ |Q| \leq s } } \E_{\bx \sim \calX} \left[\var_{\bwh\sim \dasq(w,Q)} \left[ \la \bwh, \bx \ra \right] \right]
    \end{align}
\end{definition}
We note that these objectives fix standard stochastic quantization as the rounding distribution and optimize over quantization sets.
While this is a natural and standard distribution choice, one may wonder if a better choice of rounding distribution could lead to lower inner product variance.
We show that this is unlikely to be a fruitful endeavor: for $\objo$, we prove that standard stochastic quantization is in fact the optimal rounding distribution (Theorem \ref{thm:objective-three-optimal-rounding}) and for $\objt$ we show that it is NP-Hard to compute the optimal distribution (Theorem \ref{thm:addv-np-hard}).

For both of these objectives, we design algorithms which are provably fast and obtain a solution within a $(1+\varepsilon)$ factor of the optimal objective cost, for a user-specified $\varepsilon>0$. See Section \ref{sec: theoretical-results} for a discussion of our algorithmic results.

We also implement and evaluate practical versions of these algorithms, and show they outperform previous algorithms for adaptive quantization. In particular, due to a close connection between $\objt$ and traditional ASQ (see Section \ref{sec: prelims}), we employ our techniques to develop algorithms that significantly outperform the current state-of-the-art \texttt{QUIVER} algorithm \cite{ben2024optimal}; see Figure \ref{fig:intro-plot}. In all plots and tables, we denote our novel algorithms with an asterisk$^{\ast}$. 

\begin{figure}[ht]
    \centering
    \includegraphics[width=0.49\linewidth]{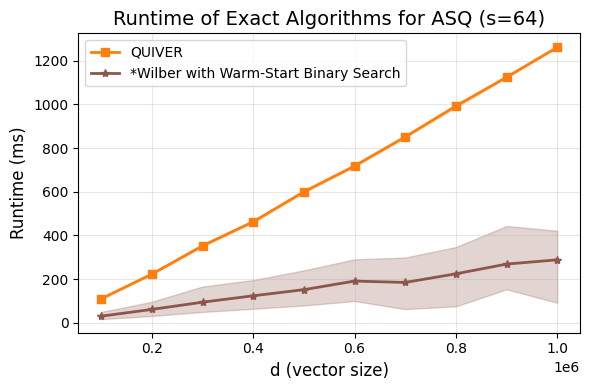}
    \includegraphics[width=0.49\linewidth]{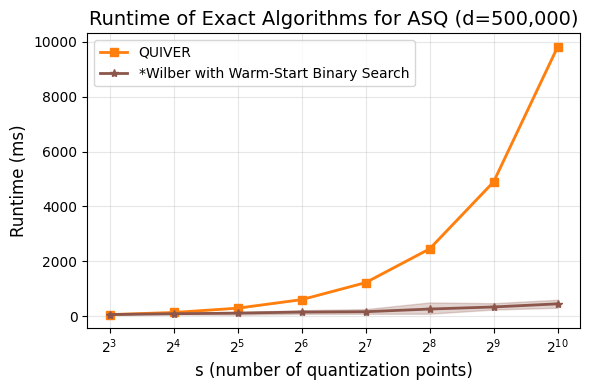}
    \caption{Figure comparing the runtime of the previous state-of-the-art algorithm for traditional ASQ with our faster algorithm (\texttt{Wilber} with Warm-Start Binary Search), on vectors drawn from $\text{LogNormal}(0,1)$.
        On the left, we fix $s=64$ and vary the vector size $d$, while on the right we fix $d=500{,}000$ and vary $s$. The shaded regions show the 10-90\% range of runtimes across sampled vectors.
    Figure \ref{fig:exact_runtime} is similar, with additional methods compared as well.}
    \label{fig:intro-plot}
\end{figure}

Since the exact algorithm is a core subroutine of the approximate \texttt{QUIVER} algorithm of \cite{ben2024optimal}, we immediately obtain speedups for this as well. More generally, any downstream task using \texttt{QUIVER} as a subroutine (e.g.,~\cite{ben2025better}) can be sped up using our implementation. See Section \ref{sec: empirical-results} for the details of our empirically fast algorithms.

Because adaptive methods come with an inherent trade-off (higher quality quantization at the cost of slower pre-processing time), one of the main impacts of our work is the use of theoretical insights to advance adaptive techniques towards efficiency in practical settings.

\subsection{Preliminaries}
\label{sec: prelims}

\paragraph{Standard Stochastic Quantization.}
For a given vector $w\in \mathbb{R}^{d}$ and quantization set $Q$, define the \textit{standard stochastic quantization} rounding distribution $\dasq(w,Q)$ to be the distribution which independently rounds each $w_i$ such that for $\bwh\sim \dasq(w,Q)$,
\[
\bwh_i =
\begin{cases}
    \wiup & \text{w.p.~} (w_i - \widown)/(\wiup - \widown) \\ 
    \widown & \text{w.p.~} (\wiup - w_i)/(\wiup - \widown) 
\end{cases}
\]
where $\wiup := \min \{q\in Q : q\geq w_i\}$ and $\widown := \max \{q\in Q : q\leq w_i\}$ are the values of $Q$ closest to $w_i$ from above and below, respectively.
Note that $\E[\bwh] = w$ and $\var[\bwh_i] = (\wiup - w_i)(w_i - \widown)$ for all $i\in [d]$.
For notational ease, we often write $\smash \VarASQ(w_i, Q) := \Var[\bwh_i] = (\wiup - w_i)(w_i - \widown)$.

\noindent\paragraph{Properties of $\objo$ and $\objt$.} Here, we make some simple but useful observations about the structure of the $\objo$ and $\objt$ objectives. We first introduce notational shorthands for the costs of each objective for a fixed quantization set. We write 
\begin{align*}
    \objo(w, Q) \triangleq \max_{\substack{ x\in \mathbb{R}^{d} \\ \|x\|_2 \leq 1 }} \var_{\bwh\sim \dasq(w,Q)}[\langle \bwh,x \rangle] \quad \text{ and } \quad \objt_{\calX}(w,Q) \triangleq \E_{\bx\sim \calX} \left[\var_{\bwh\sim \dasq(w,Q)}[\langle \bwh,\bx \rangle] \right]
\end{align*}
$\objo$ has a simple combinatorial structure; namely,
\begin{align*}
    \objo(w, Q) &= \max_{\substack{ x\in \mathbb{R}^{d} \\ \|x\|_2 \leq 1 } } \left[ \var_{\bwh\sim \dasq(w,Q)}[\langle \bwh,x \rangle] \right] \\
    &= \max_{\substack{ x\in \mathbb{R}^{d} \\ \|x\|_2 \leq 1 }} \left[ \sum_{i=1}^d x_i^2 \var_{\bwh\sim \dasq(w,Q)}[\bwh_i] \right] \tag{Independence of $\dasq$} \\
    &= \max_{j\in [d]} \varasq(w_j, Q).
\end{align*}
So, the worst-case input vector $x$ is the standard unit basis vector $e_i$ where $i$ is such that $\smash \VarASQ[w_i]$ is maximized. For $\objt$, consider a distribution $\calX$ supported on $\smash \R^{d}$ and let $\lambda_i\triangleq \smash \E_{\bx\sim \calX} [\bx_i^2]$.
Then,
\begin{align*}
    \objt_{\calX}(w, s) = \min_{\substack{ Q \subset \R \\ |Q| \leq s } } \E_{\bx \sim \calX} \left[\Var_{\bwh\sim \dasq(w,Q)} \left[ \la \bwh, \bx \ra \right] \right] &= \min_{\substack{ Q \subset \R \\ |Q| \leq s } } \E_{\bx \sim \calX} \left[ \sum_{i \in [d]} \bx_i^2 \cdot \VarASQ[w_i] \right] \\
    &= \min_{\substack{ Q \subset \R \\ |Q| \leq s } } \sum_{i \in [d]} \lambda_i \cdot \VarASQ[w_i]
\end{align*}
which is exactly the weighted MSE.
This connection between quantization for average-case inner product preservation and that of the (weighted) $\MSE$ objective also provides one possible theoretical explanation for the empirical success of Adaptive Stochastic Quantization for downstream applications in ML and vector search.


\section{Background and Motivation} \label{sec: motivation}

\begin{figure}[ht]
     \centering
     \includegraphics[width=0.49\linewidth]{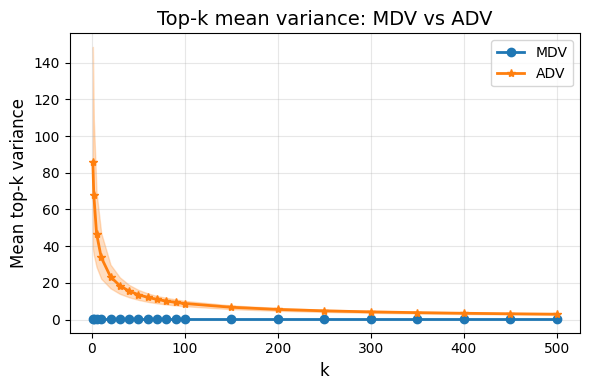} \includegraphics[width=0.49\linewidth]{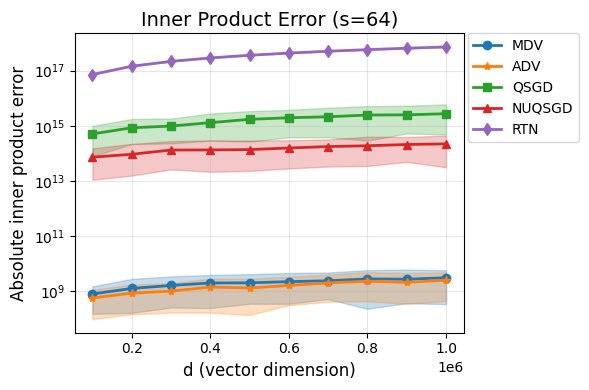}
     \caption{(a) Average variance $\VarASQ[w_i]$ of the worst $k$ coordinates on vectors $w$ drawn from $\text{LogNormal(0,1)}$  over 100 trials.
         (b) Average absolute inner product error on vectors $w$ drawn from a mixture of $16$ Gaussians with variance 10 and means separated by $10^5$.
     In both plots, the shaded regions show the 10--90\% range across trials.}

     \label{fig:motivation-plot}
\end{figure}


We now briefly discuss the merits of unbiased and adaptive quantization methods.
Methods such as NUQSGD and QSGD \cite{ramezani2021nuqsgd, alistarh2017qsgd} are unbiased but only very weakly adaptive; they choose the quantization set $Q$ only using information about the range of the entries in $w$.
Round-to-Nearest is both non-adaptive and biased, and there is evidence \cite{nagel2020up} that post-training quantization using RTN is significantly outperformed by adaptive techniques. We note that \cite{zandieh2025turboquant} study quantization under the objective of worst-case inner product preservation. However, they work in the non-adaptive setting, where the quantization set cannot depend on the vector being quantized. 

Figure 1 in \cite{ben2024optimal} illustrates the superiority of unbiased and adaptive techniques when minimizing the mean-squared error. As noted in the introduction, the worst-case variance $\smash \var[\la \bwh, x \ra]$ is realized by the standard basis vector $e_i$ where $i$ is such that $\VarASQ[w_i]$ is largest. Figure \ref{fig:motivation-plot}(a) shows that, empirically, optimizing $\objo$ actually minimizes the worst variances beyond the single worst coordinate, when compared against a solution to $\objt$. Indeed, on at least the 500 worst coordinates of a 1,000,000-dimensional vector (i.e.~the worst $0.05\%$), the optimal quantization set for $\objo$ has smaller variance than that for $\objt$. Figure \ref{fig:motivation-plot}(b) further gives evidence for the benefits of using unbiased and adaptive methods for worst-case inner product preservation.

In Table \ref{table: vector-search} (see Appendix \ref{sec: missing-proofs}), we present preliminary evidence that our methods have improved tail performance over the popular technique of product quantization (PQ) for certain types of vector search (namely maximum inner product search).
While $\objt$ obtains recall comparable to PQ across both the average and tail events for $\ell_2$ search, it achieves significantly higher recall on the worst 0.1\% of queries tested for maximum inner product search. 
These results, while limited in scope, suggest that explicitly optimizing for inner product preservation results in better quantization performance, as PQ is not tailored to inner product preservation.

\section{Theoretical Results} \label{sec: theoretical-results}
\subsection{Optimizing the Rounding Distribution}
Our main objectives $\objo$ and $\objt$ use a fixed rounding distribution (namely, standard stochastic quantization).
In this section, we outline the theoretical justification for the use of this distribution. We define new objectives for the problem of optimizing the rounding distribution given a fixed quantization set. 
In particular, we state Theorem \ref{thm:objective-three-optimal-rounding}, which shows that standard stochastic quantization is optimal for minimizing worst-case inner product variance, and Theorem \ref{thm:addv-np-hard}, which shows it is NP-hard to find the optimal distribution for minimizing average-case inner product variance.

Given a vector $w \in \R^d$ and quantization set $Q \subseteq \R$, we let $\Omega$ denote the set of rounding distributions $\calD$ such that $\text{Supp}(\calD) \subseteq Q^d$ and $\E_{\bwh \sim \calD}[\bwh] = w$.
As inner products are linear and $\calD \in \Omega$ is unbiased, it is easy to see that for any vector $x \in \R^d$ and $\calD\in \Omega$, $\E_{\bwh \sim \calD}[\la \bwh, x \ra] = \la w, x \ra$. We now introduce the new objectives. 


\begin{definition}[$\objth$, Worst-Case Distributional Directional Variance] \label{def: wddv}
    For a vector $w\in \mathbb{R}^{d}$ and a set $Q\subset \mathbb{R}$, we define
    \[\objth(w,Q)=\argmin_{\calD\in \Omega}\max_{\substack{x\in \mathbb{R}^{d} \\ \norm{x}_2 \leq 1 }}\var_{\bwh\sim \calD}[\langle \bwh,x \rangle].\]
\end{definition}

\begin{definition}[$\objf$, Average-Case Distributional Directional Variance] \label{def: addv}
    For a vector $w\in \mathbb{R}^{d}$, a set $Q\subset \mathbb{R}$, and an input distribution $\calX$ over $\mathbb{R}^{d}$, we define
    \[
        \objf_{\calX}(w,Q)=\argmin_{\calD\in \Omega}\E_{\bx\sim \calX} \left[\var_{\bwh\sim \calD}[\langle \bwh,\bx \rangle] \right]
    \]
\end{definition}

In Appendix \ref{appendix: wddv}, we prove Theorem \ref{thm:objective-three-optimal-rounding} -- that standard stochastic quantization is optimal for $\objth$. In contrast, in Appendix \ref{appendix: addv}, we prove Theorem \ref{thm:addv-np-hard} -- that finding the optimal rounding distribution for $\objf$ is NP-Hard.
\begin{restatable*}{theorem}{ASQoptimalthm} \label{thm:objective-three-optimal-rounding}
    For every vector $w \in \R^d$ and quantization set $Q \subset \R$, the rounding distribution $\DASQ(w, Q)$ minimizes $\objth(w, Q)$.
\end{restatable*}

\begin{restatable*}{theorem}{ADDVNPHard} \label{thm:addv-np-hard}
    There exists a family of instances, with $w = 0$, $Q = \{-1, 1\}$, and Gaussian input distribution $\calX$, on which computing $\objf_{\calX}(w, Q)$ is NP-Hard.
\end{restatable*}

The proof of Theorem  \ref{thm:objective-three-optimal-rounding} proceeds by transforming an arbitrary distribution into $\dasq$ and arguing that this transformation cannot have increased the worst-case variance.
Theorem \ref{thm:addv-np-hard} follows by a reduction from Maximum Cut on unweighted graphs.
These results further motivate the use of standard stochastic quantization as the fixed rounding distribution in $\objo$ and $\objt$.


\subsection{Algorithms for $\objt$}

\subsubsection{Exact Algorithms}

In the recent literature, \cite{ZLKALZ17} first defines the standard ASQ problem and presents a simple dynamic programming algorithm with runtime $O(d^2 s)$ and space $O(d^2)$. \cite{ben2024optimal} then present an algorithm named \texttt{QUIVER}, which couples the dynamic programming solution with efficient matrix-searching techniques to solve the standard ASQ problem. This improved algorithm has runtime $O(d \log d + d s)$ and uses $O(ds)$ space to output the optimal quantization set.

In a different direction, the survey paper of \cite{GKMNSS17} outlines many results from the $1$-dimensional $k$-means clustering literature.
Perhaps intuitively, as the $k$-means cost is a lower bound on the MSE of any quantization scheme,\footnote{In fact, if one allows biased quantization, the optimal scheme is to compute the $k$-means clustering and assign each element of $w$ to the nearest cluster.} $k$-means has many similarities with minimizing the MSE of an (unbiased) quantization scheme.
Indeed, the algorithms mentioned in \cite{GKMNSS17} can be easily modified to give algorithms for the standard ASQ problem. 
For example, in the context of 1-dimensional $k$-means, the dynamic programming and matrix-searching algorithm of \cite{ben2024optimal} (see Corollary \ref{cor: exact-weighted-mse}) first appeared in \cite{Wu91}.




More generally, \cite{GKMNSS17} survey a spectrum of theoretical results for $k$-means that complete the state-of-the-art optimality profile in differing regimes of $k$ and $d$.
Importantly for our work, these algorithms can all be applied to any problem whose structure satisfies the Concave Monge property (Definition \ref{def: concave-monge-property}), including but not limited to $\objt$ and $1$-dimensional $k$-median clustering in Euclidean space.

A key primitive in these algorithms is finding the shortest path with length (i.e.~the number of edges) $k$ in a directed acyclic graph (DAG) whose weights satisfy the Concave Monge property.
For this problem, \cite{AST93} gives (i) an algorithm with runtime $O(d \sqrt{k \log d})$, and (ii) an algorithm with runtime $O(d \log \Delta)$ where $\Delta$ is the difference between the min and max weight.
When $k = \Omega(\log d)$, \cite{Schieber98} gives an algorithm with runtime $d \smash 2^{O(\sqrt{\log \log d \log k})}$. At a high level, these algorithms work by solving the following \emph{regularized} version of the problem:
\begin{align*}
    \min_{\{ i_1, \ldots, i_{\ell} \} \subseteq [d]} \sum_{j=1}^{\ell-1} C[i_j, i_{j+1}] + \tau \ell
\end{align*}
where $C[i,j]$ is the weight of edge $(i,j)$.
That is, we remove the hard constraint of length $k$ edges, and instead add a penalty term to the cost function of $\tau$ per edge.
By iterating through values of $\tau$ appropriately, we can then solve the unregularized (i.e.~length-constrained) variant.
In solving this regularized problem, \cite{AST93,Schieber98} use ideas and results from \cite{Wilber98}.




These algorithms can be directly applied to $\objt$ since its structure is Concave Monge (see Lemma \ref{lem: C-totally-monotone}).
As we show in Section \ref{sec: empirical-results}, an out-of-the-box implementation by \cite{GKMNSS17} of algorithm (ii) from above outperforms \texttt{QUIVER} from \cite{ben2024optimal}.

\subsubsection{Approximation Algorithms}

An additional result of \cite{ben2024optimal} is a bi-criteria approximation algorithm for the \emph{unweighted} MSE objective. In particular, their algorithm returns a quantization set of size $2s-2$ whose quality is an additive approximation to that of the optimal. It is easy to see, following their analysis, that the same algorithm, analyzed under the \emph{weighted} MSE objective, results in the following weak additive approximation guarantee.

\begin{corollary}[Lemma 6.1 of \cite{ben2024optimal}]\label{lem: additive-approx-weighted-mse-main}
    There exists an algorithm such that for any input distribution $\calX$ and $m, s \in \N$, it has runtime $O(d + ms)$ and returns a quantization set $Q$ such that $|Q| = 2s-2$ and $\objt_{\calX}(w, Q) \leq \objt_{\calX}(w, s) + \smash \sum_{i=1}^{d} \lambda_i \cdot \Delta^2/m^2$. Here, $\Delta := \max_{i, j \in [d]} w_i - w_j$.
\end{corollary}

We develop approximation algorithms for optimizing $\objt_{\calX}$ (weighted MSE) with substantially stronger approximation guarantees than that of Corollary \ref{lem: additive-approx-weighted-mse-main}. Furthermore, our algorithms are not bi-criteria. We first give an $s$-approximation,\footnote{This first approximation algorithm underpins our empirically fast algorithms (which have similar provable approximation guarantees; see Section \ref{subsec: asq-approx-algs}).} then improve it to a $(1+\epsilon)$-approximation. 

\begin{restatable*}{theorem}{ADVApproxAlg} \label{thm: adv-approx-alg}
    There exists an algorithm that for any given vector $w \in \R^d$, target quantization set size $s \in \N$, input distribution $\calX$, and $\epsilon > 0$, returns a quantization set $Q$ such that $|Q| = s$ and $\objt_{\calX}(w, Q) \leq (1+\epsilon) \cdot \objt_{\calX}(w, s)$. The runtime of the algorithm is $O(d \log( d/\epsilon) )$.
\end{restatable*}

The first $s$-approximation algorithm works by exactly solving a suitably defined objective $\objmix$ whose optimal value is an $s$-approximation to that of $\objt$. Crucially, the algorithm uses the fact that the objective has a sortedness (Definition \ref{def: sorted-matrix}) property, which allows for fast $k$-selection algorithms to be used (see Appendix \ref{sec: technical-prelims}).
The $(1+\epsilon)$-approximation algorithm searches for the Lagrangian multiplier $\tau$ of the regularized shortest path on DAG problem described earlier.
Importantly, it is \emph{warm-started} using the solution from the $s$-approximation algorithm, which bounds both the range of this search and the precision to which $\tau$ must be found. If no value of $\tau$ at this precision yields a path of the desired length, the optimal paths on either side are combined using the Concave Monge property. See Appendix \ref{sec: adv} for details.


Along the way, we observe that the main ideas from solutions to $\objo$ and $\objt_{\calX}$ (in particular, building data-dependent coresets) can be used to improve the bi-criteria approximation algorithm for quantization under the \emph{unweighted} MSE objective. Our algorithm reports a quantization set of size $2s-2$ with objective value at most $(1+\epsilon)$ times optimal  in time $\smash O(d \log s+ s \sqrt{d/\epsilon} \log (d/\epsilon))$; see Theorem \ref{thm: obj2-2s-approx} in Appendix \ref{sec: approx-alg-MSE} for details.

\subsection{Algorithms for $\objo$}
An ideal algorithm for $\objo$ would efficiently return an optimal quantization set. Unfortunately, we prove in Lemma \ref{lem: mdv-irrational-solutions} that the optimal quantization set may contain irrational values (even when the input vector is integer), making such an algorithm impossible under standard finite bit representation (i.e.~floating point) of any precision. One must then settle for algorithms which are optimal up to some precision $\varepsilon$. 
\begin{restatable*}{theorem}{MDVApproxAlg} \label{thm:obj1-full-algo}
    There exists a randomized algorithm that for any given vector $w \in \R^d$, target quantization set size $s \in \N$, and $\epsilon > 0$, returns a quantization set $Q$ such that $|Q| = s$ and $\objo(w, Q) \leq (1+\epsilon) \cdot \objo(w, s)$. The runtime of the algorithm is $O(d \log( s/\epsilon) )$ with probability $0.99$.
\end{restatable*}

Note that this algorithm does \textit{not} require sorting, and for moderate values of $s,\varepsilon$ is faster than the $O(d\log d)$ time required to sort the input vector (as in \cite{ben2024optimal}).
Moreover, the dependence on $1/\varepsilon$ is logarithmic, i.e.~linear in the number of bits of precision one desires.

This randomized algorithm works by constructing a small \textit{coreset} of values, on which we can solve while still obtaining a close-to-optimal solution for the entire vector.
In particular, we show how to construct a subset $y\subseteq w$ of at most $O(s/\sqrt{\varepsilon})$ elements of the input vector such that $\objo(y,s)$ is within a $(1+\varepsilon)$ factor of $\objo(w,s)$.
Thus, we construct such a subset and then run a slower but still near-linear time algorithm on this coreset.

To construct this coreset, we prove several structural lemmas about $\objo$ and employ faster approximation algorithms for $k$-center clustering (see Appendix \ref{sec: technical-prelims}).
To solve on the coreset, we first find a coarse approximation by restricting the quantization points to be values from the input vector, and then use this estimate to ``warm-start'' a binary search for the optimal maximal variance (again, over the coreset).
The most technically difficult step is finding this coarse approximation: it bears some similarities to the fast matrix search algorithms employed to optimize for $\objt$, but there are a number of added difficulties specific to the max objective of $\objo$.
See Appendix \ref{sec: mdv} for the full details of the algorithm and the proof of Theorem \ref{thm:obj1-full-algo}.

\section{Empirical Results} \label{sec: empirical-results}

\subsection{Empirical Setup} \label{sec: experimental-setup}
We implement a number of algorithms in optimized \texttt{Cython} code\footnote{Our code is availiable here: \url{https://github.com/nathanllww/Inner-Product-Aware-Quantization}}.
We also modify and extend the code of \cite{GKMNSS17} in our exact algorithms for ADV/weighted ASQ.
Our experiments were run on a Mac Mini consumer desktop computer with an Apple M4 processor and 16GB of RAM.
In the supplementary information, we include empirical results from 33 different data distributions; each is qualitatively similar to the results presented in this section.

\subsection{Faster Exact Algorithms for ASQ}
Our first contribution is significantly faster algorithms for ASQ\footnote{Our algorithms, like \texttt{QUIVER}, support weights, as needed to solve $\objt$.} which, in some regimes, obtain a $10\times$ speedup over the previous state-of-the-art accelerated \texttt{QUIVER} from \cite{ben2024optimal}.
Our implementations can be used as a drop-in replacement for \texttt{QUIVER}, and thus speed up any algorithm using \texttt{QUIVER}; in particular, we give faster approximation algorithms (see Section \ref{sec:prac-algs}).

Our first insight to make these speedups is that the fast exact 1D $k$-means algorithms presented in \cite{GKMNSS17} can be modified to solve ASQ; see Section \ref{sec: ADV-exact-algorithms} for the full details.
This strategy already gives algorithms which are faster than \texttt{QUIVER}; namely, the \texttt{Wilber} algorithm from \cite{GKMNSS17}, when run with interpolation search and modified to support the ASQ objective, is faster than \texttt{QUIVER} in many regimes (see Figure \ref{fig:exact_runtime}).

To improve the performance further, we exploit the inner workings of \texttt{Wilber} to \textit{warm-start} the algorithm with an approximation.
This allows the binary search of \texttt{Wilber} to converge much faster, leading to additional speedups.
In particular, \texttt{Wilber} works by solving a shortest path on a directed acyclic graph (DAG), whose weights are determined in part by a Lagrangian multiplier $\tau$.
Unfortunately, \cite{GKMNSS17} show that the desired value $\tau$ is the difference between the optimal ASQ cost with $s$ quantization points and the cost with $s+1$ quantization points, so it must be found via binary search.

To speed up this search, we use a fast yet robust approximation algorithm (\texttt{MixApprox}, see Section \ref{sec:prac-algs}) to estimate the value $\tau$, and only search in the range around this estimate.
Letting $\tau'$ be the difference in costs returned by \texttt{MixApprox} when run with $s$ and $s+1$ quantization points, we only search for $\tau \in [\tau'/2, 2\tau']$ (reverting to searching all possible values if nothing is found).
Since \texttt{MixApprox} is an accurate approximation algorithm across a range of distributions, $\tau$ is found very quickly, leading to a large speedup. We plot the runtime of our method against \texttt{QUIVER} and vanilla \texttt{Wilber} in Figure \ref{fig:exact_runtime}.

\begin{figure}[ht]
    \centering
    \includegraphics[width=0.49\linewidth]{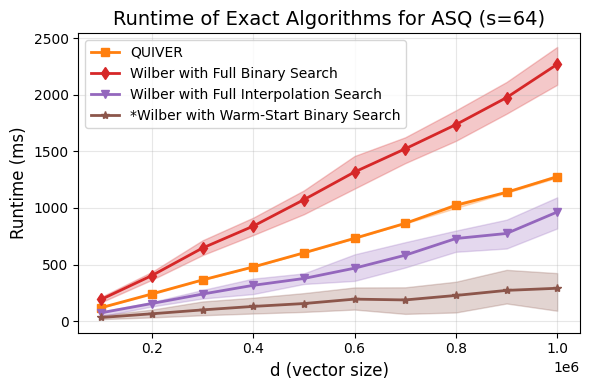} \includegraphics[width=0.49\linewidth]{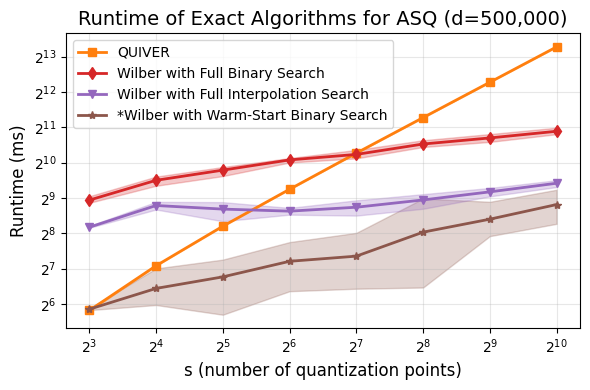}
    \caption{Runtime of exact algorithms for ASQ on vectors $w$ of dimension $d$ drawn from $\text{LogNormal}(0,1)$. The shaded regions show the 10--90\% range of runtimes across draws of $w$. Orange denotes accelerated \texttt{QUIVER}; red denotes \texttt{Wilber} with binary search over the full range of possible $\tau$; purple denotes \texttt{Wilber} with the interpolation search method of \cite{GKMNSS17}; and brown denotes our accelerated search method.}
    \label{fig:exact_runtime}
\end{figure}

\subsection{Faster Approximation Algorithms for ASQ} \label{subsec: asq-approx-algs}
The approximate \texttt{QUIVER} algorithm of \cite{ben2024optimal} operates by first placing $m$ evenly spaced markers between the minimum and maximum entries of the vector $w$, and then finds the optimal $s$ of these $m$ discretized points to include as quantization values.
Unsurprisingly, then, using our faster exact algorithms to find the optimal $s$ quantization values, we are able to improve the performance\footnote{Due to the special structure of the calls approximate \texttt{QUIVER} makes to the exact algorithm, we use a slight variant of our fast exact algorithm; see Section \ref{sec:prac-algs} for details.  The overall structure of the algorithm is the same.} of the approximate \texttt{QUIVER} algorithm as well; see Figure \ref{fig:improved-quiver-approx}.

\begin{figure}[ht]
    \includegraphics[width=0.49\linewidth]{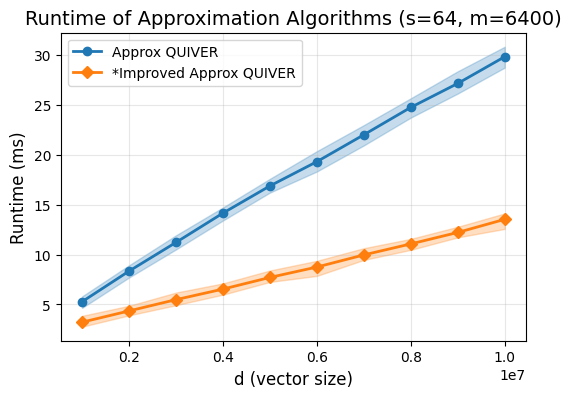} \includegraphics[width=0.49\linewidth]{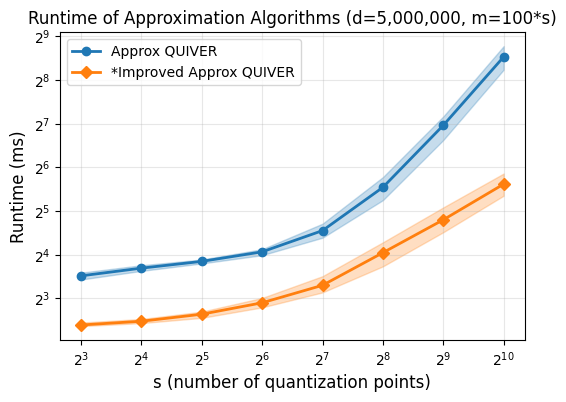}
    \caption{Runtime comparison of approximate \texttt{QUIVER} and our implementation using faster exact subroutines on vectors $w$ drawn from $\text{LogNormal}(0,1)$, with $m=100s$ uniformly spaced discretization points. The shaded regions show the 10--90\% range of runtimes across draws of $w$.}
    \label{fig:improved-quiver-approx}

\end{figure}

\subsection{Algorithms for \texorpdfstring{$\objo$}{MDV}}
We also give a fast practical algorithm for $\objo$, inspired by our theoretical algorithms; see Section \ref{sec:mdv-practical} for details.
In the first step when run on vector $w$, the algorithm constructs a coreset in a manner similar to approximate \texttt{QUIVER}: it divides the range $[\min_i w_i, \max_i w_i]$ into $10s$ equally sized buckets, and constructs the coreset by taking the minimum and maximum value of $w$ from each bucket.
The algorithm then uses binary search to find the worst-case variance, by ``guessing and checking''. 
To do the check, we utilize an algorithm which (quickly) computes the minimum quantization set size required to obtain a target worst-case variance of $v$.

However, if any bucket is larger than $\sqrt{\varepsilon v_{\max}}$, where $v_{\max}$ is the current maximum of the search range, the algorithm \textit{subdivides} the buckets into smaller sizes until all have length at most $\sqrt{\varepsilon v_{\max}}$.
This subdivision step is the crucial modification that allows for a guaranteed $(1+\varepsilon)$-approximation, even on heavily skewed data, but the adaptive nature also keeps the algorithm very fast.
In a certain sense, the algorithm is adaptively finding the \textit{instance ideal} bucket size, keeping it as fast as possible for a given desired accuracy. The proof of the $(1+\varepsilon)$-approximation can be found in Lemma \ref{lem:mdv-practical-qual}.

\begin{figure}[ht]
    \centering
    \includegraphics[width=0.49\linewidth]{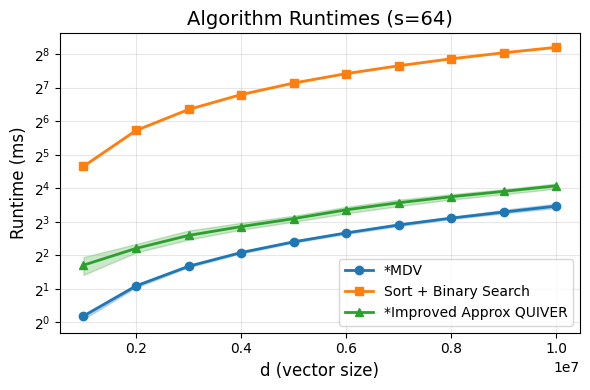}
    \includegraphics[width=0.49\linewidth]{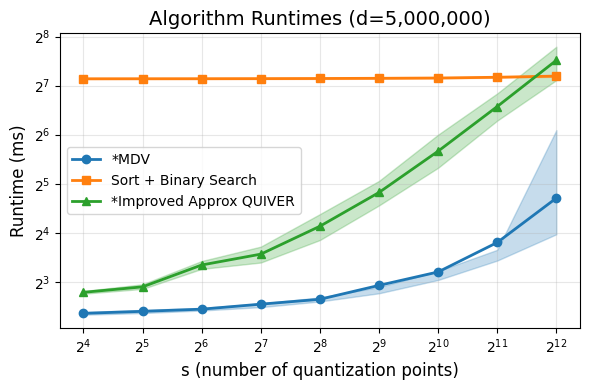}
    \caption{Runtime of algorithms for $\objo$ compared against Improved Approx \texttt{QUIVER} as a baseline (although it does not provide a good approximation to $\objo$), on vectors drawn from $\text{LogNormal}(0,1)$. The shaded regions show the 10--90\% range of runtimes across draws of the input vector.}
    \label{fig:mdv-runtime}
\end{figure}

In Figure \ref{fig:mdv-runtime}, the algorithm labeled ``$\objo$'' is this fast algorithm, run with $\varepsilon=0.01$, while the ``Sort + Binary Search'' method sorts and then uses binary search to find the smallest $v$ for which $s$ quantization points can achieve worst-case variance $v$. Improved Approx \texttt{QUIVER} is always run with quality parameter $m=100s$.
Note that this implementation of Improved Approx \texttt{QUIVER}, in contrast to the implementation of Approx \texttt{QUIVER} provided by \cite{ben2024optimal} and our implementation shown in Figure \ref{fig:improved-quiver-approx}, does not require the input vector to be sorted.

\subsection{LLM Quantization}
\label{sec:llm-expers}

We evaluate our quantization schemes against several popular options on the task of LLM quantization.
In particular, we quantize the MLP weights of the open-weight model Qwen3.5-4B \cite{qwen35}; the full details of the method used for quantizing can be found in Appendix \ref{sec: llm-setup}.

\newcommand{\mdo}{W_{\text{down}}}
\newcommand{\mgate}{W_{\text{gate}}}
\newcommand{\mup}{W_{\text{up}}}

Given input vector $x_{\text{in}}$, each MLP layer computes
\[
    x_{\text{out}} = \mdo \left( f_{\text{act}} (\mgate x_{\text{in}}) \otimes \mup x_{\text{in}} \right)
\]
where $\mdo, \mgate, \mup$ are the (learned) weight vectors, $\otimes$ denotes element-wise multiplication, and $f_{\text{act}}$ is a nonlinear activation function applied element-wise (in this case, SiLU \cite{qwen35}).

Note each weight matrix participates in only matrix-vector products, which should be thought of as inner products of an arbitrary query vector with each row of the weight matrix. It is natural to apply our vector quantization scheme to each row individually; see Table \ref{table: llm-row}.

\begin{table}[ht]
    \caption{Per-row quantization of the MLP weights of Qwen3.5-4B with $s=32$. KL divergence, top-1 agreement, and perplexity are measured on WikiText-2, and KL and top-1 agreement are computed relative to the unquantized model. ARC Acc. is accuracy on the ARC-Challenge benchmark of basic science questions, and ARC agree is agreement with the unquantized model's answers. $\objt$ and ASQ round to nearest rather than stochastically; TurboQuant is from \cite{zandieh2025turboquant}. $\ast$ denotes our method.}
    \label{table: llm-row}
    \centering
    \begin{tabular}{lccccc}
        \toprule
        \textbf{Method} & \textbf{KL ($\times 10^{-2}$)} & \textbf{Top-1 (\%)} & \textbf{PPL} & \textbf{ARC Acc. (\%)} & \textbf{ARC Agree (\%)} \\
        \midrule
        Unquantized   & 0 & 100 & 9.58 & 54.86 & 100 \\
        \midrule
        $\objt^{\ast}$ & \textbf{1.17} & \textbf{95.46} & 9.63 & 54.44 & \textbf{96.33} \\
        ASQ           & 1.36 & 95.11 & \textbf{9.58} & 54.69 & 96.08 \\
        $k$-means     & 1.37 & 95.10 & 9.65 & 53.92 & 95.48 \\
        TurboQuant    & 1.37 & 95.00 & 9.62 & \textbf{55.38} & 95.31 \\
        Uniform       & 2.89 & 92.56 & 9.75 & 54.10 & 94.37 \\
        \bottomrule
    \end{tabular}
\end{table}

However, in practice one may prefer to construct one quantization set per weight matrix to lower the codebook construction time and lower the average bits per parameter. We also test this alternative method. Despite the fact that our methods are designed with vector quantization in mind, we find that they yield strong results in the matrix quantization setting as well (see Table \ref{table: llm-matrix}).

\begin{table}[ht]
    \caption{The same evaluation as Table \ref{table: llm-row}, with one quantization set per weight matrix rather than per row.}
    \label{table: llm-matrix}
    \centering
    \begin{tabular}{lccccc}
        \toprule
        \textbf{Method} & \textbf{KL ($\times 10^{-2}$)} & \textbf{Top-1 (\%)} & \textbf{PPL} & \textbf{ARC Acc. (\%)} & \textbf{ARC Agree (\%)} \\
        \midrule
        Unquantized   & 0 & 100 & 9.58 & 54.86 & 100 \\
        \midrule
        $\objt^{\ast}$ & \textbf{2.29} & \textbf{93.51} & 9.63 & 55.55 & 93.86 \\
        ASQ           & 2.50 & 93.25 & 9.74 & 55.63 & \textbf{95.14} \\
        $k$-means     & 6.37 & 89.86 & 9.78 & 54.69 & 94.88 \\
        TurboQuant    & 6.06 & 90.33 & \textbf{9.25} & \textbf{56.40} & 93.60 \\
        Uniform       & 185.32 & 43.64 & 58.62 & 33.36 & 54.35 \\
        \bottomrule
    \end{tabular}
\end{table}



\section{Discussion} \label{sec: discussion}



\paragraph{Limitations.} We now discuss a few important limitations of our work.

\begin{itemize}
    \item Although $\objo$ outperforms $\objt$ on worst-case variance, our experimentation revealed that it is outperformed by $\objt$ on the worst $1\%$ of tail events for many distributions. This perhaps motivates future work studying a \emph{hybrid} objective that optimizes for the worst $1\%$ performance over a given input distribution $\calX$. 
    \item  Our largest speed improvements over the prior state-of-the-art \texttt{QUIVER} algorithm come when the number of quantization points $s$ is large.  It is natural to consider whether we can adapt our algorithms to achieve similar speedup over \texttt{QUIVER} for small values of $s$ as well.
\end{itemize}

\paragraph{Future Work.} 
We conclude with some interesting open questions that are natural extensions of this work:
\begin{itemize}
    \item Theorem \ref{thm:addv-np-hard} shows that computing the optimal rounding distribution for the average-case problem is NP-Hard. Do there exist efficient algorithms to find approximately optimal rounding distributions?
    \item Applying this notion of adaptive stochastic quantization to vector search requires the storage of a quantization set for every vector $w$ in the dataset.
        Can our ideas and algorithms be adapted to (potentially dynamic) \emph{multi-vector} quantization settings?
\end{itemize}

\section*{Acknowledgments}
NW is supported by the Air Force Office of Scientific Research under award number FA9550-23-F-0014 in the amount of \$140,600. KS is supported by the National Science Foundation (NSF) under Grant No. CCF-2337993. The authors thank the anonymous reviewers whose suggestions improved the quality and presentation of this work.

\newpage

\bibliographystyle{alpha}
\bibliography{krish}

\addtocontents{toc}{\protect\setcounter{tocdepth}{1}}  

\appendix
\tableofcontents
\section{Technical Preliminaries} \label{sec: technical-prelims}






\subsection{Matrix Properties and Algorithms}

\begin{definition}[Sorted Matrix] \label{def: sorted-matrix}
    We call a matrix $A \in \R^{d \times d}$ \emph{sorted} if all its rows are sorted in non-decreasing order while all columns are in non-increasing order.
\end{definition}

\begin{lemma}[\texttt{Sorted-Selection} Algorithm, Theorem 1 of \cite{FJ84} and Corollary 6.2 of \cite{mirzaian1985selection}] \label{lem: k-selection}
    Let $A \in \R^{d \times d}$ be a sorted matrix. Then, there exists an algorithm that returns the $k$th smallest element (i.e., performs $k$-selection) in $O(d)$ time. 
\end{lemma}

Importantly, the algorithms of Lemma \ref{lem: k-selection} operate in the query access model, where they are allowed $O(1)$ time access to any entry of $A$. The algorithms from \cite{FJ84, mirzaian1985selection} achieve $O(d)$ time for $k$-selection in row-and-column sorted $d \times d$ matrices by utilizing a divide-and-conquer strategy that systematically restricts the search space to a narrow ``staircase'' contour of $O(d)$ candidate elements. \cite{FJ84} accomplishes this through recursive matrix bisection into $2 \times 2$ submatrices, extracting representatives to compute bounds that strictly contain the $k$-th element, and then applying a linear-time selection subroutine to the remaining candidates along the monotonic frontier. \cite{mirzaian1985selection} streamlines this recursive procedure and greatly simplifies the algorithm and analysis. 
\begin{definition}[Monotone Matrix] \label{def: monotone-matrix}
    Let matrix $A \in \R^{d \times d}$ and define $c(i) := \min \{ j \in [d] : A[i,j] = \min_{k \in [d]} A[i,k] \}$. Matrix $A$ is monotone if for all $i \in [d]$, $c(i) \leq c(i+1)$. 
\end{definition}

\begin{definition}[Totally Monotone Matrix] \label{def: totally-monotone-matrix}
    Matrix $A \in \R^{d \times d}$ is \emph{totally} monotone if every submatrix of $A$ is monotone. Equivalently, $A$ is totally monotone if for $i < i'$ and $j < j'$, $A[i,j] > A[i, j'] \implies A[i', j] > A[i', j']$. 
\end{definition}

\begin{definition}[Concave Monge Property] \label{def: concave-monge-property}
    Matrix $A \in \R^{d \times d}$ satisfies the concave Monge property if for any $ 1 \leq j < k < d$, $A[j, k] + A[j+1, k+1] \leq A[j, k+1] + A[j+1, k]$.
\end{definition}

\begin{definition}[Quadrangle Inequality] \label{def: quadrangle-inequality}
    Matrix $A \in \R^{d \times d}$ satisfies the quadrangle inequality if for any $a \leq b \leq c \leq f$, $A[a, c] + A[b, f] \leq A[a, f] + A[b, c]$.
\end{definition}

It is clear that any matrix satisfying the Quadrangle Inequality (i) satisfies the Concave Monge property, (ii) is totally monotone, and (iii) is monotone. 

\begin{lemma}[\texttt{SMAWK} Algorithm, Theorem 4.3 of \cite{aggarwal1986geometric}] \label{lem: SMAWK}
    Let $A \in \R^{d \times d}$ be a totally monotone matrix. Then, there exists an algorithm that in $O(d)$ time returns $j(i) := \min \{ j \in [d] : A[i,j] = \max_{k \in [d]} A[i,k] \}$ for all rows $i \in [d]$.
\end{lemma}

The \texttt{SMAWK} algorithm leverages the total monotonicity property to systematically eliminate redundant columns, thereby shrinking the matrix width. Following this, the algorithm recurses on the even-indexed rows of the reduced matrix to find their respective minima. In the final step, it uses the computed minima of the even rows to tightly constrain the search space for the odd-indexed rows, locating the remaining minima in strictly linear time. 

\subsection{Clustering Algorithms}

Many algorithms in Appendices \ref{sec: approx-alg-MSE} and \ref{sec: mdv} make use of exact and/or approximate clustering algorithms as subroutines. These clustering algorithms are often used to give \emph{data-dependent coresets} of the original instance, upon which further processing is done. 

In particular, we utilize an algorithm for $s$-center clustering in Euclidean space. Here, an input consists of $d$ vectors $v_1, \ldots, v_d \in \R^n$ and the task is to return an $s$-clustering $\calC = (C_1, \ldots, C_s)$ such that $C_1 \sqcup \ldots \sqcup C_s = \{ v_1, \ldots, v_d \}$ that minimizes the maximum radius of any cluster.
\begin{align*}
    \text{Radius}(\calC) := \frac{1}{2} \cdot \max_{i \in [s]} \max_{v_j, v_k \in C_i} \norm{v_j - v_k}_2
\end{align*}
 The following $2$-approximate $s$-center clustering algorithm of \cite{FG88} is used often. 

\begin{lemma}[\texttt{s-Center-Clustering} Approximation Algorithm, Theorem 4.1 of \cite{FG88}] \label{lem: s-center-clustering-alg}
    Given $d$ vectors $v_1, \ldots, v_d \in \R^n$ and a target number of clusters $s \in \N$, there exists an algorithm with runtime $O(d \log s)$ that returns an $s$-clustering $\calC$ such that $$\mathrm{Radius}(\calC) \leq 2 \cdot \min_{\substack{ \calA = (A_1, \ldots, A_s) \\ A_1 \sqcup \ldots \sqcup A_s = \{ v_1, \ldots, v_d \}}} \mathrm{Radius}(\calA)$$
\end{lemma}

The $\texttt{s-Center-Clustering}$ algorithm of Lemma \ref{lem: s-center-clustering-alg} works similarly to Gonzalez's algorithm \cite{gonzalez1985clustering} by finding an independent set in a suitable implicitly defined geometric graph. This graph, however, is defined over \emph{boxes} of points rather than the points themselves. This ultimately allows for a more efficient data structure that can be used to construct the independent set. The work is technical and involved, so for further details, consult the original paper \cite{FG88}.

Hidden in the lemma statement is that the runtime of the above algorithm has exponential dependence on the dimension of the space $n$. In this work, however, we only utilize this algorithm on 1-dimensional instances. 

\begin{remark}
    The result of Lemma \ref{lem: s-center-clustering-alg} is largely theoretical due to large constant factors and overall complexity of the data structures involved in the algorithm. As such, any algorithm using it as a subroutine is most-probably impractical to implement. However, empirically implemented algorithms are heavily inspired by the theoretical results and are outlined in each relevant appendix. 
\end{remark}

\newcommand{\dwasq}{\calD^{\text{asq}}}
\newcommand{\uqd}{Q^{\scriptscriptstyle \downarrow}}
\newcommand{\oqd}{Q^{\scriptscriptstyle \uparrow}}
\newcommand{\uqdcom}{\uqd_{\text{COM}}}
\newcommand{\oqdcom}{\oqd_{\text{COM}}}
\newcommand{\uup}{u^{\scriptscriptstyle \uparrow}}
\newcommand{\udown}{u^{\scriptscriptstyle \downarrow}}

\section{\texorpdfstring{Standard Stochastic Quantization is Optimal for $\objth$}{Adaptive Stochastic Quantization is Optimal for WDDV}} \label{appendix: wddv}

In this section, we prove that the standard stochastic quantization rounding distribution is optimal for $\objth$ (Definition \ref{def: wddv}). Recall that $\dasq(w,Q)$ is the distribution that rounds each point $w_i$ to either $\wiup(Q)$ or $\widown(Q)$ independently, with probabilities chosen to ensure that the rounding is unbiased.

\ASQoptimalthm

To prove Theorem \ref{thm:objective-three-optimal-rounding}, we first show the following key lemma.

\begin{lemma} \label{lem:min-variance-independent-rounding}
    Consider $u \in \mathbb{R}$ and quantization set $Q\subseteq \mathbb{R}$.
    Let $\udown := \max \{q \in Q : q \leq u\}$, $\uup := \min \{ q \in Q: q \geq u \}$, and $\calD^{\ast}$ denote the unique distribution that rounds $u$ to $\{ \udown, \uup\}$ such that $\E_{\buh \sim \calD^{\ast}} [\buh] = u$.
    Then, for any distribution $\calD$ such that $\supp(\calD)\subseteq Q$ and $\mathbb{E}_{\buh\sim \calD}[\buh]= u$, we have that $\Var_{\buh\sim \calD}[\buh] \geq \Var_{\buh \sim \calD^{\ast}} [\buh]$.
\end{lemma}
\begin{proof}
     Let $\uqd := \supp(\calD) \cap \{ q \in Q : q \leq u\}$ and analogously $\oqd := \supp(\calD) \cap \{q \in Q: q \geq u\}$. The argument proceeds by first constructing an unbiased distribution $\calD'$ supported only on the centers of mass of $\uqd$ and $\oqd$ and showing that $\Var_{\buh \sim \calD'} [\buh] \leq \Var_{\buh \sim \calD} [\buh]$.
    We then show that $\Var_{\buh \sim \calD^{\ast}}[\buh] \leq \Var_{\buh \sim \calD'}[\buh]$ to conclude the claim. First, define the centers of mass
    \begin{align*}
        \uqdcom := \sum_{q \in \uqd} q \cdot \frac{\calD(q)}{\calD(\uqd)}, \quad \quad \quad \oqdcom := \sum_{q \in \oqd} q \cdot \frac{\calD(q)}{\calD(\oqd)}
    \end{align*}
    where $\calD(S) \triangleq \sum_{v \in S} \calD(v)$ for $S \subseteq Q$.
    Define distribution $\calD'$ by placing mass $\calD(\uqd)$ at $\uqdcom$ and mass $\calD(\oqd)$ at $\oqdcom$ as follows\footnote{Note that because $\calD'$ is an intermediate distribution only used for the purposes of the analysis, the fact that $\{\uqdcom, \oqdcom\}$ may \emph{not} be a subset of $Q$ does not affect correctness}
    \begin{align*}
        \calD'(x) = \begin{cases}
           \calD(\uqd) \text{ if } x = \uqdcom \\
           \calD(\oqd) \text{ if } x = \oqdcom \\
           0 \text{ otherwise }
        \end{cases}
    \end{align*}
    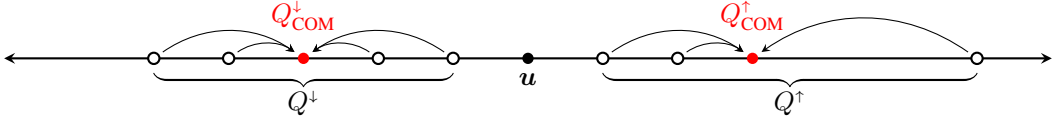
\begin{figure}[ht]
        \centering
        \begin{tikzpicture}[
            dot/.style={circle, inner sep=1.5pt, outer sep=0pt},
            open dot/.style={dot, draw=black, fill=white, thick},
            solid dot/.style={dot, fill=black},
            cm dot/.style={dot, fill=red},
            scale=0.99
        ]
            \draw[stealth-stealth, thick] (-7,0) -- (7,0);

            \node[solid dot, label=below:{$\bu$}] (u) at (0,0) {};

            \node[cm dot, label={[red, yshift=5pt]above:{$\uqdcom$}}] (CMminus) at (-3, 0) {};

            \foreach \x/\bdir in {-5/left, -4/left, -2/right, -1/right} {
                \node[open dot] (q\x) at (\x, 0) {};
                \draw[-stealth, black, shorten >=2pt, shorten <=2pt] (q\x) to[bend \bdir=35] (CMminus);
            }

            \node[cm dot, label={[red, yshift=5pt]above:{$\oqdcom$}}] (CMplus) at (3, 0) {};

            \foreach \x/\bdir in {1/left, 2/left, 6/right} {
                \node[open dot] (q\x) at (\x, 0) {};
                \draw[-stealth, black, shorten >=2pt, shorten <=2pt] (q\x) to[bend \bdir=35] (CMplus);
            }

            \draw[decorate, decoration={calligraphic brace, amplitude=5pt}, thick]
                (-1,-0.2) -- (-5,-0.2) node[midway, below=3pt] {$\uqd$};

            \draw[decorate, decoration={calligraphic brace, amplitude=5pt, mirror}, thick]
                (1,-0.2) -- (6,-0.2) node[midway, below=3pt] {$\oqd$};
        \end{tikzpicture}
        \caption{Illustration of distribution $\calD'$, showing the reduction of support of $\calD$ to just two elements}
        \label{fig: com-figure}
    \end{figure}
    See Figure \ref{fig: com-figure} for an illustration of the transformation from distribution $\calD$ to $\calD'$. Notice that
    \begin{align*}
        \E_{\buh \sim \calD'}[\buh] = \uqdcom \cdot \calD(\uqd) + \oqdcom \cdot \calD(\oqd) &= \sum_{q \in \uqd} q \cdot \calD(q) + \sum_{q \in \oqd} q \cdot \calD(q) \\
        &= \sum_{q \in Q} q \cdot \calD(q) = \E_{\buh \sim \calD}[\buh] = u
    \end{align*}
    Bounding the variance,
    \begin{align*}
        \Var_{\buh \sim \calD'}[\buh] &= \E_{\buh \sim \calD'}[\buh^2] - \E_{\buh \sim \calD'}[\buh]^2 \\
                                      &= [\calD(\uqd) \cdot {\uqdcom}^2 + \calD(\oqd) \cdot {\oqdcom}^2 ] - u^2 \\
        &= \left[ \frac{\left(\sum_{q \in \uqd} q \cdot \calD(q) \right)^2}{\calD(\uqd)} + \frac{\left(\sum_{q \in \oqd} q \cdot \calD(q) \right)^2}{\calD(\oqd)} \right] - u^2 \\
        &\leq \left[ \sum_{q \in \uqd} q^2 \calD(q) + \sum_{q \in \oqd} q^2 \calD(q) \right] - u^2 \tag{Cauchy-Schwarz Inequality} \\
        &= \left[ \sum_{q \in Q} q^2 \calD(q) \right] - u^2 = \Var_{\buh \sim \calD}[\buh]
    \end{align*}
    and thus $\var_{\buh\sim \calD'}[\buh] \leq \var_{\buh \sim \calD}[\buh]$. Because $\E_{\buh \sim \calD'}[\buh] = u$, it follows that $\calD'$ performs unbiased rounding of $u$ to the set $\{\uqdcom, \oqdcom\}$.
    Moreover, as $\uqdcom \leq \udown$ and $\oqdcom \geq \uup$ by definition of the center of mass, we have
    \begin{align*}
        \Var_{\buh \sim \calD'}[\buh] &= (u - \uqdcom)(\oqdcom - u) \geq (u - \udown)(\uup - u) = \Var_{\buh \sim \calD^{\ast}} [\buh]
    \end{align*}
    and so
    \[
        \var_{\buh\sim \calD^{*}}[\buh] \leq \var_{\buh\sim \calD'}[\buh] \leq \var_{\buh\sim \calD}[\buh].\qedhere
    \]
\end{proof}

We are now ready to prove Theorem \ref{thm:objective-three-optimal-rounding}. Here, for a given $w \in \R^d$ and $Q \subseteq \R$, we use the notation $\objth(w,\calD) := \max_{x : \norm{x}_2 \leq 1} \Var_{\buh \sim \calD} [\la \buh, x\ra]$ to denote the cost of choosing rounding distribution $\calD$. 

\begin{proof}[Proof of Theorem \ref{thm:objective-three-optimal-rounding}]
    Observe that 
    \begin{align*}
        \objth(w,\DASQ(w, Q)) &= \max_{x : \, \norm{x}_2 \leq 1} \Var_{\buh \sim \DASQ}[\la \buh, x \ra] \\
        &= \max_{x : \, \norm{x}_2  \leq 1} \sum_{i=1}^{d} x_i^2 \Var_{\buh \sim \DASQ} [\buh_i ] \tag{$\DASQ$ is a product distribution} \\
        &= \max_{i \in [d]} \Var_{\buh \sim \DASQ} [\buh_i ] = \Var_{\buh \sim \DASQ} [\buh_{i^\ast} ]
    \end{align*}
    where $i^\ast := \argmax_{i \in [d]} \Var_{\buh \sim \DASQ} [\buh_i ]$.
    Define $\Sigma := \Var_{\buh \sim \DASQ} [\buh_{i^\ast}]$.
    Now, consider any arbitrary rounding $\calD \in \Omega$. It follows that
    \begin{align*}
        \objth(w, \calD) \geq \Var_{\buh \sim \calD}[\la \buh, e_{i^{\ast}} \ra] = \Var_{\buh \sim \calD} [\buh_{i^{\ast}}] &\geq \Var_{\buh \sim \DASQ} [\buh_{i^\ast} ] \tag{Lemma \ref{lem:min-variance-independent-rounding}} \\
        &= \objth(w, \DASQ(w, Q))
    \end{align*}
    as desired.
\end{proof}

\begin{remark}
    Theorem \ref{thm:objective-three-optimal-rounding} shows that $\DASQ$ is a minimizer of $\objth(w, Q)$. Importantly, however, it may not be the \emph{unique} minimizer. It is feasible that there exist non-adaptive and/or dependent rounding schemes that simultaneously minimize $\objth(w, Q)$.
\end{remark}

\newcommand{\hwast}{\widehat{w}^{\ast}}
\newcommand{\what}{\widehat{w}}

\section{\texorpdfstring{NP-Hardness of Optimizing $\objf$}{NP-Hardness of Optimizing ADDV}} \label{appendix: addv}

\begin{lemma}[Gershgorin Circle Theorem, Theorem 6.1.1 in \cite{horn2012matrix}] \label{lem:gershgorin}
    Let $A \in \mathbb{C}^{d \times d}$ and let $R_i(A) := \sum_{j \neq i} |A_{ij}|$ for all $i \in [d]$. Then,
    the eigenvalues of $A$ are contained in the union of discs
    \begin{align*}
        \bigcup_{i \in [d]} \{ z \in \mathbb{C} : |z - A_{ii} | \leq R_i(A) \}
    \end{align*}
\end{lemma}

\begin{corollary}[Lower Bound on Minimum Eigenvalue of Real Symmetric Matrices] \label{cor: lower-bound-real-symmetric-eigenvalue}
    Let $A \in \R^{d \times d}$ be a symmetric matrix with eigenvalues $\lambda_1 \geq \ldots \geq \lambda_d$. Then,
    \begin{align*}
        \lambda_d \geq \min_{i \in [d]} \left( A_{ii} - \sum_{j \neq i} |A_{ij}| \right)
    \end{align*}
\end{corollary}
\begin{proof}
    By Lemma \ref{lem:gershgorin}, every eigenvalue of $A$ lies inside the union of discs
    \begin{align*}
        \bigcup_{i \in [d]} \{ z \in \mathbb{C} : A_{ii} - R_i(A) \leq z \leq A_{ii} + R_i(A) \}
    \end{align*}
    It follows that
    \begin{align*}
        \lambda_d \geq \min_{i \in [d]} \left( A_{ii} - R_i(A) \right) = \min_{i \in [d]} \left( A_{ii} - \sum_{j \neq i} |A_{ij}| \right)
    \end{align*}
\end{proof}

\begin{lemma}\label{lem:psd-cov}
    Let $M \in \R^{d \times d}$ be a symmetric positive semi-definite matrix. Then, there exists a random variable $\bx$ such that
    \begin{align*}
        \Cov(\bx) = \E[\bx \bx^{\t}] = M.
    \end{align*}
\end{lemma}

\begin{proof}
Define $\bz \sim \calN(0,I)$ and $\bx := M^{1/2} \bz$.
$\bx$ is well-defined as $M$ is symmetric and positive semi-definite, and thus has a real square root $M^{1/2}$.
Therefore,
\begin{align*}
    \Cov(\bx) &= \E[(\bx - \E[\bx]) (\bx- \E[\bx])^{\t}] \\
    &= \E[\bx \bx^{\t}] \tag{$\E[\bx] = \E[M^{1/2} \bz] = 0$} \\
    &= M^{1/2} \E[\bz \bz^{\t}] M^{1/2} \\
    &= M^{1/2} I M^{1/2} = M.\qedhere
\end{align*}
\end{proof}

\ADDVNPHard
\begin{proof}
    We show a reduction from the decision version of unweighted Max-Cut.
    Consider an instance given by an unweighted (simple) graph $G$. The algorithm is as follows.

    \begin{algorithm}[H]
        \caption{Max-Cut via $\objf$}
        \label{alg: np-hardness-reduction}
        \KwIn{Graph $G$, Oracle access to $\objf_{\calX}(w, Q)$}
        \KwOut{Max-Cut$(G)$}
        \BlankLine
        Construct adjacency matrix $A$ of $G$ \\
        Let $y := \min_{i \in [d]} \left( A_{ii} - \sum_{j \neq i} |A_{ij}| \right)$ and $M:= A - yI$ \\
        Let $\calD^{\text{opt}} := \objf_{\calN(0,M)}(0, \{-1,1\})$ \\
        Let $\hat{w}^{\text{opt}}$ be an arbitrary vector from the support of $\calD^{\text{opt}}$ \\
        Return $\frac{1}{4} \sum_{i \neq j} M_{ij} - \frac{1}{4} \sum_{i\neq j}\hat{w}^{\text{opt}}_i \hat{w}^{\text{opt}}_j M_{ij}$
    \end{algorithm}
    The adjacency matrix $A \in \R^{d \times d}$ of $G$ is
    \[
        A_{ij} := \begin{cases} 1 \text{ if } (i,j) \in G \\
                                0 \text{ otherwise} \end{cases}
    \]
    Then, define $y := \min_{i \in [d]} \left( A_{ii} - \sum_{j \neq i} |A_{ij}| \right)$ and let $M :=A - yI$. Observe that $M$ is positive semi-definite and symmetric. To see why this is the case, let $(v, \lambda)$ be an eigenvector-eigenvalue pair of $A$. Then,
     \begin{align*}
         Mv = (A - y I ) v = Av - yIv = \lambda v - yv = (\lambda - y) v,
     \end{align*}
     so the eigenvalues of $M$ are $\lambda_1 - y \geq \ldots \geq \lambda_d - y$.
     By Corollary \ref{cor: lower-bound-real-symmetric-eigenvalue}, $\lambda_d \geq y$, and all eigenvalues of $M$ are real and non-negative; hence $M$ is positive semi-definite. Symmetry of $M$ follows immediately from symmetry of $A$.

    We now construct a corresponding instance of $\objf$ and show that its optimum yields a solution to Max-Cut on $G$. In particular, consider the instance given by $w= 0 \in \R^d$, $Q = \{-1, 1\}$, and $\calX = \calN(0, M)$. Note that $\calX$ is a valid distribution due to Lemma \ref{lem:psd-cov}, and recall that $\Omega$ denotes the set of unbiased rounding distributions with support $\subseteq Q^d$. Then, the $\objf$ objective cost is given by

    \begin{align*}
        \min_{\calD \in \Omega} \E_{\bx \sim \calN(0, M)} \left[ \Var_{\bwh \sim \calD}[\la \bwh, \bx \ra] \right] &= \min_{\calD\in \Omega} \E_{\substack{\bx \sim \calN(0, M) \\ \bwh \sim \calD}} [\la \bwh, \bx \ra^2] \tag{$\E_{\bwh \sim \calD}[\la \bwh, \bx \ra] = \la w, x \ra = 0$}\\
        &= \min_{\calD\in \Omega} \E_{\bwh \sim \calD} \left[\sum_{i, j = 1}^{d} \bwh_i \bwh_j \cdot \E_{\bx \sim \calN(0, M)}[\bx_i \bx_j]\right] \\
        &= \min_{\calD\in \Omega} \E_{\bwh \sim \calD} \left[\sum_{i=1}^{d} \bwh_i^2 \cdot \E_{\bx \sim \calN(0, M)}[\bx_i^2] + \sum_{i \neq j} \bwh_i \bwh_j \cdot \E_{\bx \sim \calN(0, M)}[\bx_i \bx_j]\right] \\
        &= \min_{\calD\in \Omega} \E_{\bwh \sim \calD} \left[\sum_{i=1}^{d} \E_{\bx \sim \calN(0, M)}[\bx_i^2] + \sum_{i \neq j} \bwh_i \bwh_j \cdot \E_{\bx \sim \calN(0, M)}[\bx_i \bx_j]\right] \tag{$Q = \{-1,1\}$ } \\
        &= \min_{\calD\in \Omega} \E_{\bwh \sim \calD} \left[\sum_{i=1}^{d} M_{ii} + \sum_{i \neq j} \bwh_i \bwh_j \cdot M _{ij} \right] \tag{$\Cov(\bx) = M$}
    \end{align*}
    We now construct an optimal distribution $\calD^{\ast}$ for $\objf_{\calN(0,M)}(0, \{-1,1\})$. Define
    \begin{align*}
        \hwast := \argmin_{\what \in \{-1, 1\}^{d}} \left[ \sum_{i=1}^{d} M_{ii} + \sum_{i \neq j} \what_i \what_j \cdot M _{ij} \right] = \argmin_{\what \in \{-1, 1\}^{d}} \left[ \sum_{i \neq j} \what_i \what_j \cdot M _{ij} \right]
    \end{align*}
    Since the objective is even, let $\calD^{\ast}$ place equal weight on $\hwast$ and $-\hwast$ and observe that
    \begin{align*}
        \E_{\bwh \sim \calD^{\ast}} \left[\sum_{i=1}^{d} M_{ii} + \sum_{i \neq j} \bwh_i \bwh_j M_{ij} \right] &= \sum_{i=1}^{d} M_{ii} + \sum_{i \neq j} \hwast_i \hwast_j M_{ij}\\
        &= \min_{\calD \in \Omega} \E_{\bwh \sim \calD} \left[\sum_{i=1}^{d} M_{ii} + \sum_{i \neq j} \bwh_i \bwh_j \cdot M _{ij} \right]
    \end{align*}
    by definition of $\hwast$. Additionally, for any $x \in \text{supp}(\calX)$,
    \begin{align*}
        \E_{\bwh \sim \calD^{\ast}}[\la \bwh, x \ra] = \la \hwast, x \ra + \la -\hwast, x \ra = 0 = \la w, x \ra
    \end{align*}
    Thus, $\calD^{\ast}$ is an optimal distribution giving objective value
    \begin{align*}
         \sum_{i=1}^{d} M_{ii} + \min_{\what \in \{-1, 1\}^d} \sum_{i \neq j} \what_i \what_j M_{ij}
    \end{align*}
    $\calD^{\ast}$ may not be the unique optimal distribution, but any optimal distribution must be supported only on vectors $\hat{w}$ that minimize the quantity $\sum_{i \neq j} \what_i \what_j M_{ij}$ (since its objective value is the same as $\calD^{\ast}$). Therefore, given any optimal distribution $\calD^{\text{opt}}$, we can extract an arbitrary vector $\hat{w}^{\text{opt}}$ from its support.   

    Finally, observe that the max-cut objective on graph $G$ can be written as
     \begin{align*}
         \text{Max-Cut}(G) &:= \max_{\what \in \{-1,1\}^d} \frac{1}{2} \sum_{i \neq j} \frac{(\what_i - \what_j)^2}{4} \cdot M_{ij} = \frac{1}{4} \sum_{i \neq j} M_{ij} - \min_{\what \in \{-1,1\}^d} \frac{1}{4} \sum_{i \neq j} \what_i \what_j \cdot M_{ij}
     \end{align*}
     or equivalently
     \begin{align*}
        \text{Max-Cut}(G) &:=\frac{1}{4} \sum_{i \neq j} M_{ij} - \frac{1}{4} \sum_{i\neq j}\hat{w}^{\text{opt}}_i \hat{w}^{\text{opt}}_j M_{ij}
     \end{align*}

    \paragraph{Runtime.} It is clear that constructing $A$, $y$, and $M$ take $O(d^2)$ time. Because computing and reporting $\objf_{\calN(0, M)}(0, \{-1,1\})$ is assumed to take polynomial time, extracting $w^{\text{opt}}$ from the support of $\calD^{\text{opt}}$ also takes polynomial time. Therefore, Algorithm \ref{alg: np-hardness-reduction} is a polynomial-time reduction from unweighted Max-Cut to $\objf$. Because Max-Cut on unweighted, simple graphs is NP-Hard \cite{garey1974some}, this proves the claim. 
\end{proof}

\newcommand{\qopt}{Q^{\text{opt}}}
\newcommand{\Qstar}{Q^{\ast}}

\newcommand{\qiup}{q_i^{\scriptscriptstyle \uparrow}}
\newcommand{\qidown}{q_i^{\scriptscriptstyle \downarrow}}

\newcommand{\qup}{q^{\scriptscriptstyle \uparrow}}
\newcommand{\qdown}{q^{\scriptscriptstyle \downarrow}}

\newcommand{\qapprox}{Q_{\text{approx}}}

\section{Average Directional Variance} \label{sec: adv}

We begin by making important observations about the $\objt$ objective and defining quantities that will be useful for the algorithms in the remainder of the section. 
Oftentimes, we also think of $w$ as a multiset. 

Furthermore, we assume that the input distribution $\calX$ has $\supp(\calX) \subseteq \R^d$ and has finite marginal second moments (i.e. $\E_{\bx \sim \calX}[\bx_i^2]$ is finite for all $i \in [d]$). The problems we consider become trivial otherwise. Our algorithms also assume full-precision access to the marginal second moments. 

\begin{lemma} \label{lem: endpoints-in-quant-set}
    Let $\calX$ be an arbitrary input distribution and $w\in \mathbb{R}^d$ be sorted such that $w_1\leq \cdots \leq w_d$.
    Let $\Qstar$ be such that $|\Qstar| = s$ and $\objt_{\calX}(w, \Qstar) = \objt_{\calX}(w,s)$. Then, $w_1, w_d \in \Qstar$. 
\end{lemma}

\begin{lemma} \label{lem: quant-set-inside-w}
    For every input distribution $\calX$, $w \in \R^d$, and $s \in \N$, there exists a quantization set $Q^{\ast} \subseteq \R$ such that (i) $|Q^{\ast}| = s$, (ii) $Q^{\ast} \subseteq w$, and (iii) $\objt_{\calX}(w, Q^{\ast}) = \objt_{\calX}(w, s)$. 
\end{lemma}

We now define and prove a property of the optimization problem that is the crux of all the results in this section. We first define a shorthand, with $w$ indexed so $w_1 \leq \cdots \leq w_d$,
\begin{align*}
    C[j, k] &\triangleq \sum_{i= j}^{k} \lambda_i \cdot (w_k - w_i) (w_i - w_j) 
\end{align*}
so that $C \in \R^{d \times d}$. We use the standard convention that $C[j,k] = 0$ if $j > k$. Here, $\lambda_i := \E_{\bx \sim \calX}[\bx_i^2] \in \R^{+}$ denotes a weight associated with the element $w_i$. The quantity $C[j,k]$ represents the (weighted) sum of variances of points in the region $[w_j, w_k]$ assuming $w_j, w_k \in Q$. 


\begin{lemma}[Lemma 4.2 of \cite{ben2024optimal}] \label{lem: C-totally-monotone}
    For every input distribution $\calX$, Matrix $C$ satisfies the Concave Monge property (Definition \ref{def: concave-monge-property}). 
\end{lemma}

\begin{lemma} \label{lem: C-sorted}
    For every input distribution $\calX$, Matrix $C$ is sorted (Definition \ref{def: sorted-matrix}).
\end{lemma}

The proofs of Lemmas \ref{lem: endpoints-in-quant-set}, \ref{lem: quant-set-inside-w}, \ref{lem: C-totally-monotone}, and \ref{lem: C-sorted} are deferred to Appendix \ref{sec: missing-proofs}. Note that Lemma \ref{lem: C-totally-monotone} is the weighted version of Lemma 4.2 in \cite{ben2024optimal}. The proof closely follows that of Lemma 4.2 in \cite{ben2024optimal}.

\begin{lemma} \label{lem: efficiently-compute-C}
    Given (unsorted) $w\in \mathbb{R}^d$, with $O(d \log d)$ pre-processing time, any entry of $C[i,j]$ can be computed in $O(1)$ time.  
\end{lemma}
\begin{proof}
    Sort $w$ so that $w_1 \leq \cdots \leq w_d$.
    We define prefix sum vectors
    \begin{align*}
        \alpha[j] = \sum_{i=1}^{j} \lambda_i \quad \quad \quad \quad \beta[j] = \sum_{i = 1}^{j} \lambda_i w_i \quad \quad \quad \quad \gamma[j] = \sum_{i=1}^{j} \lambda_i w_i^2
    \end{align*}
    Then,
    \begin{align*}
        C[j, k] &:= \sum_{i= j}^{k} \lambda_i \cdot (w_k - w_i) (w_i - w_j) = \sum_{i= j+1}^{k} \lambda_i \cdot (w_k - w_i) (w_i - w_j) \\
        &= - w_j w_k \sum_{i=j+1}^{k} \lambda_i + (w_j + w_k) \cdot \sum_{i=j+1}^{k} \lambda_i w_i - \sum_{i = j+1}^{k} \lambda_i w_i^2 \\
        &= -w_j w_k \cdot (\alpha[k] - \alpha[j]) + (w_j + w_k) \cdot (\beta[k] - \beta[j]) - (\gamma[k] - \gamma[j])
    \end{align*}
    \paragraph{Runtime.} Sorting $w$ takes $O(d \log d)$ time. Then, the prefix sums $\alpha, \beta, \gamma$ can be computed in $O(d)$ time once the vector $w$ is sorted. Thus, the total pre-processing time is $O(d \log d)$. Computing a specific entry $C[j,k]$ can be done with $O(1)$ accesses to $w$ and the prefix sums $\alpha, \beta, \gamma$.
\end{proof}

\begin{remark}
    Throughout this section, we assume the exact marginal second moments are known (or can computed easily from the representation of the distribution $\calX$), i.e.~the values $\lambda_i$ are given for all $i\in [d]$.
    However, if instead the algorithms are given approximations $\lambda_i'$ for each $\lambda_i$ with $\lambda_i \leq \lambda_i' \leq \alpha \lambda_i$, with $\alpha$ known, then the approximation ratios for each of the algorithms in this section increase by a factor of $\alpha$.
    The knowledge of $\alpha$ is required for Algorithm \ref{alg: 1+eps-adv-alg}, in order to be able to bound the search range for the Lagrangian $\tau$.
\end{remark}

\subsection{Exact Algorithms} \label{sec: ADV-exact-algorithms}

In Appendix K of \cite{ben2024optimal}, the authors note that optimizing the weighted MSE objective can be solved using a slight variant of the dynamic programming and matrix-search based algorithm presented in their paper.
The crux of the improvement in their result comes from the total monotonicity of matrix $C$ that allows a more efficient computation of the dynamic program. 

The survey paper of \cite{GKMNSS17} outlines many results from the $1$-dimensional $k$-means clustering literature, which turn out to immediately give algorithms for the problem of vector quantization with error measured by the weighted MSE objective. For example, it turns out that the exact same dynamic programming and matrix-searching solution of Corollary \ref{cor: exact-weighted-mse} of \cite{ben2024optimal} was actually first published by \cite{Wu91}, in the setting of 1D $k$-means.
Formally, this implies the following about objective $\objt$.

\begin{corollary}[\cite{Wu91}, Adapted Algorithm 1 of \cite{ben2024optimal}] \label{cor: exact-weighted-mse}
    There exists an algorithm such that for any $w \in \R^d$, input distribution $\calX$, and target quantization set size $s \in \N$, returns a quantization set $Q$ such that $|Q| = s$ and $\objt_{\calX}(w, Q) = \objt_{\calX}(w,s)$. The algorithm has runtime $O(d \log d + d \cdot s)$ and uses $O(ds)$ space. 
\end{corollary}


More generally, there is a spectrum of theoretical results that complete the state-of-the-art optimality profile in differing regimes of $k$ and $d$. In particular, the results outlined in the survey paper of \cite{GKMNSS17} can be applied to any problem whose structure can be shown to satisfy the Concave Monge property (Definition \ref{def: concave-monge-property}) (including but not limited to, $1$-dimensional $k$-means clustering, $1$-dimensional $k$-medians clustering, and $\objt_{\calX}$).

More specifically, the algorithms outlined in \cite{GKMNSS17} solve the following problem. Let $G$ be a complete, directed acyclic graph on $d$ vertices with weights $A \in \R^{d \times d}$ that satisfy the Concave Monge property (Definition \ref{def: concave-monge-property}). Then, given any two vertices $i, j \in [d]$ and integer $k$, we wish to find the minimum weight length-$k$ path between $i$ and $j$ in $G$. 

\cite{AST93} gives (i) an algorithm with runtime $O(d \sqrt{k \log d})$, and (ii) an algorithm with runtime $O(d \log \Delta)$ where $\Delta := \max_{i, j} A[i,j] - \min_{i, j} A[i,j]$. When $k = \Omega(\log d)$, \cite{Schieber98} gives an algorithm with runtime $d \smash 2^{O(\sqrt{\log \log d \log k})}$. At a high level, these algorithms work by solving the following \emph{regularized} version of the problem
\begin{align*}
    \min_{\ell, \{ i_1, \ldots, i_{\ell} \} \subseteq [d]} \sum_{j=1}^{\ell-1} C[i_j, i_{j+1}] + \tau \ell
\end{align*}
where the task is to find an integer $\ell$ and a path of length $\ell$ that minimizes the regularized objective function (i.e. each edge has additional cost $\tau$). In doing so, the results of \cite{AST93,Schieber98} use ideas and results from \cite{Wilber98} that solves the unregularized version (still assuming the weights are Concave Monge).

Notice that by Lemma \ref{lem: C-totally-monotone}, applying these algorithms to weights $C[i,j]$ directly gives a solution optimizing $\objt_{\calX}$, implying the following results. 

\begin{corollary}[\cite{AST93}] \label{cor: dag-exact-alg}
    There exist algorithms such that for any $w \in \R^d$, input distribution $\calX$, and target quantization set size $s \in \N$, return a quantization set $Q$ such that $|Q| = s$ and $\objt_{\calX}(w, Q) = \objt_{\calX}(w,s)$. The algorithms have runtime (i) $O(d \sqrt{s \log d})$ and (ii) $O(d \log \Delta)$ where $\Delta := \max_{i, j} C[i,j] - \min_{i, j} C[i,j]$.
\end{corollary}

\begin{corollary}[\cite{Schieber98}]
    There exists an algorithm such that for any $w \in \R^d$, input distribution $\calX$, and target quantization set size $s = \Omega( \log d)$, returns a quantization set $Q$ such that $|Q| = s$ and $\objt_{\calX}(w, Q) = \objt_{\calX}(w,s)$. The algorithm has runtime $d2^{O(\sqrt{\log \log d \log s})}$.
\end{corollary}

 This summarizes the state-of-the-art of exact algorithms for $\objt_{\calX}$. 





\subsection{Approximation Algorithms}

The work of \cite{ben2024optimal} provides a bi-criteria approximation algorithm for the \emph{unweighted} MSE objective. In particular, their algorithm returns a quantization set of size $2s-2$ whose quality is an additive approximation to that of the optimal. It is easy to see, following their analysis, that the same algorithm, analyzed under the \emph{weighted} MSE objective results in the following weak additive approximation guarantee. 

\begin{corollary}[Lemma 6.1 of \cite{ben2024optimal}]\label{lem: additive-approx-weighted-mse}
    There exist algorithms such that for any $w \in \R^d$, input distribution $\calX$, and target quantization set size $s \in \N$, returns a quantization set $Q$ such that $|Q| = 2s-2$ and $\objt_{\calX}(w, Q) \leq \objt_{\calX}(w, s) + \smash \sum_{i=1}^{d} \lambda_i \cdot \Delta^2/m^2$. Here, $\Delta := \max_{i, j \in [d]} w_i - w_j$. The algorithm has runtime $O(d + ms)$. 
\end{corollary}

In this section, we develop approximation algorithms for optimizing the $\objt_{\calX}$ (weighted MSE) objective. We first give an $s$-approximation, then improve it to a $(1+\epsilon)$-approximation.
Along the way, we observe that the main ideas from solutions to $\objo$ and $\objt_{\calX}$ (in particular, building data-dependent coresets) can be used to develop a much stronger approximation algorithm for quantization under the unweighted MSE objective. That result is detailed in Appendix \ref{sec: approx-alg-MSE}. 

\subsubsection{\texorpdfstring{An $s$-approximation}{An s-approximation}}
\label{sec:vmix}

Here, we give an $s$-approximation algorithm by defining a new intermediate objective function that (loosely speaking) ``interpolates'' between $\objo$ and $\objt_{\calX}$. We then show that the optimum of this new \emph{mixed} objective is an $s$-approximation to that of $\objt_{\calX}$. In particular, consider a quantization set $Q = \{ w_{i_1}, \ldots, w_{i_{|Q|}} \} \subseteq w$ where $w_{i_1} \leq \ldots \leq w_{i_{|Q|}}$.

We now define the objective 
\begin{align}
    \objmixx(w, s) := \min_{Q \subset w: |Q| \leq s} \max_{j \in [s-1]} C[i_j, i_{j+1}]
\end{align}
which seeks to find the quantization set $Q$ (lying entirely inside $w$) that minimizes the maximum sum of the variances in the intervals between adjacent quantization points of $Q$. We also define a shorthand to denote the objective value of a fixed quantization set $Q$:  
\begin{align*}
    \objmixx(w, Q) := \max_{j \in [|Q|-1]} C[i_j, i_{j+1}]
\end{align*}
We now prove the following relationship between $\objmixx$ and $\objt_{\calX}$. 

\begin{lemma} \label{lem: objmix-approximates-objt}
    For $s \geq 2$, $\objtx(w, s) / s \leq \objmixx(w, s) \leq \objtx(w, s)$
\end{lemma}
\begin{proof}
    Notice that for any quantization set $Q = \{w_{i_1}, \ldots, w_{i_{s}} \} \subseteq w$ of size $s$, it is the case that
    \begin{align*}
        \objtx(w, s) \leq \objtx(w, Q) = \sum_{j \in [s-1]} C[i_j, i_{j+1}] \leq |Q| \cdot \objmixx(w, Q) 
    \end{align*}
    where all steps follow by the definitions of $\objtx$ and $\objmixx$. Let $Q^{\ast}_{\objmixx} \subseteq w$ be such that $|Q^{\ast}_{\objmixx}| = s$ and $\objmixx(w, Q^{\ast}_{\objmixx}) = \objmixx(w, s)$. 
    Then, instantiating the above with $Q^{\ast}_{\objmixx}$, we conclude that $\objmixx(w, s) \geq \objtx(w, s) / s$. It remains to show that $\objmixx(w, s) \leq \objtx(w, s)$. Let $Q_{\objtx}^{\ast}$ be such that $|Q_{\objtx}^{\ast}| = s$ and $\objtx(w, Q_{\objt}^{\ast}) = \objtx(w, s)$. It is clear that
    \begin{align*}
        \objmixx(w, s) \leq \objmixx(w, Q_{\objt}^{\ast}) \leq \objtx(w, Q_{\objt}^{\ast}) = \objtx(w, s)
    \end{align*}
    where again all steps follow by definition of the objectives. 
\end{proof}

Lemma \ref{lem: objmix-approximates-objt} suggests a natural approximation algorithm for $\objtx$. In particular, an exact solution to $\objmixx$ immediately gives an $s$-approximation to $\objtx$. 

\begin{algorithm}[ht]
    \caption{Exact Algorithm for $\objmixx$}
    \label{alg: exact-mix-algorithm}
    \KwIn{$w\in \mathbb{R}^{d}$, $s \in \N$, $\lambda \in \R^d$ where $\lambda_i = \E_{\bx \sim \calX}[\bx_i^2]$ for $i \in [d]$}
    \KwOut{Quantization set $Q$}

    \BlankLine
    Sort $w$ and compute prefix sums $\alpha, \beta, \gamma$ \\
    Initialize $k^{\arup} \gets d^2$ and $k^{\ardo} \gets 0$ \\
    Initialize $Q^{\ast} \gets \emptyset$ \\
    \BlankLine

    \While{$k^{\arup} \geq k^{\ardo}$}{
        $k \gets \lfloor (k^{\arup} + k^{\ardo})/2 \rfloor$ \\
        $v \gets  \texttt{Sorted-Selection}(\alpha, \beta, \gamma, k)$ \\
        \BlankLine \algcomment{Check whether objective value $v$ is possible with $s$ quantization points} \\
        Initialize $Q\gets \{ w_1, w_d \}$ and $i \gets 1$ \label{line: begin-check}\\
        \For{$j=2, \ldots, d$}{ 
            \If{$C[i, j] > v$}{ 
                $Q \gets Q \cup \{ w_{j-1} \}$ \\
                Set $i \gets j-1$ \\
            }
        }
        
        \BlankLine
        \If{$|Q| \leq s$}{ \label{line: size-check}
            $k^{\arup} \gets k-1$ \\
            $Q^{\ast} \gets Q$
        }
        \Else{
            $k^{\ardo} \gets k+1$
        } \label{line: end-check}
    }

    \Return $Q^{\ast}$
\end{algorithm}

\begin{theorem} \label{thm: adv-s-approx}
    There exists an algorithm such that for any $w \in \R^d$, input distribution $\calX$, and target quantization set size $s \in \N$, returns a quantization set $Q$ such that $|Q| = s$ and $\objmixx(w, Q) = \objmixx(w, s)$, and consequently $\objtx(w, s) / s \leq \objmixx(w, Q) \leq \objtx(w, s)$. The runtime of the algorithm is $O(d \log d)$.
\end{theorem}
\begin{proof}
    To produce an $s$-approximation to $\objt$ without relaxing the target quantization set size, we will use Algorithm \ref{alg: exact-mix-algorithm} to exactly solve the $\objmixx$ objective. 
    
    Observe that the true optimal cost $\objmixx(w, s)$ is of the form $C[i_j, i_k]$ for some $j, k \in [d]$. Algorithm \ref{alg: exact-mix-algorithm} binary searches over the $O(d^2)$ many possible objective values, using the linear time checking algorithm \texttt{Sorted-Selection} from Lemma \ref{lem: k-selection} to direct the binary search. The fact that matrix $C$ is sorted is proven as Lemma \ref{lem: C-sorted} in Section \ref{sec: missing-proofs}. It remains to prove the correctness of the checking step (Lines \ref{line: begin-check} - \ref{line: end-check}). A given objective value $v \in \R$ is achievable if there exists an arrangement of quantization points $Q \subset w$ for $\objmixx(w, Q) \leq v$. Algorithm \ref{alg: exact-mix-algorithm} greedily constructs a quantization set by performing a linear scan through the sorted $w$ and only adding a quantization point once the sum of variances exceeds $v$. It is trivial to see that if $\objmixx(w, s) \leq v$, then the greedily constructed quantization set $Q_{\text{Greedy}}$ must have $|Q_{\text{Greedy}}| \leq s$ and $\objmixx(w, Q_{\text{Greedy}}) \leq v$. 
    
    By construction, $Q_{\text{Greedy}}$ is such that $\objmixx(w, Q_{\text{Greedy}}) \leq v$. Therefore, if $|Q_{\text{Greedy}}| \leq s$, objective value $v$ is achievable. Otherwise, it is not. This check in lines \ref{line: size-check} - \ref{line: end-check} directs the binary search. It follows that the returned quantization set $Q^{\ast}$ satisfies $\objmixx(w, Q^{\ast}) = \objmixx(w, s)$. If $|Q^{\ast}| < s$, we pad it with arbitrary elements of $w$; this only splits intervals, which by Lemma \ref{lem: C-sorted} does not increase the cost of any interval, and so leaves $\objmixx(w, Q^{\ast})$ unchanged. The claim that $\objtx(w, s)/s \leq \objmixx(w, Q^{\ast}) \leq \objtx(w,s)$ then follows from Lemma \ref{lem: objmix-approximates-objt}.

    \paragraph{Runtime.} Sorting $w$ takes $O(d \log d)$ time. Computing the prefix sums $\alpha, \beta, \gamma$ takes $O(d \log d)$ time as per Lemma \ref{lem: efficiently-compute-C}. The while loop is performing binary search over $O(d^2)$ many values and is thus only run $O(\log d)$ times. The $k$-selection algorithm takes time $O(d)$ as per Lemma \ref{lem: k-selection} (note that the matrix $C$ is never explicitly written out by the algorithm. It is implicitly defined by the prefix sums $\alpha, \beta, \gamma$ which allow $O(1)$ access time to an entry of $C$ as per Lemma \ref{lem: efficiently-compute-C}). The \emph{check} logic in Lines \ref{line: begin-check} - \ref{line: end-check} take $O(d)$ time. Therefore, the total algorithm runtime is $O(d \log d)$.
\end{proof}

\subsubsection{\texorpdfstring{A $(1+\epsilon)$-approximation}{A (1+epsilon)-approximation}}

Using the $s$-approximation from Algorithm \ref{alg: exact-mix-algorithm}, we can get a $(1+\epsilon)$-approximation algorithm with efficient runtime. At a high level, the idea is to first obtain an $s$-approximation to the objective value, and then use it to bound both the range of the search for the Lagrangian multiplier and the precision to which it must be found in the $O(d \log \Delta)$ time algorithm of \cite{AST93} (Corollary \ref{cor: dag-exact-alg}). We begin by giving a primer on this algorithm. This discussion closely mirrors that in Section 3 of \cite{GKMNSS17}. Recall that the problem at hand is to optimize the \emph{regularized} objective
\begin{align*}
    \min_{\ell, \{ i_1, \ldots, i_{\ell} \} \subseteq [d]} \sum_{j=1}^{\ell-1} C[i_j, i_{j+1}] + \tau \ell
\end{align*}
which corresponds to choosing the lightest path in a complete DAG with weights given by $C$ and cost $\tau$ for each node in the path. Observe that the unregularized version is simply optimizing $\objtx$. Setting $\tau = 0$ means that there is no cost associated with choosing extra nodes, so the optimal path is to include every node, resulting in an objective value of $0$. Now consider setting $\tau \geq \opt_{d-1}$ where 
\begin{align*}
    \opt_{d-1} = \min_{\{ i_1, \ldots, i_{d-1} \} \subseteq [d]} \sum_{j=1}^{\ell-1} C[i_j, i_{j+1}]
\end{align*}
is equivalent to $\objtx(w, d-1)$. For these values of $\tau$, it is better to construct the optimal $d-1$ length path than construct a $d$ length path. This is because the extra additive $\tau$ cost of using the $d$th node outweighs the reduction of $C[i,j]$ costs when using $d$ nodes. In this way, we find \emph{critical points} for the setting of $\tau$ at which the optimal path is of length $k$. That is $\tau_k = \opt_{k} - \opt_{k+1}$ ensures that the optimal path is of length $k$. It turns out that $0 = \tau_d \leq \tau_{d-1} \leq \ldots \leq \tau_1$ (see \cite{AST93, GKMNSS17} for proofs and details). The $O(d \log \Delta)$ algorithm then proceeds by binary searching for the correct setting of $\tau$; it checks whether a particular setting of $\tau$ yields an optimal solution of length $s$ using the $O(d)$ time algorithm of \cite{AST93, Schieber98} which finds shortest paths on DAGs with concave monge weights. In particular, \cite{AST93, Schieber98} gives three types of algorithms: (1) $\texttt{ShortestPathDAG-Min}$ which returns the shortest path with optimal cost (2) $\texttt{ShortestPathDAG-Max}$ which returns the longest path with optimal cost and (3) $\texttt{ShortestPathDAG}(k)$ which returns the length-$k$ path with optimal cost (if it exists). All have runtime $O(d)$. 
Running and checking these algorithms allows us to direct the binary search. 

Algorithm \ref{alg: 1+eps-adv-alg} first runs the $s$-approximation Algorithm \ref{alg: exact-mix-algorithm} and computes the objective value $v = \objmixx(w, s)$ of the returned quantization set, which by Lemma \ref{lem: objmix-approximates-objt} satisfies $\objtx(w,s)/s \leq v \leq \objtx(w,s)$ (this value can be computed from the returned quantization set in linear time trivially). It then performs the binary search for $\tau$ described above, but only over multiples of $\epsilon v /s$. Each iteration runs $\texttt{ShortestPathDAG-Min}$ and $\texttt{ShortestPathDAG-Max}$ on the current setting of $\tau$ and records the lengths $k_{\textsf{min}}$ and $k_{\textsf{max}}$ of the two returned paths. If some multiple of $\epsilon v /s$ admits an optimal path of length $s$, that path is returned. Otherwise, the search ends at two consecutive multiples of $\epsilon v /s$, where all optimal paths for the smaller value are longer than $s$ and all optimal paths for the larger value are shorter than $s$. We then combine an optimal path for each of these two values into a single path of length $s$, using the exchange argument of Lemma \ref{lem: path-exchange} below. 

Throughout, a \emph{path} is an increasing sequence of indices $P = (i_1, \ldots, i_k)$ with $i_1 = 1$ and $i_k = d$. We write $|P| = k$ for its length (i.e.~its number of nodes) and $\cost(P) = \sum_{j=1}^{k-1} C[i_j, i_{j+1}]$ for its cost, and $\opt_k$ for the minimum cost of a path of length $k$. Recall that by Lemmas \ref{lem: endpoints-in-quant-set} and \ref{lem: quant-set-inside-w}, a path $P$ corresponds to the quantization set $Q = \{ w_{i_1}, \ldots, w_{i_k} \}$ with $\objtx(w, Q) = \cost(P)$, and that $\opt_s = \objtx(w, s)$. For $\mu \geq 0$, we say that $P$ is \emph{$\mu$-optimal} if it minimizes $\cost(P) + \mu |P|$ over all paths. Since all paths of the same length incur the same penalty, a $\mu$-optimal path $P$ satisfies $\cost(P) = \opt_{|P|}$.

\begin{algorithm}[ht]
    \caption{$(1+\epsilon)$-Approximation Algorithm for $\objt_{\calX}$}
    \label{alg: 1+eps-adv-alg}
    \KwIn{$w\in \mathbb{R}^{d}$, $s \in \N$, $\epsilon > 0$, $\lambda \in \R^d$ where $\lambda_i = \E_{\bx \sim \calX}[\bx_i^2]$ for $i \in [d]$}
    \KwOut{Quantization set $Q$}

    \BlankLine
    $v \gets \objmixx(w, \text{Algorithm } \ref{alg: exact-mix-algorithm}(w, s, \lambda))$ \label{line: one} \\
    Sort $w$ and compute prefix sums $\alpha, \beta, \gamma$ \\
    Initialize $\tau^{\arup} \gets \lceil s^2 (1+\epsilon) / \epsilon \rceil$ and $\tau^{\ardo} \gets 0$ \label{line: search-range} \\
    \BlankLine
    \While{$\tau^{\arup} \geq \tau^{\ardo}$}{ \label{line: binary-search-start}
        $\tau \gets \lfloor (\tau^{\arup} + \tau^{\ardo})/2 \rfloor$ \\
        $k_{\textsf{min}} \gets \texttt{ShortestPathDAG-Min}(C, \tau \epsilon v / s)$ \\
        $k_{\textsf{max}} \gets \texttt{ShortestPathDAG-Max}(C, \tau \epsilon v / s)$ \\

        \If{$k_{\textsf{min}} \leq s \leq k_{\textsf{max}}$}{
            Return $\texttt{ShortestPathDAG}(C, \tau \epsilon v / s, s)$ \label{line: hit-return}
        }
        \ElseIf{$s < k_{\textsf{min}}$}{ \label{line: update-start}
            $\tau^{\ardo} \gets \tau + 1$
        }
        \Else{
            $\tau^{\arup} \gets \tau - 1$
        }
    } \label{line: binary-search-end}
    \BlankLine
    Return the path of Lemma \ref{lem: path-exchange} for $\texttt{ShortestPathDAG-Min}(C, \tau^{\arup} \epsilon v /s)$ and $\texttt{ShortestPathDAG-Max}(C, \tau^{\ardo} \epsilon v /s)$ \label{line: exchange-return} \quad \algcomment{here $\tau^{\ardo} = \tau^{\arup} + 1$}
\end{algorithm}

\begin{lemma} \label{lem: path-exchange}
    Let $P_1 = (a_1, \ldots, a_{k_1})$ and $P_2 = (b_1, \ldots, b_{k_2})$ be paths with $k_2 < s < k_1$. Then, there exist paths $P$ and $P'$ with $|P| = s$ and $|P'| = k_1 + k_2 - s$ such that
    \begin{align*}
        \cost(P) + \cost(P') \leq \cost(P_1) + \cost(P_2)
    \end{align*}
    Moreover, $P$ can be computed in $O(d)$ time.
\end{lemma}
\begin{proof}
    Let $t = s - k_2 \geq 1$ and notice that $q + t \leq s < k_1$ for every $q \in [k_2]$. Since $b_1 = 1 \leq a_{1+t}$ while $b_{k_2} = d = a_{k_1} > a_{k_2 + t}$, taking $q \in [k_2 - 1]$ to be minimal with $b_{q+1} > a_{q+t+1}$ also gives $b_q \leq a_{q + t}$. Setting $p = q + t$, we therefore have $b_q \leq a_p < a_{p+1} < b_{q+1}$. Define
    \begin{align*}
        P := (a_1, \ldots, a_p, b_{q+1}, \ldots, b_{k_2}) \quad \quad \text{and} \quad \quad P' := (b_1, \ldots, b_q, a_{p+1}, \ldots, a_{k_1})
    \end{align*}
    Both are paths, of lengths $p + k_2 - q = s$ and $q + k_1 - p = k_1 + k_2 - s$ respectively. Their edges are exactly those of $P_1$ and $P_2$, except that $(a_p, a_{p+1})$ and $(b_q, b_{q+1})$ are replaced by $(a_p, b_{q+1})$ and $(b_q, a_{p+1})$. Since $b_q \leq a_p \leq a_{p+1} \leq b_{q+1}$ and $C$ satisfies the Quadrangle Inequality (Definition \ref{def: quadrangle-inequality}; see the proof of Lemma \ref{lem: C-totally-monotone} in Appendix \ref{sec: missing-proofs}),
    \begin{align*}
        C[b_q, a_{p+1}] + C[a_p, b_{q+1}] \leq C[b_q, b_{q+1}] + C[a_p, a_{p+1}]
    \end{align*}
    which proves the claim. Finding $q$ requires a single pass over $P_2$, so $P$ can be computed in $O(d)$ time.
\end{proof}

\ADVApproxAlg

\begin{proof}
    The algorithm is described in Algorithm \ref{alg: 1+eps-adv-alg}. Let $\eta = \epsilon v / s$ and $T = \lceil s^2 (1+\epsilon)/\epsilon \rceil$. By Theorem \ref{thm: adv-s-approx}, Line \ref{line: one} computes $v = \objmixx(w, s)$, which by Lemma \ref{lem: objmix-approximates-objt} satisfies $\opt_s / s \leq v \leq \opt_s$.

    Suppose first that the algorithm returns in Line \ref{line: hit-return}, and let $\tau$ denote the value of the loop variable in that iteration. Since $k_{\textsf{min}} \leq s \leq k_{\textsf{max}}$, the call to $\texttt{ShortestPathDAG}$ in that line returns a $\tau \eta$-optimal path of length $s$ \cite{AST93, Schieber98}, and so its cost is $\opt_s = \objtx(w, s)$, as desired.

    Call an integer $\tau$ \emph{too small} if $\texttt{ShortestPathDAG-Min}(C, \tau \eta)$ returns a path of length greater than $s$, and \emph{too large} if $\texttt{ShortestPathDAG-Max}(C, \tau \eta)$ returns a path of length less than $s$.
    Lines \ref{line: update-start} - \ref{line: binary-search-end} maintain that $\tau^{\ardo} - 1$ is too small whenever $\tau^{\ardo} > 0$ and that $\tau^{\arup} + 1$ is too large whenever $\tau^{\arup} < T = \lceil s^2 (1+\epsilon)/\epsilon \rceil$, and the loop terminates with $\tau^{\ardo} = \tau^{\arup} + 1$.
    Observe that $0$ is not too large, since $C$ is nonnegative and $C[i, i+1] = 0$ for all $i$, so the path through all $d$ nodes is $0$-optimal.
    Observe also that $T$ is not too small: a $T \eta$-optimal path $P$ satisfies $\cost(P) + T \eta |P| \leq \opt_s + T \eta s$, which for $|P| > s$ would give $T \eta \leq \opt_s$, but $T \eta \geq (1+\epsilon) s v > \opt_s$ (Lemma \ref{lem: objmix-approximates-objt}).
    It follows that $0 \leq \tau^{\arup} < T$ at termination, and so $\tau^{\arup}$ is too small while $\tau^{\ardo}$ is too large.

    Let $P_1$ and $P_2$ be the paths computed in Line \ref{line: exchange-return}, so that $k_1 := |P_1| > s$ and $k_2 := |P_2| < s$, and let $P$ and $P'$ be the paths given by Lemma \ref{lem: path-exchange}, with $|P| = s$ and $|P'| = m := k_1 + k_2 - s$. Since $\cost(P_1) = \opt_{k_1}$, $\cost(P_2) = \opt_{k_2}$, and $\cost(P') \geq \opt_m$, Lemma \ref{lem: path-exchange} gives
    \begin{align}
        \cost(P) \leq \opt_{k_1} + \opt_{k_2} - \opt_m. \label{eq: exchange-bound}
    \end{align}
    As $P_1$ is $\tau^{\arup} \eta$-optimal and $P_2$ is $\tau^{\ardo} \eta$-optimal,
    \begin{align*}
        \opt_{k_1} + \tau^{\arup} \eta \, k_1 \leq \opt_m + \tau^{\arup} \eta \, m \quad \quad \text{and} \quad \quad \opt_{k_2} + \tau^{\ardo} \eta \, k_2 \leq \opt_s + \tau^{\ardo} \eta \, s.
    \end{align*}
    Rearranging each of these as $\opt_{k_1} - \opt_m \leq \tau^{\arup} \eta \, (m - k_1)$ and $\opt_{k_2} - \opt_s \leq \tau^{\ardo} \eta \, (s - k_2)$, then adding and using $m - k_1 = k_2 - s$ and $\tau^{\ardo} = \tau^{\arup} + 1$, we obtain
    \begin{align*}
        \opt_{k_1} + \opt_{k_2} - \opt_m - \opt_s &\leq \tau^{\arup} \eta \, (m - k_1) + \tau^{\ardo} \eta \, (s - k_2) \\
        &= (\tau^{\ardo} - \tau^{\arup}) \cdot \eta \, (s - k_2) \\
        &= \eta \, (s - k_2) < \eta s = \epsilon v.
    \end{align*}
    Thus,
    \[
        \cost(P) \leq \opt_{k_1} + \opt_{k_2} - \opt_m < \opt_s + \varepsilon v \leq (1+\varepsilon)\opt_s.
    \]
    and so the quantization set $Q$ corresponding to $P$ satisfies $|Q| = s$ and $\objtx(w, Q) = \cost(P) < (1+\epsilon) \cdot \objtx(w, s)$.

    \paragraph{Runtime.} Running Algorithm \ref{alg: exact-mix-algorithm} and computing the objective value $v$ takes $O(d \log d)$ time by Theorem \ref{thm: adv-s-approx}. Sorting $w$ and computing the prefix sums takes $O( d \log d)$ time as well by Lemma \ref{lem: efficiently-compute-C}. The binary search (lines \ref{line: binary-search-start} - \ref{line: binary-search-end}) searches over $O(s^2/\epsilon)$ many values. Each iteration of the binary search runs a \texttt{ShortestPathDAG} algorithm $O(1)$ times, each of which takes $O(d)$ time. Therefore, the entire binary search portion of Algorithm \ref{alg: 1+eps-adv-alg} has runtime $O(d \log(s/\epsilon))$. Line \ref{line: exchange-return} makes two further such calls and applies Lemma \ref{lem: path-exchange} in $O(d)$ time. The total runtime of Algorithm \ref{alg: 1+eps-adv-alg} is then $O(d \log d + d \log (s / \epsilon)) = O(d \log (d/\epsilon))$.
\end{proof}



\subsection{Practical Algorithms}
\label{sec:prac-algs}
Our practical algorithms build off the \texttt{Wilber} algorithm of \cite{GKMNSS17} (and our implementations build off their open-source code as well).
The \texttt{Wilber} algorithm proceeds by searching for a correct Lagrangian multiplier $\tau$ on which to solve a relaxed (or \textit{regularized}) version of the problem; see Section \ref{sec: ADV-exact-algorithms}.

To improve the runtime of vanilla \texttt{Wilber}, which searches for $\tau$ using a method the authors of \cite{GKMNSS17} call interpolation search, we reduce the search space needed by employing approximation algorithms to get a rough estimate of the cost.

In particular, \cite{GKMNSS17} show that the optimal value of $\tau$ occurs at $\tau = \opt_s - \opt_{s+1}$, where $\opt_s$ is the cost (of the optimal $\objt$ solution) with $s$ quantization values and $\opt_{s+1}$ the cost with $s+1$.
This immediately suggests two ways of reducing the search space and thus runtime of the algorithm:
\begin{enumerate}
    \item Run a fast, but potentially looser, approximation algorithm, and use this to upper bound $\tau$ at some value $U=\objt(w,Q)$, where $w$ is the input vector and $Q$ is the set returned by the approximation.
        Then, search the range $[0,U]$ for $\tau$.
        This is the \textit{fast approximation} technique.
    \item Run a slower, but more accurate, approximation algorithm and obtain estimates $\hat{v}_s, \hat{v}_{s+1}$ for $\opt_s, \opt_{s+1}$, respectively.
        Then, search the range $[(\hat{v}_s - \hat{v}_{s+1})/2 , 2(\hat{v} - \hat{v}_{s+1})]$ (where the factors 2 are parameters which can be increased or decreased depending on the quality of the approximation).
        This is the \textit{accurate approximation} technique.
\end{enumerate}
When using the fast approximation technique, we always construct $Q$ as the optimal solution for $\objmixx$ (see Section \ref{sec:vmix}), as it is very fast to compute and provably an $s$-approximation (Lemma \ref{lem: objmix-approximates-objt}).
The hope of using the accurate approximation technique is that the additional runtime of the approximation algorithm will be offset by a much smaller search space; for this we use both a faster implementation of approximation \texttt{QUIVER} (as discussed in Section \ref{sec:iaq}) and a new approximation algorithm \texttt{MixApprox}.

\texttt{MixApprox} proceeds similarly to approximate \texttt{QUIVER}: given a parameter $m$, we construct a subset $\calC \subseteq w$ of size $m$ and find the optimal $s$ values to select from this subset $\calC$.\footnote{When solving on this subset, we do not recurse, and instead always employ the fast approximation technique with the estimate from $\objmixx$.}
The difference lies in how $C$ is constructed: while approximate \texttt{QUIVER} uniformly spaces $C$ across the range $[\min(w), \max(w)]$, in \texttt{MixApprox}, we set $\calC$ to be the optimal solution to $\objmixx$ over $w$ with $m$ quantization points.
This \textit{adaptive} approach to constructing $\calC$ gives much better approximations, allowing us to use a smaller value of $m$ yet still obtain comparable approximation ratio.
In the accurate approximation technique, the approximation algorithm is only run once, while an exact solver is run twice on the set returned by the approximation algorithm (to estimate $\opt_s$ and $\opt_{s+1}$), so a smaller value of $m$ reduces the runtime of these calls.

We compare the fast approximation technique with the accurate approximation technique, using both \texttt{MixApprox} and a faster implementation of approximate \texttt{QUIVER}.
Due to the robust nature of \texttt{MixApprox} and the ability to run with smaller value $m$, using \texttt{MixApprox} as the approximation algorithm gives the most reliable performance across data distributions (see the supplementary information for a CSV file of performance evaluations over a range of distributions).
However, there is still a cost of running the algorithm for $\objmixx$ rather than simply uniformly spacing points, so we expect on some data, especially those which are quite uniform, using (improved) approximate \texttt{QUIVER} as the estimator will lead to better performance.

\begin{table}[ht]
    \centering
    \caption{Runtime of \texttt{Wilber} in milliseconds (ms) with different search methods, across vectors of size $d$ sampled from $\text{LogNormal}(0,1)$ and $s=64$, averaged across 10 trials. Interp.~search is the standard interpolation search method outlined in \cite{GKMNSS17}, $\objmixx$ is the fast approximation technique with solving $\objmixx$, \texttt{MixApprox} is the accurate approximation technique using \texttt{MixApprox}, and Imp.~approx \texttt{QUIVER} is the accurate approximation technique using our improved implementation of approximate \texttt{QUIVER}.
       For \texttt{MixApprox}, we use $m=4s$, and for approximate \texttt{QUIVER}, we use $m=200s$.
   }
   \medskip

    \begin{tabular}{ c c c c c }
        \hline
        $d$ & Interp.~Search & $\objmixx$ & \texttt{MixApprox} & Imp.~Apx.~\texttt{QUIVER} \\
        \hline
         100,000 &            70.27 &         72.41 &           42.49 &              43.67 \\
         200,000 &           143.33 &        144.57 &           68.63 &              82.84 \\
         300,000 &           221.92 &        205.69 &          100.32 &             114.77 \\
         400,000 &           295.17 &        328.10 &          175.08 &             169.40 \\
         500,000 &           371.57 &        378.40 &          182.93 &             180.73 \\
         600,000 &           428.94 &        474.88 &          241.17 &             236.43 \\
         700,000 &           389.62 &        521.06 &          275.22 &             270.25 \\
         800,000 &           650.12 &        622.00 &          297.00 &             301.70 \\
         900,000 &           773.60 &        663.57 &          327.45 &             332.53 \\
       1,000,000 &           906.50 &        780.26 &          349.30 &             402.50 \\
       \hline
   \end{tabular}
\end{table}



\subsubsection{Improved Approximate QUIVER}
\label{sec:iaq}
Using our exact algorithms, we are able to also speed up approximate \texttt{QUIVER} from \cite{ben2024optimal}.
Usually, $m$ is not too much larger than $s$, and so for this special case, using the accurate approximation technique described above is not faster than the fast approximation technique.
So, we warm-start \texttt{Wilber} with the solution cost of $\objmixx$; see Figure \ref{fig:improved-quiver-approx} for a runtime comparsion.

One may be interested in using \texttt{MixApprox} as an approximation algorithm, since it can obtain very good approximations with small values of $m$.
Unfortunately, the cost of computing $\objmixx$ usually negates the cost of increasing $m$, except on very skewed data.
This is in contrast to our exact algorithms, where the additional call of the exact solver (solving for not just a quantization set of size $s$ but also one of size $s+1$) allows \texttt{MixApprox} to be the more performant choice.

\section{An Improved Approximation Algorithm for Unweighted MSE} \label{sec: approx-alg-MSE}

Alongside the exact algorithm for weighted (and therefore also unweighted) MSE given by Corollary \ref{cor: exact-weighted-mse}, \cite{ben2024optimal} also present the following approximation algorithm for \textit{unweighted} MSE. 

\begin{corollary}[Lemma 6.1 of \cite{ben2024optimal}]\label{cor: approx-unweighted-mse}
    There exists an algorithm with runtime $O(d + ms)$ that returns a quantization set $Q$ such that $|Q| = 2s-2$ and $\MSE(w, Q) \leq \MSE(w, s) + d \Delta^2/m^2$. Here, $\Delta := \max_{i, j \in [d]} w_i - w_j$. 
\end{corollary}

The approximation algorithm of Corollary \ref{cor: approx-unweighted-mse} leaves much to be desired. It gives a bicriteria approximation: relaxing the reported quantization set size while guaranteeing only a weak additive approximation to the optimal mean squared error. In this section, we outline a new approximation algorithm for optimizing unweighted MSE under the standard stochastic quantization rounding distribution. We give a new \textit{data-dependent} bicriteria algorithm with stronger approximation guarantees than that of Corollary \ref{cor: approx-unweighted-mse}.

\subsection{An Improved Bicriteria Approximation Algorithm}

At a high level the bicriteria approximation algorithm of Lemma \ref{cor: approx-unweighted-mse} works by considering a candidate set of quantization points given by $m$ points uniformly spaced along the range of $w$. It then solves exactly on this candidate set to choose the best $2s-2$ points. Algorithm \ref{alg: data-dependent-obj2-approx} improves upon this by choosing the candidate quantization points in a data-dependent fashion. We employ the $s$-center clustering algorithm of Lemma \ref{lem: s-center-clustering-alg} to construct the candidate \emph{coreset}, subdivide these clustered regions evenly, then exactly solve on this new candidate set. 

The correctness of Algorithm \ref{alg: data-dependent-obj2-approx} relies on a connection between the unweighted $\MSE$ and maximum variance when rounding according to the standard stochastic quantization distribution. In particular, let $s \in \N$ be the target quantization set size, and consider quantization set $Q_{\MSE}$ of size $s$ that minimizes the mean-squared error 
\begin{align*}
    \MSE(\bwh, w) = \E \left[ \norm{\bwh - w}_2^2 \right] = \sum_{i \in [d]} \Var[\bwh_i].
\end{align*}
Then, consider quantization set $Q_{\textsf{MaxVar}}$ of size $s$ that minimizes $\max_{i \in [d]} \Var[\bwh_i]$. 
It is immediate that
\begin{align*}
     \max_{i \in [d]} \Var_{\bwh \sim \dasq(w, Q_{\textsf{MaxVar}} )}[\bwh_i] \leq \sum_{i \in [d]} \Var_{\bwh \sim \dasq(w, Q_{\MSE} )}[\bwh_i] 
\end{align*}
Recall that this can simply be written as $\objo(w, s) \leq \MSE(w, s)$. Concretely, this relation will allow us to use a ``worst-case" clustering algorithm ($s$-center clustering) and connect back to the unweighted $\MSE$. 




We note that Algorithms \ref{alg: exact-mix-algorithm} and \ref{alg: 1+eps-adv-alg} for $\objmix_{\calX}$ and $\objt_{\calX}$ can be seen to work more generally for a candidate set of quantization points $P = \{p_1, \ldots, p_m\} \subset \R$ (such that $p_1 \leq w_1$ and $p_m \geq w_m$).\footnote{As written, Algorithms \ref{alg: exact-mix-algorithm} and \ref{alg: 1+eps-adv-alg} solve the problem for $P = w$ because there always exists an optimal quantization set $Q \subseteq w$.} In particular, the modified objectives are
\begin{align*}
    \objmixx(w, s) &:= \min_{Q \subset P: |Q| \leq s} \max_{j \in [s-1]} C[i_j, i_{j+1}] \\
    \objt_{\calX}(w, s) &:= \min_{Q \subset P: |Q| \leq s} \sum_{j \in [s-1]} C[i_j, i_{j+1}]
\end{align*}
where $Q = \{p_{i_1}, \ldots, p_{i_{|Q|}} \}$ and $C[i_j, i_{j+1}] := \sum_{w_i \in [p_{i_j}, p_{i_{j+1}}]} (p_{i_{j+1}} - w_i)(w_i - p_{i_j})$.
Given $P$ and $O(|P| \log |P| + d \log s)$ preprocessing time, we can provide $O(1)$ time access to matrix $C$. In particular, we wish to compute prefix sums $\alpha, \beta, \gamma \in \R^{|P|}$
\begin{align*}
    \alpha[j] = \sum_{w_i \in [p_1, p_j]} 1 \quad \quad \quad \quad \beta[j] = \sum_{w_i \in [p_1, p_j]} w_i \quad \quad \quad \quad \gamma[j] = \sum_{w_i \in [p_1, p_j]} w_i^2
\end{align*}
    
This can be achieved by first sorting $P$, then for each $i \in [d]$ finding the unique $j \in [|P|]$ for which $w_i \in [p_j, p_{j+1}]$ (which be done in $O(d \log |P|)$ time using binary search). We then construct $A, B, \Gamma$ in each bucket 
\begin{align*}
    A[j] = \sum_{w_i \in [p_j, p_{j+1}]} 1 \quad \quad \quad \quad 
    B[j] = \sum_{w_i \in [p_j, p_{j+1}]} w_i \quad \quad \quad \quad 
    \Gamma[j] = \sum_{w_i \in [p_j, p_{j+1}]} w_i^2
\end{align*}
Then,
\begin{align*}
    \alpha[j] = \sum_{i=1}^{j-1} A[i] \quad \quad \quad \beta[j] = \sum_{i=1}^{j-1} B[i] \quad \quad \quad \gamma[j] = \sum_{i=1}^{j-1} \Gamma[i] \quad \quad \quad 
\end{align*}
The runtimes of Algorithms \ref{alg: exact-mix-algorithm} and \ref{alg: 1+eps-adv-alg} become $O(|P| \log |P| + d \log |P|)$ and $O(|P| \log (|P|/\epsilon) + d \log |P|)$ respectively.

The above approach works for any set $P$.
In Algorithm \ref{alg: data-dependent-obj2-approx}, we will use a set $P$ with special structure which allow us to construct the prefix sums faster, in $O(|P| + d\log s)$ time; see the proof of Theorem \ref{thm: obj2-2s-approx}.


\begin{algorithm}[ht]
    \caption{Data-Dependent Bicriteria Algorithm for Unweighted $\MSE(w, s)$}
    \label{alg: data-dependent-obj2-approx}  
    \KwIn{$w \in \R^d$, $s \in \N$, $\epsilon > 0$}
    \KwOut{Quantization set $Q$} 

    $\calC \gets$ $s$ clusters from \texttt{s-Center-Clustering}$(w, s)$ \\
    Let $p_i^{\textsf{max}}$ and $p_i^{\textsf{min}}$ be the max and min elements in $C_i$ for all $i \in [s]$ \\
    Initialize $P \gets \emptyset$ \\
    \For{$i \in [s]$}{
        \For{$j \in [\sqrt{4d/\epsilon}]$}{
            $P \gets P \cup \{ (p_i^{\textsf{max}} - p_i^{\textsf{min}}) \cdot \sqrt{\epsilon /4d} \cdot j\}$
        }
    }
    Return Algorithm \ref{alg: 1+eps-adv-alg} run on $w, P, s, \epsilon$ with uniform weights
\end{algorithm}

The proof of Theorem \ref{thm: obj2-2s-approx} closely follows that of Lemma 6.1 in  \cite{ben2024optimal}. The key difference is in analyzing the quality of the data-dependent coreset and its implications for the ultimate approximation ratio. 

\begin{theorem} \label{thm: obj2-2s-approx}
    There exists an algorithm such that given $w\in \R^d$, target quantization set size $s \in \N$, and parameter $\epsilon > 0$ returns a quantization set $Q \subset \R$ such that $|Q| = 2s-2$ and $\MSE(w, Q) \leq (1+\epsilon) \MSE(w, s)$. The runtime of the algorithm is $\smash O (d \log s + s  \sqrt{d/ \epsilon} \log(d/\epsilon) )$.
\end{theorem}

\begin{proof}
    We will show the existence of a quantization set $Q \subseteq P$ of size $2s-2$ such that $\MSE(w, Q) \leq (1+\epsilon) \cdot \MSE(w, s)$. 

    In particular,  consider quantization set $Q^{\ast}$ of size $s$ which minimizes the mean-squared error. For any $q \in Q^{\ast}$, let $q^{\ardo} := \min \{ p \in P : p \geq q\}$ and $q^{\arup} := \max \{ p \in P: p \leq q\}$. The solution $Q := \{q^{\arup}, q^{\ardo} :  q \in Q\}$ is such that (i) $|Q| = 2s-2$ and (ii) $\MSE(w, Q) \leq (1+\epsilon) \MSE(w, s)$. To see that $|Q| = 2s-2$ note that $w_1, w_d \in Q^{\ast}$ but $q^{\ardo} = q^{\arup}$ for both. 

    To see that $\MSE(w, Q) \leq (1+\epsilon) \MSE(w, s)$, consider some element $w_i$. There are two cases
    \begin{enumerate}
        \item First, suppose $w_i \in [(\widown(Q^{\ast}))^{\arup}, (\wiup(Q^{\ast}))^{\ardo}]$. In this case, 
        \begin{align*}
            \Var_{\bwh \sim \dasq(w, Q)} [\bwh_i] &= ((\wiup(Q^{\ast}))^{\ardo} -  w_i)(w_i - (\widown(Q^{\ast}))^{\arup}) \\
            &\leq (\wiup(Q^{\ast}) - w_i)(w_i - \widown(Q^{\ast})) = \Var_{\bwh \sim \dasq(w, Q^{\ast})} [\bwh_i]
        \end{align*}
        
        \item Otherwise, $w_i \in [(\widown(Q^{\ast}))^{\ardo}, (\widown(Q^{\ast}))^{\arup}] \cup [(\wiup(Q^{\ast}))^{\ardo}, (\wiup(Q^{\ast}))^{\arup}]$. In this case, $w_i$ is quantized between points from $P$. The original $s$-clustering given by the call to $\texttt{s-Center-Clustering}$ results in clusters with radius $\leq 4 \sqrt{\objo(w, s)}$. Algorithm \ref{alg: data-dependent-obj2-approx} then further subdivides and refines the granularity of $P$ to be $\leq 4 \sqrt{\objo(w, s) \cdot \epsilon / 4d} $. Therefore, $w_i$  is quantized in an interval of length $\leq 4 \sqrt{\objo(w, s) \cdot \epsilon / 4d} $ and its variance is then $\leq \objo(w, s) \cdot \epsilon /d \leq \objt(w, s) \cdot \epsilon /d$. 
    \end{enumerate}
    See Figure \ref{fig:casework} for an illustration of the two cases. Taking the sum of variances across all $d$ coordinates then shows that the resulting mean-squared error is at most $(1 + \epsilon) \MSE(w, s)$.

    Therefore, any solution found by running Algorithm \ref{alg: 1+eps-adv-alg} with candidate quantization points $P$ will have objective value $(1+\epsilon)^2 \cdot \MSE(w, s)$. Thus, running Algorithm $\ref{alg: data-dependent-obj2-approx}$ with accuracy parameter $\epsilon' = \epsilon/3$ results in a $(1+\epsilon)$-approximation algorithm.

    \paragraph{Runtime.} Finding the $2$-approximate clustering using $\texttt{s-Center-Clustering}$ takes $O(d \log s)$ time as  per Lemma \ref{lem: s-center-clustering-alg}. Computing the minimum and maximum of each cluster takes $O(d)$ time and constructing the coreset $P$ by uniform subdivision in each cluster takes $\smash  O(s \sqrt{d/\epsilon})$ time. 
    
    Finally, note that the prefix sums can be computed in time $O(|P| + d \log s)$ time (as opposed to the $O(|P| \log |P| + d \log |P|)$ bound discussed before) by exploiting the structure of coreset $P$. 
    In constructing $P$, we may first sort the $s$ centers returned by \texttt{s-Center-Clustering}$(w, s)$, and then uniformly subdivide.
    Thus, we have constructed $P$ in sorted order in $O(|P| + s \log s) = O(s \sqrt{d / \epsilon})$ time.
    To compute the prefix sums, for each $i \in [d]$, we can find the unique $j \in [s \sqrt{4d / \epsilon}]$ for which $w_i \in [p_j, p_{j+1}]$ by binary searching among the $s$ sorted centers, then performing arithmetic to find the subdivision inside which $w_i$ lies. 
    This thus improved the construction time of the prefix sums to $O(|P| + d\log s)$.
    
    So, Algorithm $\ref{alg: 1+eps-adv-alg}$ runs in $O(d\log s + s\sqrt{d/\epsilon} \log(s d / \epsilon))$ time. Therefore, the total runtime of Algorithm \ref{alg: data-dependent-obj2-approx} is $O(d \log s + s \sqrt{d/\epsilon} \log(d/\epsilon))$. 
\end{proof}

\begin{figure}[ht]
    \centering
    \begin{tikzpicture}[
        x=0.8cm, 
        dot/.style={circle, fill=#1, inner sep=1.2pt}
      ]
    
      \begin{scope}[xshift=0cm]
          \draw[<->, thick] (0.2, 0) -- (7.8, 0);
        
          \foreach \x/\pcolor/\plabel/\pos in {
            1/blue/(\widown(Q^{\ast}))^{\ardo}/below, 
            2/red/\widown(Q^{\ast})/above, 
            3/blue/(\widown(Q^{\ast}))^{\arup}/below, 
            4/black/w_i/above, 
            5/blue/(\wiup(Q^{\ast}))^{\ardo}/below, 
            6/red/\wiup(Q^{\ast})/above, 
            7/blue/(\wiup(Q^{\ast}))^{\arup}/below%
          } {
            \node[dot=\pcolor, label=\pos:{\scriptsize $\plabel$}] at (\x, 0) {};
          }
      \end{scope}

      \begin{scope}[xshift=6.5cm] 
          \draw[<->, thick] (0.2, 0) -- (7.8, 0);
        
          \foreach \x/\pcolor/\plabel/\pos in {
            1/blue/(\widown(Q^{\ast}))^{\ardo}/below, 
            2/red/\widown(Q^{\ast})/above, 
            3/black/w_i/below, 
            4/blue/(\widown(Q^{\ast}))^{\arup}/above, 
            5/blue/(\wiup(Q^{\ast}))^{\ardo}/below, 
            6/red/\wiup(Q^{\ast})/above, 
            7/blue/(\wiup(Q^{\ast}))^{\arup}/below%
          } {
            \node[dot=\pcolor, label=\pos:{\scriptsize $\plabel$}] at (\x, 0) {};
          }
      \end{scope}

    \end{tikzpicture}
    \caption{Illustration of (a) case 1 and (b) case 2}
    \label{fig:casework}
\end{figure}
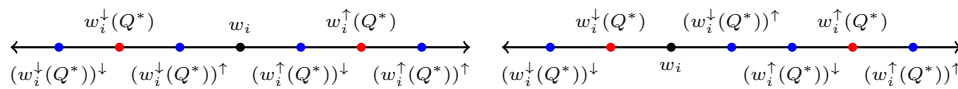

We remark that Algorithm \ref{alg: data-dependent-obj2-approx} is faster than running the exact algorithm on the full instance in regimes where $s$ and $\epsilon$ are modestly valued. As mentioned in Appendix \ref{sec: technical-prelims}, however, the use of \texttt{s-Center-Clustering} as a subroutine makes Algorithm \ref{alg: data-dependent-obj2-approx} difficult to implement efficiently in practice.

\section{Maximum Directional Variance} \label{sec: mdv}

For any $v\in \mathbb{R}$, we define $\objos(w,v)$ to be the smallest value of $s$ for which $\objo(w,s)\leq v$. We first study $\objos$ and give both exact and approximation algorithms. We then use many of these algorithmic primitives and ideas to develop algorithms with small asymptotic runtime and guaranteed $(1+\varepsilon)$-approximations for $\objo$.

\subsection{\texorpdfstring{Exact Algorithms for $\objos$}{}}

We first give a simple greedy exact algorithm for $\objos$ that runs in time $O(d)$ {on sorted vectors}.


\begin{algorithm}[ht]
    \caption{Exact Algorithm for $\objos(w,v)$}
    \label{alg:exact-algorithm}
    \KwIn{Sorted vector $w\in \mathbb{R}^{d}$ so that $w_1\leq \ldots \leq w_d$, variance parameter $v$}
    \KwOut{Quantization set $Q$}

    Initialize $Q=\{w_1\}$, $q_\ell = w_1$ and $q_r = \infty$ \\
    \For{$i = 2, \ldots, d-1$}{
        $q_r = \min \{q_r,\; w_i + v/(w_i - q_\ell)\}$ \\
        \If{$q_r \leq w_i$}{
            Add $q_r$ to $Q$ \\
            Set $q_{\ell} = q_r$ and update $q_r = w_i + v/(w_i - q_\ell)$ \label{line:greedy-update} \\
        }
    }
    \eIf{$q_r < w_d$}{
        Add $q_r$ and $w_d$ to $Q$ \\
    }{
        Add $w_d$ to $Q$ \\
    }
    \Return{$Q$}
\end{algorithm}


\begin{lemma}
    \label{lem:objos-exact-analysis}
    If $w\in \mathbb{R}^{d}$ is sorted, Algorithm \ref{alg:exact-algorithm} solves $\objos$ in time $O(d)$.
\end{lemma}
\begin{proof}
    First, we note that any quantization set $Q$ must contain the endpoints $w_1$ and $w_d$.
    Since ASQ requires that every $w_i$ have a quantization point on each side, $Q$ must contain some $q_1 \leq w_1$ and $q_{|Q|} \geq w_d$.
    Replacing $q_1$ with $w_1$ and $q_{|Q|}$ with $w_d$ does not increase $\varasq(w_i, Q)$ for any $i$, so the objective value $\max_{i\in [d]}\varasq(w_i, Q)$ is non-increasing under this substitution.

    Second, we prove that the set $Q$ returned by Algorithm \ref{alg:exact-algorithm} has objective value at most $v$.
    For each weight $w_i$, Line \ref{line:greedy-update} computes the furthest right a new quantization point $q_r$ can be placed while maintaining $\varasq(w_j, Q) \leq v$ for all $j$ with $w_j < q_r$.
    Taking the running minimum over these upper bounds ensures the constraint is satisfied for every weight $w_i$.

    Lastly, we prove that any quantization set $Q'$ with $\objo(w, Q') \leq v$ must satisfy $|Q'| \geq |Q|$.
    Suppose for contradiction that $\objo(w, Q') \leq v$ but $|Q'| < |Q|$.
    Since $Q'$ is feasible, it must place its second point no further right than $q_2$, i.e., $q_2' \leq q_2$; otherwise the constraint would be violated for some weight between $q_1' = w_1$ and $q_2'$.
    Applying this argument inductively gives $q_i' \leq q_i$ for all $i \leq |Q'|$, so the $|Q'| < |Q|$ points of $Q'$ fail to cover all weights in $w$, contradicting feasibility.

    The algorithm makes a single pass over the sorted weights, performing $O(1)$ work per step, so the total runtime is $O(d)$.
\end{proof}

We note that the runtime of Algorithm \ref{alg:exact-algorithm} can be improved to $O(s\log(d/s))$; for all $s\leq d$, this is an improvement over the $O(d)$ runtime of Algorithm \ref{alg:exact-algorithm}.
The algorithm is as follows.

\begin{algorithm}[ht]
    \caption{Improved Algorithm for $\objos(w,v)$}
    \label{alg:faster-exact}
    \KwIn{Sorted $w\in \mathbb{R}^{d}$, variance parameter $v$}
    \KwOut{Set $Q$ such that $\objo(w,Q)\leq v$ and $|Q| = \objos(w,v)$}

    Initialize $Q = \{w_1\}$, the smallest element of $w$ \\
    Let $q = w_1$ be the max element of $Q$ \\
    \While{$q < w_d$}{
        Binary search for the smallest index $i$ such that $w_i \geq q + \sqrt{v}$ \\
        Set $x = \min \left\{ w_i + \dfrac{v}{w_i - q},~ w_{i-1} + \dfrac{v}{w_{i-1} - q} \right\}$ \\
        Add $x$ to $Q$ and set $q = x$ \\
    }
    Add $w_d$ to $Q$ \\
    \Return{$Q$}
\end{algorithm}



\begin{lemma}
    \label{lem:improved-check}
    Algorithm \ref{alg:faster-exact} outputs a set $Q$ such that $\objo(w,Q)\leq v$ and $|Q| = \objos(w,v)$.
\end{lemma}
\begin{proof}
    The goal of the algorithm is to implement the greedy strategy: given the current largest quantization point $q$, find the largest possible next point $x^{*} > q$ such that the variance constraint is satisfied for all $w_j \in [q,x^{*}]$.
    Let $x' = 2\sqrt{v} + q$; for all $w_k \in [q,x']$, we have
    \[
        (w_k - q)(x' - w_k) \leq \left( \frac{x'-q}{2} \right)^2 = v
    \]
    and thus $x'$ is a feasible next quantization point and $x^{*}\geq x'$.
    Define $m^{*} = (x^{*} + q)/2$ to be the midpoint of $[q,x^{*}]$; since $x^{*}\geq x'$, it follows that $m^{*} \geq q+\sqrt{v}$ as well.

    Let $w_i$ be the smallest element of $w$ such that $w_i \geq q+ \sqrt{v}$, and consider any $y$ such that $(y+q)/2 > w_i$.
    Then,
    \[
        (w_i - q)(y - w_i) > (w_i - q)(2w_i - q - w_i) = (w_i - q)^2 \geq v
    \]
    so $y$ is not a feasible next quantization point.
    Thus, $x^{*}\in[x',y]$ and $m^{*}\in [w_{i-1}, w_i]$ as $(x'+q)/2 = q+\sqrt{v} > w_{i-1}$ by construction of $i$ and $(y+q)/2 > w_i$ by definition of $y$.

    Since the function $f(z)=(z-q)(x^{*}-q)$ is concave and symmetric around $(x^{*}+q)/2=m^{*}$ and $m^{*}\in[w_{i-1}, w_i]$, the maximum variance of all $w_k$ in $[q,x^{*}]$ is attained at either $w_{i-1}$ or $w_i$.
    Thus, $x^{*}$ is the largest value for which both $w_{i-1}$ and $w_i$ have variance at most $v$, which is exactly what the algorithm computes.

    The size analysis then follows by the same argument as in the proof of Lemma \ref{lem:objos-exact-analysis}.
\end{proof}

\newcommand{\idx}{\mathsf{idx}}
\begin{lemma}
    \label{lem:improved-check-runtime}
    Let $s=\objos(w,v)$.
    Then, the above algorithm runs in time $O(s\log(d/s))$ (assuming $w$ is already sorted).
\end{lemma}
\begin{proof}
    Let $Q$ be the set returned by the algorithm, which has size $|Q|\leq s$ by Lemma \ref{lem:improved-check}.
    For each $k\in [s]$, let $\idx_k$ denote the index $i$ found in the $k$-th iteration of the while loop of Algorithm \ref{alg:faster-exact}.

    Since $q$ increases monotonically, the indices $\idx_k$ also increase monotonically across iterations.
    By using the doubling search variant of binary search starting from the previous index $\idx_{k-1}$, the time complexity to find $\idx_k$ is $O(\log(\idx_k - \idx_{k-1}))$.
    So, the total runtime to construct $Q$ is $O\left(\log \prod_{k=1}^s B_k\right)$, where $B_k = \idx_k - \idx_{k-1}$.
    By construction, $\sum_{k=1}^s B_k \leq d$.
    Then, by the AM-GM inequality,
    \[
        \prod_{k=1}^s B_k \leq \left( \frac{\sum_{k=1}^s B_k}{s} \right)^s = (d/s)^s
    \]
    and thus the total runtime is $O(\log ((d/s)^{s})) = O(s\log(d/s))$.
\end{proof}

\subsection{Some Basic Approximation Algorithms for \texorpdfstring{$\objos$}{}}
In this section, we give algorithms which do not require the $O(d\log d)$-time sorting required for exact algorithms for $\objos$.
These algorithms each run in time $O(d\log s)$ and achieve a constant approximation factor.



\subsubsection{\texorpdfstring{$(2s,4v)$-Bicriteria via $s$-Center Clustering}{}}

\begin{algorithm}[H]
    \caption{$(2s,4v)$-Bicriteria Approximation Algorithm}
    \label{alg:obos-bicriteria}
    \KwIn{Vector $w\in \mathbb{R}^{d}$, variance parameter $v$}
    \KwOut{quantization set $Q$}

    Initialize $C=\emptyset$ \\
    \For{$i \in [d]$}{
        \If(\label{line:inclusion-check}){for all $p\in C$, $|w_i - p| > 2\sqrt{v}$}{
            Add $w_i$ to $C$ \\
        }
    }
    Let $Q$ be the set which contains $p - 2\sqrt{v}$ and $p + 2\sqrt{v}$ for all $p\in C$\\
    \Return{$Q$}
\end{algorithm}

\begin{lemma}
    \label{lem:bicriteria}
    Given $w\in \mathbb{R}^{d}$ and $v\in \mathbb{R}$, let $s=\objos(w,v)$.
    Then, the above algorithm outputs a set $Q$ of size at most $2s$ such that $\objo(w,Q)\leq 4v$.
    Moreover, with an appropriate implementation of Line \ref{line:inclusion-check}, the algorithm runs in time $O(d\log s)$.
\end{lemma}
\begin{proof}
    Since $\objo(w,s)\leq v$ by definition of $\objos$, there exists a set $Q^{*}$ of size $s$ such that for all $i$
    \[
        (w_i - \widown(Q^{*}))(\wiup(Q^{*}) - w_i) \leq v
    \]
    which implies that either $w_i - \widown(Q^{*})\leq \sqrt{v}$ or $\wiup(Q^{*}) - w_i \leq \sqrt{v}$.
    Equivalently, every $w_i$ is within $\sqrt{v}$ of some element of $Q^{*}$ and the set $\{w_i\}_{i\in [d]}$ can be covered with a collection of at most $s$ intervals of length $2\sqrt{v}$.

    Let $I_1,\ldots ,I_s$ be a collection of disjoint intervals of length $2\sqrt{v}$ which cover $\{w_i\}_{i\in [d]}$, and let $C$ be as constructed in Algorithm \ref{alg:obos-bicriteria}.
    By construction, for any distinct $x,y\in C$, $|x-y| > 2\sqrt{v}$, and since each $I_j$ has length $2\sqrt{v}$, we thus must have that $|I_j \cap C|\leq 1$ for all $j\in [s]$.
    Moreover, since $\{I_j\}_j$ covers $\{w_i\}_i$ and $C\subseteq \{w_i\}_i$, we have that $|C|= \sum_{j\in [s]} |I_j \cap C| \leq s$.
    Since $|Q|= 2|C|$, we thus have that $|Q|\leq 2s$ as desired.

    By construction, for all $i$, there exists a $p\in C$ such that $|w_i - p| \leq 2\sqrt{v}$.
    Equivalently, $w_i \in [p - 2\sqrt{v}, p + 2\sqrt{v}]$.
    Since both $p - 2\sqrt{v}$ and $p + 2\sqrt{v}$ are in $Q$, it follows that
    \[4\sqrt{v} \geq \wiup(Q) - \widown(Q) = (w_i - \widown(Q)) + (\wiup(Q) - w_i).\]
    For any $x,y \geq 0$, $xy \leq (x+y)^2/4$, and thus it follows that
    \[
        (w_i - \widown(Q))(\wiup(Q) - w_i )\leq 4v
    \]
    as desired.
\end{proof}



\subsection{\texorpdfstring{$(1+\varepsilon)$-approx for $\objos$ in time $O(d\log(s/\varepsilon))$}{}}

We first need a key structural lemma which can be viewed as a way of constructing a ``coreset'' for $\objo$.
\begin{lemma}
    \label{lem:approx-by-endpoints}
    Let $I_1,\ldots ,I_k$ be intervals, each of length at most $D$, such that $w\subseteq \bigcup_{j\in [k]} I_j$.
    Let $x$ be the vector of all endpoints of the intervals $I_1,\ldots ,I_k$, and consider any set $Q\subseteq \mathbb{R}$.
    Then,
    \[
        \objo(w,Q)\leq \objo(x,Q) + \frac{1}{4}\cdot D^2
    \]
\end{lemma}
\begin{proof}
    Let $v=\objo(x,Q)$, and fix some interval $I_j=[a,b]$.
    We will show that $\varasq(w_i, Q)\leq v + D^2/4$ for all $w_i\in [a,b]$, which then gives the lemma.

    First, suppose that $Q\cap [a,b]\neq \emptyset$.
    Let $q_1=\min(Q\cap [a,b])$ and $q_2=\max(Q\cap [a,b])$ (note $q_1=q_2$ if $|Q\cap [a,b]|=1$).
    For all $w_i \in [q_1,q_2]$, there exists some $\alpha\in [0,1]$ such that $w_i - q_1 \leq \alpha D$ and $q_2 - w_i \leq (1-\alpha)D$, as $q_2 - q_1 \leq b - a \leq D$.
    Thus,
    \[
        \varasq(w_i, Q) \leq \alpha D \cdot (1-\alpha)D \leq \frac{1}{4}\cdot D^2
    \]
    as $\alpha(1-\alpha)\leq 1/4$ for all $\alpha\in [0,1]$.

    So, consider $w_i \in [a,q_1)$.
    Let $B = q_1 - a$ and $L=a-a^{\ardo}(Q)$, where $a^{\ardo}(Q)$ is the largest element of $Q$ which is at most $a$.
    By assumption on $Q$, $BL = \varasq(a, Q)\leq v$.
    For all $w_i \in [a,q_1)$, by definition of $B$, there exists some $\alpha\in[0,1)$ such that $w_i = a + \alpha B$.
    So,
    \begin{align*}
        \varasq(w_i, Q) &= (L + \alpha B)\cdot (1-\alpha)B \\
                      &= (1-\alpha)BL + \alpha(1-\alpha)B^2 \\
                      &\leq BL + \alpha(1-\alpha)D^2 \\
                      &\leq v + \frac{1}{4}\cdot D^2
    \end{align*}
    as $BL\leq v$, $B\leq D$ and $\alpha(1-\alpha)\leq 1/4$ for all $\alpha\in [0,1]$.

    An analogous argument shows that for all $w_i \in (q_2,b]$, $\varasq(w_i, Q)\leq v + D^2/4$, thus completing the case where $Q\cap [a,b]\neq \emptyset$.

    Now, suppose $Q\cap [a,b] = \emptyset$, and let $L_1=a - a^{\ardo}(Q)$, $L_2 = b^{\arup}(Q) - b$, where $b^{\arup}(Q)$ is the smallest element of $Q$ which is at least $b$.
    Let $D_j = b-a$.
    Then,
    \begin{align*}
        (D_j + L_1)L_2 &= \varasq(b, Q) \leq v \\
        (D_j + L_2)L_1 &= \varasq(a, Q) \leq v
    \end{align*}
    Together, this implies that $D_jL + L_1L_2 \leq v$, where $L=\max(L_1,L_2)$.

    Fix some $w_i \in [a,b]$ and set $\alpha\in [0,1]$ so that $w_i = a + \alpha D_j$ (and thus $w_i = b - (1-\alpha)D_j$).
    Then,
    \begin{align*}
        \varasq(w_i, Q) &= (\alpha D_j + L_1)\cdot((1-\alpha)D_j + L_2) \\
                      &= \alpha D_j L_2 + (1-\alpha) D_j L_1 + L_1L_2 + \alpha(1-\alpha)D_j^2 \\
                      &\leq \alpha D_j L + (1-\alpha) D_j L + L_1L_2 + \alpha(1-\alpha)D^2\\
                      &= (D_jL + L_1L_2) + \alpha(1-\alpha)D^2 \\
                      &\leq v + \frac{1}{4}D^2
    \end{align*}
    as $D_jL + L_1L_2 \leq v$, $D_j \leq D$ and $\alpha(1-\alpha)\leq 1/4$.
\end{proof}

Lemma \ref{lem:approx-by-endpoints} suggests the following approximation algorithm: construct a small number of intervals, and run an exact algorithm for $\objos$ on only the endpoints of those intervals.

\begin{algorithm}[H]
    \caption{$(1+\varepsilon)$-Approximation Algorithm}
    \label{alg:ptas-objos}
    \KwIn{Vector $w\in \mathbb{R}^{d}$, variance parameter $v$, error tolerance $\varepsilon$}
    \KwOut{Set $Q\subseteq \mathbb{R}$ such that $\objo(w,Q)\leq (1+\varepsilon)v$ and $\objo(w, |Q|)\geq v$}

    Initialize $C=\emptyset$ \\
    Iterate through $i\in [d]$ and add $w_i$ to $C$ if for all $c\in C$, $|w_i - c| > 2\sqrt{v}$ \\
    \For{all $c \in C$}{
        Compute $I_c = \{w_i \mid c=\argmin_{p\in C}|w_i - p|\}$ \\
        Compute $a_c = \min(I_c)$, $b_c = \max(I_c)$, and $D_c = b_c - a_c$ \\
        Compute subintervals $I_c^{j}=[a_c + D_c (j-1)\sqrt{\varepsilon}/2,\; a_c + D_c j\sqrt{\varepsilon}/2]$ for each $j\in \left\lceil 2/\sqrt{\varepsilon} \right\rceil$ \\
    }
    \For{each $c\in C,\; j\in \left\lceil 2/\sqrt{\varepsilon} \right\rceil$}{
        Compute endpoints $\ell_c^{j} = \min(w\cap I_c^{j})$ and $r_c^{j}=\max(w\cap I_c^{j})$ \\
    }
    Let $x$ be a vector of all $\ell_c^{j}$ and $r_c^{j}$.  Run Algorithm \ref{alg:faster-exact} on $x$ and $v$ to get a set $Q$ \label{line:coreset-constr} \\
    \Return{$Q$}
\end{algorithm}

\begin{lemma}
    \label{lem:ptas-quality}
    Given $w,v,\varepsilon$, Algorithm \ref{alg:ptas-objos} outputs a set $Q$ such that $\objo(w,Q)\leq (1+\varepsilon)v$ and $|Q|\leq \objos(w,v)$.
\end{lemma}
\begin{proof}
    We first show $|Q|\leq \objos(w,v)$.
    By the analysis of the exact algorithm, we have $|Q|\leq \objos(x,v)$.
    Furthermore, by construction, $x\subseteq w$, and so $|Q|\leq \objos(x,v)\leq \objos(w,v)$ as desired.

    By Lemma \ref{lem:improved-check}, $\objo(x,Q)\leq v$.
    $x$ consists of the endpoints of the set of intervals $\{I_c^{j}\}$, which cover $w$ by construction.
    Moreover, each interval $I_c^{j}$ has length at most $\sqrt{\varepsilon} D_c/2 \leq 2\sqrt{\varepsilon v}$, as $D_c \leq 4\sqrt{v}$ by the construction of $C$.

    So applying Lemma \ref{lem:approx-by-endpoints}, we have that
    \[
        \objo(w,Q) \leq v + \frac{1}{4}\cdot (2\sqrt{\varepsilon v})^2 = (1+\varepsilon)v
    \]
    as desired.
\end{proof}

\begin{lemma}
    \label{lem:ptas-runtime}
    Given $w,v,\varepsilon$, let $s=\objos(w,v)$.
    Then, the runtime of Algorithm \ref{alg:ptas-objos} is
    \[
        O\left(d\log(s/\varepsilon)\right).
    \]
\end{lemma}
\begin{proof}
    Let $C, I_c$ for each $c\in C$ be defined as in Algorithm \ref{alg:ptas-objos}.
    By Lemma \ref{lem:bicriteria}, $|C| \leq s$ and takes time $O(d\log s)$ to construct.
    Similarly, computing all $I_c$ can be done in time $O(d\log s)$ by iterating through the $d$ points of $w$ and, for each, using binary search to find the closest $c\in C$.

    To compute $a_c,b_c$, we iterate through each of the $d$ elements of $w$ and maintain the min and max elements of $I_c$ encountered so far, for all $c\in C$.
    Checking which $I_c$ a given $w_i$ lies in takes time $O(\log s)$; updating the stored min/max if needed for $I_c$ can then be done in $O(1)$ time.
    Similarly, we can compute all $O(s/\sqrt{\varepsilon})$ values $\ell_c^{j},r_c^{j}$, and thus the vector $x$ in Line \ref{line:coreset-constr}, in time $O(d\log s)$ (since the subintervals $I^{j}_c$ are equally spaced in $I_c$, for each $y\in I_c$, finding the subinterval in which it lies takes $O(1)$ time).

    Finally, sorting and running Algorithm \ref{alg:faster-exact} takes time $O(|x|\log(|x|))$.
    Note that $|x|\leq d$, since it is a subset of $w$, and also $|x|= O(s/\sqrt{\varepsilon})$, by construction.
    So, $O(|x|\log|x|)= O(|x|\log(s/\sqrt{\varepsilon}))=O(d\log(s/\varepsilon))$.

    Thus the total time is $O(d\log(s/\varepsilon))$.
\end{proof}

\subsection{Main results for \texorpdfstring{$\objo$}{MDV}}
In this section, we give algorithms for constructing a quantization set of size $s$ with maximum variance approximately $\objo(w,s)$; in other words, approximation algorithms for $\objo(w,s)$.

Combined with algorithms from the previous section (i.e.~approximation algorithms for $\objos$), we are able to give algorithms with matching runtime of $O(d\log(s/\varepsilon))$.

\MDVApproxAlg

\subsection{\texorpdfstring{A 4-approximation to $\objo$}{}}
An important subroutine of our algorithm for Theorem \ref{thm:obj1-full-algo} is a fast 4-approximation, which we will run on a small coreset $u\in \mathbb{R}^{k}$, that gives a rough (i.e.~$O(1)$-approximate) estimate of the true optimal cost $\objo(u,s)$. In this subsection, we prove the following lemma.\footnote{The algorithm is described as a Las Vegas algorithm, i.e.~the runtime is probabilistic while the quality guarantee is deterministic.  This can be easily converted to a Monte Carlo algorithm with deterministic runtime and probabilistic quality by simply terminating the algorithm if it runs for too many steps.}.
\begin{lemma}
    \label{lem:4-approx}
    For any $u\in \mathbb{R}^{k}$, $s\in \mathbb{N}$ and $\delta\in(0,1)$, there exists an algorithm which outputs a 4-approximation to the value $\objo(u,s)$.
    Moreover, with probability at least $1-\delta$, the algorithm runs in time
    \[O(k\log k + \sqrt{k}\cdot \log^{3} k \cdot \log(1/\delta)).\]
\end{lemma}

Our first observation is that, in order to obtain a 4-approximation, it suffices to consider quantization sets which only include points from the input vector.
\begin{lemma}
    \label{lem:approx-by-input}
    Consider any $u\in \mathbb{R}^{k}$ and $Q\subset \mathbb{R}$.
    There exists a $Q'\subset \{u_i\}_{i\in [k]}$ such that $|Q'|=|Q|$ and $\objo(u,Q')\leq 4\objo(u,Q)$.
\end{lemma}
\begin{proof}
    Reindex so that $u_1 \leq u_2 \leq \ldots \leq u_k$.
    Construct $Q'$ to be $Q$, but with each element $q\in Q$ replaced by the nearest element of $\{u_i\}_{i\in [k]}$ to $q$ (breaking ties arbitrarily).
    By construction, $|Q'|=|Q|$.

    Consider some $i$ such that $u_i^{\ardo}(Q)\not\in Q'$, and let $q'=u_j$ be the nearest element of $\{u_i\}_{i\in [k]}$ to $u_i^{\ardo}(Q)$ (so, $q'\in Q'$).
    Since, by construction, there are no elements of $\{u_i\}$ in $(u_i^{\ardo}(Q),q')$, it follows that $q' \leq u_i$.

    First, suppose $q' > u_i^{\ardo}(Q)$.
    Then, as $u_i\geq q'$, $u_i - u_i^{\ardo}(Q') \leq u_i - q' < u_i - u_i^{\ardo}(Q)$.

    So, now consider the case where $q' < u_i^{\ardo}(Q)$, and let $u_j = q'$.
    Then, $u_i^{\ardo}(Q)$ is closer to $u_j$ than $u_{j+1}$, and so
    \[
        (u_i - q') - (u_i - u_i^{\ardo}(Q)) = |u_i^{\ardo}(Q) - u_j| \leq |u_i^{\ardo}(Q) - u_{j+1}| = u_{j+1} - u_i^{\ardo}(Q)
    \]
    As there are no elements of $\{u_i\}$ in $(u_j,u_{j+1})=(q',u_{j+1})$ by construction, it follows that $u_i \geq u_{j+1}$.
    So, $u_{j+1} - u_i^{\ardo}(Q) \leq u_i - u_i^{\ardo}(Q)$ and it follows from the above inequality that
    \[
        u_i - u_i^{\ardo}(Q') \leq u_i - q' \leq (u_i - u_i^{\ardo}(Q)) + (u_{j+1} - u_i^{\ardo}(Q)) \leq 2(u_i - u_i^{\ardo}(Q)).
    \]

    Therefore, for all $i$, $u_i - u_i^{\ardo}(Q') \leq 2(u_i - u_i^{\ardo}(Q))$.
    An analogous argument shows that for all $i$, $u^{\arup}_i(Q') - u_i \leq 2(u^{\arup}_i(Q) - u_i)$, and so for all $i$,
    \[
        (u_i - u_i^{\ardo}(Q'))(u^{\arup}_i(Q') - u_i) \leq 4(u_i - u_i^{\ardo}(Q))(u^{\arup}_i(Q) - u_i)
    \]
    implying $\objo(u,Q')\leq 4\objo(u,Q)$.
\end{proof}

\subsubsection{Sorted Matrices and Algorithms}
Recall from Definition \ref{def: sorted-matrix} that a sorted matrix $M$ is such that all rows and columns are sorted in non-decreasing order, i.e.~for all $i,j$, $M[i,j] \geq M[i+1, j]$ and $M[i,j] \leq M[i,j+1]$.

Sorted matrices have several useful properties, one of which is that we can efficiently determine all elements of the matrix which lie in a given interval.

\begin{lemma}
    \label{lem:staircase}
    Given any sorted matrix $M$ of dimension $k\times k$ and interval $(v_1, v_2)\subset \mathbb{R}$, there exist indices $\ell_1,\ldots ,\ell_k$ and $r_1,\ldots, r_k$ such that for all $i\in [k]$, $M[i,j]\in (v_1, v_2)$ if and only if $j\in [\ell_i, r_i]$ (or $\ell_i=r_i = \bot$ if no such elements exist).
    Moreover, there is an $O(k)$ time algorithm to compute $\ell_1,\ldots ,\ell_k$ and $r_1,\ldots ,r_k$, which probes $O(k)$ elements of the matrix $M$.

\end{lemma}
\begin{proof}
    We show how to compute the values $\ell_i$; the $r_i$ can be computed analogously by finding the largest element in each row which is strictly less than $v_2$.
    The algorithm is as follows\footnote{Note that if no element of $M$ is strictly greater than $v_1$, then the algorithm will return $\ell_i = k$; this is not a problem, as the final indices will check if $M[i,\ell_i]$ is in $(v_1,v_2)$ and set $\ell_i = r_i = \bot$ if not.}:
    \begin{enumerate}
        \item Initialize $j=1$.
        \item For each $i\in [k]$:
            \begin{enumerate}
                \item While $j\leq k$ and $M[i,j] \leq v_1$, increment $j$.
                \item Set $\ell_i = j$.
            \end{enumerate}
    \end{enumerate}
    Note that the while loop increments $j$ at most $k$ times over the entire algorithm, so the total time is $O(k)$ (and only $O(k)$ elements of $M$ are probed).

    For correctness, fix some row $i$.
    By construction of the algorithm and the sortedness of the matrix, it follows that $M[i,\ell_i] > v_1$, and thus for all $j \geq \ell_i$, $M[i,j] > v_1$ as well.
    Suppose for contradiction that there exists some $j < \ell_i$ such that $M[i,j] > v_1$.
    By construction of the algorithm, we then must have had $M[i',j] \leq v_1$ for some row $i' < i$; however, by sortedness of $M$, we have $M[i',j] \geq M[i,j] > v_1$, a contradiction.

    Finally, for any $i$ where $\ell_i = r_i$ but $M[i,\ell_i] \not \in (v_1,v_2)$, set $\ell_i = r_i = \bot$.
\end{proof}

Given such $\ell_1,\ldots ,\ell_k$ and $r_1,\ldots ,r_k$, we can also efficiently sample a uniformly random element of $M$ which lies in $(v_1, v_2)$.
\begin{lemma}
    \label{lem:staircase-sample}
    Let $M$ be a sorted matrix of dimension $k\times k$, and let $\ell_1,\ldots ,\ell_k$ and $r_1,\ldots ,r_k$ be as in Lemma \ref{lem:staircase} for some interval $(v_1, v_2)$.
    With $O(k)$ preprocessing, there is a data structure which can output a uniformly random element of $M$ which lies in $(v_1, v_2)$ in time $O(\log k)$ and probes only 1 element of $M$.
\end{lemma}
\begin{proof}
    The data structure is as follows.
    At construction time, first let $S=\{i_1,\ldots ,i_m\}$ be the (sorted) list of $i$ such that $\ell_i \neq \bot$; i.e.~the rows for which there exists some element of $M$ in $(v_1, v_2)$.
    For each $j\in [m]$, compute $B_j = (r_{i_j} - \ell_{i_j})+1$ and $C_{i_j} = \sum_{h=1}^{j}B_h$.
    This preprocessing takes $O(k)$ time by first removing the rows for which $\ell_i = r_i = \bot$ and then computing $C_{i_j}$ as prefix sums.

    To sample a uniformly random element of $M$ in $(v_1, v_2)$, first sample a uniformly random integer $\bh$ in $[C_{i_m}]$.
    Then, use binary search to find the largest element $i$ of $S$ with $C_i \leq \bh$, and return $M[i, \ell_i + (\bh - C_i)]$.

    Since all elements of $M$ in $(v_1, v_2)$ are exactly those in the intervals $[\ell_i, r_i]$ for $i\in S$, there are $C_{i_m}$ such elements and each is returned with probability $1/C_{i_m}$.
    The runtime follows from the fact that the binary search takes time $O(\log |S|) = O(\log k)$.
\end{proof}

\subsubsection{The algorithm}
\textbf{Definition ($C_u(i,j)$).} Given a sorted vector $u$ and $i\leq j$, define $C_u(i,j)$ to be the max variance of $u_h$ for $h\in [i,j]$ when quantized to the endpoints $u_i,u_j$; that is, $C_u(i,j) = \max_{i\leq h\leq j} (u_h - u_i)(u_j - u_h)$.
\bigbreak
\begin{observation}
Let $C_u$ be the matrix such that $C_u[i,j] = C_u(i,j)$ for $i \leq j$ and $C_u[i,j] = 0$ otherwise.
Then, $C_u$ is a sorted matrix.
\end{observation}
\begin{proof}
    First, note that for any fixed $i$, $C_u(i,j)$ is non-decreasing in $j$; similarly, for any fixed $j$, $C_u(i,j)$ is non-increasing in $i$.
    Thus, $C_u[i,j] \leq C_u[i,j+1]$ and $C_u[i,j] \leq C_u[i-1,j]$.
\end{proof}

\begin{observation}
    When $u$ is sorted, $C_u(i,j)$ can be computed in $O(\log k)$ time.
\end{observation}
\begin{proof}
    Note that $C_u(i,j)$ is maximized at the point $u_h$ closest to the midpoint of $[u_i,u_j]$: $(x-u_i)(u_j - x)$ is increasing for $x\leq (u_i + u_j)/2$ and decreasing for $x\geq (u_i + u_j)/2$, and is symmetric about the midpoint.
    Thus, we can binary search for the $u_h$ closest to the midpoint $(u_i + u_j)/2$ in time $O(\log k)$, and compute $C_u(i,j)$ in $O(1)$ time given $u_h$.
\end{proof}

\begin{lemma}
    \label{lem:value-to-cij}
    For any sorted $u\in \mathbb{R}^{k}$ and $Q\subseteq u$, there exists $i,j$ such that $\objo(u,Q) = C_u(i,j)$.
\end{lemma}
\begin{proof}
    Let $Q=\{q_1 < q_2 < \ldots < q_s\}$, and let $\beta_i$ be the index such that $q_i = u_{\beta_i}$.
    Then, by definition,
    \[
        \objo(u,Q) = \max_{j\in [s-1]} \max_{\beta_j \leq i \leq \beta_{j+1}} (u_i - q_j)(q_{j+1} - u_i) = \max_{j\in [s-1]} C_u(\beta_j, \beta_{j+1}).
    \]
    So, there exists some $j$ such that $\objo(u,Q) = C_u(\beta_j, \beta_{j+1})$.
\end{proof}

\begin{algorithm}[ht]
    \caption{4-Approximation Algorithm for $\objo(u,s)$}
    \label{alg:4-approx}
    \KwIn{Sorted $u\in \mathbb{R}^{k}$, size parameter $2\leq s < k$, failure probability $\delta$}
    \KwOut{Value $v$ such that $\objo(u,s)\leq v \leq 4\objo(u,s)$ with probability at least $1-\delta$}

    Initialize $v_1 = 0$, $v_2 = \max_{i,j}\{C_u(i,j)\} = C_u(1,k)$ \\
    \For{$i\in [1,2,\ldots ,\log k/2]$}{
        Sample and compute $2^{i+2}\cdot\log^2 k\cdot\log (\log k / \delta)$ elements of $C_u$ uniformly at random \\
        Let $v'$ be the median of the sampled elements which lie in $(v_1, v_2)$ \\
        Run Algorithm \ref{alg:faster-exact} on $u$ with variance parameter $v'$, which outputs a set of size $s'$ \\
        \lIf{$s' > s$}{set $v_1 = v'$} \lElse{set $v_2 = v'$}
    }
    \algcomment{At this step, the number of elements of $C_u$ in $(v_1,v_2)$ is $O(k^{3/2})$ w.p.~$1-\delta$ (Lemma \ref{lem:loop1-sizes})} \\
    Use Lemma \ref{lem:staircase} on $C_u$ with interval $I_1^{*}=(v_1,v_2)$ to compute indices $\ell_1,\ldots ,\ell_k$ and $r_1,\ldots ,r_k$ \\
    \For{$i\in [1,2,3,\ldots ,\log k/2]$}{
        Use Lemma \ref{lem:staircase-sample} to sample and compute $2^{i+4}\cdot\log^2k \cdot \log(\log k / \delta)$ elements of $C_u$ with values in $I_1^{*}$. \\
        Let $v'$ be the median of the sampled elements which lie in $(v_1, v_2)$ \\
        Run Algorithm \ref{alg:faster-exact} on $u$ with variance parameter $v'$, which outputs a set of size $s'$ \\
        \lIf{$s' > s$}{set $v_1 = v'$} \lElse{set $v_2 = v'$}
    }
    \algcomment{At this step, the number of elements of $C_u$ in $(v_1,v_2)$ is $O(k)$ w.p.~$1-\delta$ (Lemma \ref{lem:cals-size})}\\
    Use Lemma \ref{lem:staircase} on $C_u$ with interval $I_2^{*}=(v_1,v_2)$ to compute indices $\ell_1,\ldots ,\ell_k$ and $r_1,\ldots ,r_k$ \\
    Set $\calS$ to be the elements of $C_u$ in $(v_1,v_2)$ by computing $C_u[i,j]$ for all $j\in [k]$, $i\in [\ell_j, r_j]$ \label{line:final-interval} \\
    Add $v_2$ to $\calS$ \label{line:add-v2} \\
    Sort $\calS$, and binary search over $\calS$ using Algorithm \ref{alg:faster-exact} to find the smallest $v^{*}\in \calS$ such $\objos(u,v^{*})\leq s$ \label{line:bin-search-s}\\
    \Return{$v^{*}$}
\end{algorithm}

\begin{lemma}
    \label{lem:loop1-sizes}
    For each $i\in [0,1,\ldots ,\log k/2]$, with probability at least $1-2i\delta/\log k$, at the end of the $i$th iteration of the first \texttt{for} loop of Algorithm \ref{alg:4-approx} there are at most $k^2\cdot (1/2+1/\log k)^{i}$ elements of $C_u$ which lie in the interval $(v_1,v_2)$.
\end{lemma}
\begin{proof}
    We proceed by induction on $i$.
    The base case of $i=0$ is trivial, as $C_u$ contains only $k^2$ elements.
    So, fix some $i\geq 1$ and suppose the claim holds for $i-1$.
    Let $\bm$ denote the number of elements of $C_u$ in $(v_1,v_2)$ at the end of the $i-1$ iteration, so by the inductive hypothesis, with probability at least $1-2(i-1)\delta/\log k$, $\bm\leq k^2(1/2+1/\log k)^{i-1}$.

    If $\bm \leq k^2(1/2+1/\log k)^i$, then since iteration $i$ can only shrink the interval $(v_1, v_2)$, the count at the end of the iteration is at most $\bm \leq k^2(1/2+1/\log k)^i$ and the claim holds.
    It therefore suffices to consider the case $\bm > k^2(1/2+1/\log k)^i$.

    In iteration $i$, Algorithm~\ref{alg:4-approx} draws $T_i = 2^{i+2} L$ entries uniformly at random from all $k(k-1)/2\leq k^2$ entries of $C_u$, where $L = \log^2 k\log(\log k / \delta)$.
    Let $\bY$ be the random variable of the number of these entries falling in $(v_1, v_2)$.  Since $\bm/k^2 > (1/2+1/\log k)^i$,
    \[
        \mathbb{E}[\bY] = T_i \cdot \frac{\bm}{k^2} > 2^{i+2} L \cdot \left(\frac{1}{2}+\frac{1}{\log k}\right)^i = 4\left(1+\frac{2}{\log k}\right)^i L \geq 4L.
    \]
    By the Chernoff bound,
    \[
        \Pr\left[\bY < \mathbb{E}[\bY]/4\right] \leq \exp\left(-\frac{9\,\mathbb{E}[\bY]}{32}\right) \leq \exp\left(-\frac{9L}{8}\right) < \frac{\delta}{\log k}.
    \]
    Thus, $\bY \geq L$ with probability at least $1 - \delta / \log k$.

    Among the $\bY$ sampled entries in $(v_1, v_2)$, Algorithm~\ref{alg:4-approx} takes their sample median $\bv'$.
    Let $\bR_1$ denote the number of elements of $C_u$ in $(v_1,\bv')$ and let $\bR_2$ be the number of elements of $C_u$ in $(\bv', v_2)$.
    We will show that, conditioned on $\bY \geq L$, with probability at least $1-\delta/\log k$, both $\bR_1,\bR_2$ are at most $(1/2+1/\log k)\bm$.


    Let $r$ be the largest entry of $C_u$ with rank less than $(1/2-1/\log k)\bm$.
    Then, by definition, each of the $\bY \geq L$ sampled entries in $(v_1,v_2)$ is less than or equal to $r$ with probability at most $1/2 - 1/\log k$.
    Since $\bv'$ is the sample median, it follows that $\bR_2\geq (1/2+1/\log k)\bm$ if and only if at least half the sampled elements are less than or equal to $r$.
    So, by Hoeffding's inequality,
    \[
        \Pr[\bR_2 > (1/2+1/\log k)\bm \mid \bY \geq L] \leq e^{-2L/\log^2 k} \leq \frac{\delta}{2\log k}.
    \]
    An analogous argument shows that $\Pr[\bR_1 > (1/2+1/\log k)\bm \mid \bY \geq L]\leq \delta/2\log k$.
    Since $\bY \geq L$ with probability at least $1-\delta/\log k$, a union bound gives
    \[
        \Pr\bigl[\bR_1 \leq (1/2+1/\log k)\bm \text{ and } \bR_2 \leq (1/2+1/\log k)\bm\bigr] \geq 1-2\delta/\log k
    \]

    At the end of iteration $k$, the algorithm updates either $v_1$ or $v_2$ to $\bv'$, so the number of points remaining is either $\bR_1$ or $\bR_2$.
    With probability $1-2\delta/\log k$, both $\bR_1$ and $\bR_2$ are less than $(1/2+1/\log k)\bm$, and by the inductive hypothesis, $\bm\leq k^2(1/2+1/\log k)^{i-1}$ with probability at least $1-2(i-1)\delta/\log k$.
    So, by the union bound, the number of elements remaining in the interval at the end of iteration $k$ is at most $k^2(1/2+1/\log k)^{i}$ with probability at least $1-2(i-1)\delta/\log k - 2\delta/\log k = 1-2i\delta/\log k$.\qedhere

\end{proof}

Lemma \ref{lem:loop1-sizes} then immediately gives the following corollary.
\begin{corollary}
    \label{corr:loop1-final-size}
    At the end of the first \texttt{for} loop of Algorithm \ref{alg:4-approx}, with probability at least $1-\delta$, there are at most $3k^{3/2}$ elements of $C_u$ in the range $(v_1, v_2)$.
\end{corollary}
\begin{proof}
    The first \texttt{for} loop runs for $\log k/2$ iterations.
    Applying Lemma \ref{lem:loop1-sizes} at $i = \log k/2$, we have that with probability at least $1-\delta$, the number of elements of $C_u$ in $(v_1, v_2)$ is at most
    \[
        k^2 \cdot \left(\frac{1}{2}+\frac{1}{\log k}\right)^{\log k/2} = k^2\cdot \left(\frac{1}{2}\right)^{\log k /2} \cdot \left(1+\frac{1}{\log k/2}\right)^{\log k/2} \leq 3k^{3/2}
    \]
    for sufficiently large $k$.
\end{proof}

\begin{lemma}
    \label{lem:cals-size}
    Let $\calS$ be the set constructed in Line \ref{line:final-interval} of Algorithm \ref{alg:4-approx}.
    Then, with probability at least $1-2\delta$, $|\calS|=O(k)$.
\end{lemma}
\begin{proof}
    By Corollary \ref{corr:loop1-final-size}, after the end of the first \texttt{for} loop, with probability at least $1-\delta$, there are at most $3k^{3/2}$ elements of $C_u$ which lie in $(v_1,v_2)$.
    The algorithm of the second \texttt{for} loop then samples only from these elements.
    So, an analogous argument to the proofs of Lemma \ref{lem:loop1-sizes} and Corollary \ref{corr:loop1-final-size} but replacing the $k^2$ total elements of $C_u$ with the $3k^{3/2}$ in $(v_1,v_2)$ gives that there are at most $3k^{3/2} \cdot 3k^{-1/2} = O(k)$ elements of $\calS$ with probability at least $1-\delta$.
    A union bound then gives the result.
\end{proof}

\begin{lemma}
    \label{lem:4-approx-runtime}
    Algorithm \ref{alg:4-approx} runs in time
    \[
        O\left(k\log k + \sqrt{k}\cdot \log^{3} k \cdot \log(\log k / \delta)\right)
    \]
    with probability at least $1-2\delta$.
\end{lemma}
\begin{proof}
    Throughout the two \texttt{for} loops, there are $\log k$ iterations, each of which requires\\ $T=O(2^{\log k/2}\cdot \log^2k \cdot \log(\log k/\delta))=O(\sqrt{k}\cdot \log^2 k\cdot \log(\log k/\delta))$ samples, so the algorithm queries $T$ elements of $C$.
    Each element of $C$ can be computed in time $O(\log k)$, so the total time spent computing elements of $C$ is $O(T\log k)$.
    There are also two calls to the algorithm of Lemma \ref{lem:staircase}, each of which involves $O(k)$ probes to $C$ and $O(k)$ additional time, so the total time spent on these calls is $O(k\log k)$.

    By Lemma \ref{lem:cals-size}, with probability at least $1-2\delta$, $|\calS|=O(k)$.
    So, sorting and computing the values of $\calS$ takes time $O(k\log k)$.
    Finally, the binary search of Line \ref{line:bin-search-s} takes time $O(k\log k)$, as Algorithm \ref{alg:faster-exact} has runtime $O(s\log (k/s)) = O(k)$.

    Combining all the steps, the total time is $O(k\log k + T \log k) = O(k\log k + \sqrt{k}\cdot \log^{3}k\cdot\log(\log k/\delta))$.
\end{proof}

\begin{lemma}
    \label{lem:4-approx-quality}
    Algorithm \ref{alg:4-approx} returns a value $v$ such that $\objo(u,s)\leq v \leq 4\objo(u,s)$.
\end{lemma}
\begin{proof}

    Let $v^{*}$ be the smallest element of $C_u$ for which $\objos(u,v^{*})\leq s$.
    We claim that at all times throughout the execution of the algorithm, $v^{*}\in (v_1,v_2]$.

    Initially, $v_1 = 0 < v^{*}$\footnote{As $s < k$} and $v_2 = C_u(1,k) \geq v^{*}$ (since $C_u(1,k)$ is the maximum variance when including only the two endpoints $u_1, u_k$).
    At every step of the algorithm, the algorithm updates $v_1$ to a value $v'$ only if $\objos(u,v') > s$, and thus it must be that $v^{*} > v'$ as $\objos(u,v^{*})\leq s$ by construction.
    Similarly, the algorithm only updates $v_2$ to a value $v'$ if $\objos(u,v')\leq s$.
    Since the algorithm only ever considers updates to $v_2$ of the form $v'=C_u(i,j)$ for some $i,j\in [k]$, and $v^{*}$ is the minimum such value for which $\objo(u,v')\leq s$, it follows that $v^{*}\leq v_2$ as well.

    The algorithm constructs in Lines \ref{line:final-interval} and \ref{line:add-v2} the set $\calS$ to be all values of $C_u$ in $(v_1,v_2]$ (with the values $v_1,v_2$ as set in the prior \texttt{for} loops).
    Since $v^{*}\in (v_1,v_2]$ and there exists $i,j$ such that $v^{*} = C_u(i,j)$ by construction, it follows that $v^{*}\in\calS$.
    Moreover, by definition, $v^{*}$ must be the smallest element of $\calS$ for which $\objos(u,v^{*})\leq s$, and thus the algorithm returns $v^{*}$.
    Since $\objos(u,v^{*})\leq s$, $v^{*}\geq \objo(u,s)$.
    Finally, by Lemma \ref{lem:approx-by-input}, $v^{*}\leq 4\objo(u,s)$, completing the proof.
\end{proof}

\subsection{\texorpdfstring{Proving Theorem \ref{thm:obj1-full-algo}}{}}

The algorithm of Theorem \ref{thm:obj1-full-algo} now follows from combining a fast $k$-center algorithm, the 4-approximation of Algorithm \ref{alg:4-approx}, and a small modification to the $(1+\varepsilon)$-approximation for $\objos$ from Algorithm \ref{alg:ptas-objos}.
\newcommand{\vapx}{v_{\textsf{apx}}}
\newcommand{\esub}{E_{\textsf{sub}}}
\newcommand{\qapx}{Q_{\textsf{apx}}}
\begin{algorithm}[ht]
    \caption{4-Approximation Algorithm for $\objo(w,s)$}
    \label{alg:full-alg-objo}
    \KwIn{$w\in \mathbb{R}^{d}$, size parameter $2\leq s < d$, approximation parameter $\varepsilon$, failure probability $\delta$}
    \KwOut{Value $v$ such that $\objo(w,s)\leq v \leq (1+\varepsilon)\objo(w,s)$}
    Let $C$ be a 2-approximate $s$-center solution over $w$, computed using $\texttt{s-Center-Clustering}$ of \cite{FG88} \\
    Let $I_c$ be the points of $w$ closest to $c\in C$.  Compute $E=\bigcup_{c\in C} \{\min I_c, \max I_c\}$ \label{line:compute-E} \\
    Sort $E$ and run Algorithm \ref{alg:4-approx} on $E$ with size parameter $s$ and error parameter $\delta$ to get an estimate $\vapx$ \\
    Run Algorithm \ref{alg:faster-exact} on $E$ with $\vapx$ to get a set $\qapx$, and compute $v' = \objo(w,\qapx)$ \label{line:compute-qapx} \\
    \For{all $c \in C$}{
        Compute $a_c = \min(I_c)$, $b_c = \max(I_c)$, and $D_c = b_c - a_c$ \\
        Compute subintervals $I_c^{j}=[a_c + D_c (j-1)\sqrt{\varepsilon}/4, a_c + D_c j\sqrt{\varepsilon}/4]$ for each $j\leq \left\lceil 4/\sqrt{\varepsilon} \right\rceil$ \\
    }
    Let $\esub=\bigcup_{c\in C, j\leq \left\lceil 4/\sqrt{\varepsilon} \right\rceil } \{\min I^{j}_c, \max I^{j}_c\}$ \\
    Sort $\esub$ \\
    Use binary search and Algorithm \ref{alg:faster-exact} to find the smallest $i^{*}\in \mathbb{Z}$ for which $v_i = (1+\varepsilon/2)^{i}\in [v'/9, v'(1+\varepsilon/2)]$ and $\objos(\esub,v_i) \leq s$ \label{line:bin-search-objo}\\
    Return $v_{i^{*}}\cdot(1+\varepsilon/2)$
\end{algorithm}
While the algorithm as written returns the \textit{value} of an approximate solution, constructing a set with at most this value is straightforward: instead of returning $v_{i^{*}}(1+\varepsilon/2)$, use call Algorithm \ref{alg:faster-exact} on $\esub$ with $v_{i^{*}}$ to obtain a set $Q'$ and return $Q'$.

\begin{lemma}
    \label{lem:full-alg-objo-runtime}
    Algorithm \ref{alg:full-alg-objo} runs in time
    \[
        O\left( d\log(s/\varepsilon) + \sqrt{s}\cdot\log^3(s)\cdot\log(1/\delta) \right)
    \]
    with probability at least $1-\delta$.
\end{lemma}
\begin{proof}
    Computing $C$ and all $I_c$ takes time $O(d\log s)$, from the $\texttt{s-Center-Clustering}$ algorithm of \cite{FG88}, and thus it takes time $O(d\log s)$ to compute the set $E$.
    $|E|\leq 2s$, as $|C|= s$, and so by Lemma \ref{lem:4-approx-runtime}, sorting $E$ and running Algorithm \ref{alg:4-approx} takes time $O(s\log s + s^{0.51}\log(1/\delta))$ w.p.~$1-\delta$.
    Computing the set $\qapx$ takes time $O(s)$ using Algorithm \ref{alg:faster-exact}, and computing $v'=\objo(w,\qapx)$ takes time $O(d\log s)$ as $|\qapx|=s$.
    Thus, the total runtime before the \texttt{for} loop is $O(d\log s + s^{0.51}\log(1/\delta))$.

    It takes time $O(d\log(s/\varepsilon))$ to compute the subintervals $I^{j}_c$, and the same time to compute $\esub$.
    Note that $|\esub|= O(\min(d,s/\sqrt{\varepsilon}))$, as $\esub\subseteq w$ and each of the $O(s/\sqrt{\varepsilon})$ subintervals contributes at most 2 points to $\esub$.
    So, sorting $\esub$ takes time $O(|\esub|\log |\esub|)=O(d\log(s/\varepsilon))$.

    Finally, the binary search of Line \ref{line:bin-search-objo} considers $O(\log 1/\varepsilon)$ options $v_i$, and on each runs Algorithm \ref{alg:faster-exact}.
    Each call to Algorithm \ref{alg:faster-exact} takes time $O(s\log(|\esub|/s)) = O(d)$, and so the total runtime of Line \ref{line:bin-search-objo} is $O(d\log(1/\varepsilon))$.

    Combining all the steps, the total runtime is thus
    \[
        O\left( d\log s + s^{0.51}\log(1/\delta) + d\log(s/\varepsilon) + d\log(1/\varepsilon) \right) = O\left( d\log(s/\varepsilon) + s^{0.51}\log(1/\delta) \right).\qedhere
    \]
\end{proof}

\begin{lemma}
    \label{lem:full-alg-objo-quality}
    Algorithm \ref{alg:full-alg-objo} returns a value $v$ such that $\objo(w,s)\leq v\leq (1+\varepsilon)\objo(w,s)$.
\end{lemma}
\begin{proof}
    Let $v^{*}=\objo(w,s)$, and let $Q^{*}\subset \mathbb{R}$ be a set with $\objo(w,Q^{*})=v^{*}$.
    Since $\varasq(w_i, Q^{*})\leq v^{*}$ for all $i\in [d]$, it follows that for each $w_i$, there exists an element $q\in Q^{*}$ such that $|q-w_i|\leq \sqrt{v^{*}}$.
    Thus, the optimal $s$-center instance has cost at most $\sqrt{v^{*}}$, and the $\texttt{s-Center-Clustering}$ algorithm of \cite{FG88} returns a set $C$ such that for all $i\in [d]$, there exists a $c\in C$ with $|c - w_i|\leq 2\sqrt{v^{*}}$.
    For each $c\in C$, it then follows that the interval $I_c$ has length at most $4\sqrt{v^{*}}$.

    Let $E$ be constructed as in Line \ref{line:compute-E} of Algorithm \ref{alg:full-alg-objo}, and $\qapx$ as computed in Line \ref{line:compute-qapx}.
    $E$ consists of the endpoints of a set of intervals which cover $w$ and which each have length at most $4\sqrt{v^{*}}$. 
    So, by Lemma \ref{lem:approx-by-endpoints}, $v'=\objo(w,\qapx)\leq \objo(E,\qapx) + 4v^{*}$.
    Moreover, by the guarantees of Algorithm \ref{alg:4-approx} (Lemma \ref{lem:4-approx-quality}), $\objo(E,\qapx)\leq 4 \objo(E,s)$.
    Thus,
    \begin{align*}
        v' &\leq \objo(E,\qapx) + 4v^{*} = 4\objo(E,s) + 4v^{*} \leq 8v^{*}
    \end{align*}
    where the final inequality follows from $E\subseteq w$ and so $\objo(E,s)\leq \objo(w,s) = v^{*}$.
    By definition of $v^{*}$, $\objo(E,\qapx)\geq v^{*}$ and we have that
    \[
        v'\in [v^{*}, 8v^{*}].
    \]
    Define the intervals $I_c^{j}$ and set $\esub$ as in Algorithm \ref{alg:full-alg-objo}.
    Each interval $I_c$ has length at most $4\sqrt{v^{*}}$, and thus, by construction, each $I_c^{j}$ has length at most $4\sqrt{v^{*} \varepsilon}/{4} = \sqrt{v^{*} \varepsilon}$.
    As $\esub$ is the union of the endpoints of the $I_c^{j}$, it thus follows that
    \[
        v^{*}=\objo(w,s) \leq \objo(\esub, s) + \varepsilon v^{*}/4 \quad\text{and so}\quad \objo(\esub, s) \geq (1-\varepsilon/4)v^{*}.
    \]
    Let $\nu = \objo(\esub, s)$; we have $\nu \leq v^{*}\leq v'$ (as $\esub\subseteq w$) and $\nu \geq (1-\varepsilon/4)v^{*} > v'/9$.
    Thus, $\nu\in [v'/9, v']$.

    Now, consider the $i$ such that $v_i=(1+\varepsilon/2)^{i}$ is in $[v'/9, v']$ and $\nu \leq v_i < (1+\varepsilon/2)\nu$.
    Since $v_i\geq \nu$, it follows that $\objos(\esub, v_i) \leq s$, and moreover $v_i$ is the smallest such power of $(1+\varepsilon/2)$ by construction.
    As $\nu \in [v'/9, v']$, $v_i\in [v'/9, v'(1+\varepsilon/2)]$ and so the binary search of Line \ref{line:bin-search-objo} will return $v_i$.

    Finally, the algorithm returns $v_i(1+\varepsilon/2)$.
    We have
    \[
        v_i(1+\varepsilon/2) \geq \nu(1+\varepsilon/2) \geq v^{*}(1-\varepsilon/4)(1+\varepsilon/2) > v^{*}
    \]
    as $\nu \geq (1-\varepsilon/4)v^{*}$.
    Moreover, as $\nu\leq v^{*}$,
    \[
        v_i(1+\varepsilon/2) \leq \nu(1+\varepsilon/2)^2 \leq v^{*} (1+\varepsilon)
    \]
    giving the desired bounds.
\end{proof}

Theorem \ref{thm:obj1-full-algo} then follows from Lemma \ref{lem:full-alg-objo-quality} and Lemma \ref{lem:full-alg-objo-runtime}.

\subsection{Irrational Solutions}
While we have presented efficient algorithms that achieve high-precision estimates of the optimal quantization set, it turns out that returning an optimal quantization set that minimizes the $\objo$ objective is impossible under standard bit representation (i.e.~floating point) of any precision. In particular, here we show a simple example where the vector $w \in \Z^d$, but $\objo(w,s)$ is irrational (and thus any optimal quantization set contains irrational values as well).

\begin{lemma} \label{lem: mdv-irrational-solutions}
    Let $w=(0,2,5,8,10)$ and $s=4$. Then, $\objo(w,s)=8-2\sqrt{7}$.
\end{lemma}
\begin{proof}
    The quantization set must contain elements 0 and 10 (the min and max of $w$), so the problem of finding an optimal quantization set reduces to finding the position of the remaining two elements.
    Let $Q^{*}=\{0,x,y,10\}$ be an optimal quantization set, with $x < y$.
    Let $Q'=\{0,3,7,10\}$; so $\objo(w,Q')=4$ and thus $\objo(w,Q^{*})\leq 4$.
    It then follows that $x\leq 5$ and $y\geq 5$: any set $Q$ with no elements in $(5,10)$ or none in $(0,5)$ has $\objo(w,Q) > 6$.
    Similarly, $x\geq 2$ and $y\leq 8$: if $x < 2$, the set $Q' = \{0,2,y,10\}$ would have $\objo(w,Q') < \objo(w,Q^{*})$, and analogously for $y > 8$.

    When $x\in [2,5]$ and $y\in [5,8]$, we have by definition
    \[
        \objo(w,Q^{*}) = \max \left\{2(x-2), (5-x)(y-5), 2(8-y)\right\}
    \]
    Solving, we find the minimum occurs when $x=6-\sqrt{7}$, $y=4+\sqrt{7}$, and $\objo(w,Q^{*})=8-2\sqrt{7}$.
\end{proof}

\subsection{Practical Algorithms}
\label{sec:mdv-practical}
While the algorithm presented in this section has excellent theoretical guarantees, its complexity and use of the $\texttt{s-Center-Clustering}$ algorithm of \cite{FG88} makes it challenging to use in practice.
As such, in this section we detail how we modify our algorithm to better perform on real-world systems, while retaining theoretical guarantees, and provide performance and accuracy evaluations.

The full pseudocode of our algorithm can be found in Algorithm \ref{alg:mdv-practical}. Algorithm \ref{alg:mdv-practical} is still guaranteed to output a $(1+\varepsilon)$-approximate solution.

\begin{algorithm}[ht]
    \caption{Practical Algorithm for $\objo(w, s)$}
    \label{alg:mdv-practical}
    \KwIn{$w \in \mathbb{R}^d$, size parameter $s\in \mathbb{N}$, number of initial intervals $m$, approximation parameter $\varepsilon > 0$}
    \KwOut{Set $Q$ of size $s$}

    Let $a = \min(w)$, $b = \max(w)$ \\
    Initialize $\mathcal{I}$ as the $m-1$ equal intervals $\left[a + \frac{(b-a)(j-1)}{m-1},\; a + \frac{(b-a)j}{m-1}\right]$ for $j = 1, \ldots, m-1$ \\
    \For{each $I \in \mathcal{I}$ with $I \cap w \neq \emptyset$}{
        Add $\min\{w_i : w_i \in I\}$ and $\max\{w_i : w_i \in I\}$ to $x$
    }
    Sort $x$; set $\texttt{lo} = 0$ and $\texttt{hi} = (b - a)^2/4$ \\
    \While{$\texttt{hi} - \texttt{lo} > \varepsilon\cdot\texttt{hi}$}{
        \If{there exists some interval $I\in \calI$ with $\max\{w_i : w_i \in I\} - \min\{w_i : w_i \in I\} > \sqrt{\varepsilon \cdot \texttt{hi}}$}{
        Set $\calI$ to be the $m' = \left\lceil (b-a)/\sqrt{\varepsilon\cdot \texttt{hi}} \right\rceil$ equal-sized intervals of $[a,b]$.\\
        For each $I\in \calI$ with $I\cap w \neq \emptyset$, add $\min \{w_i : w_i \in I\}$ and $\max \{w_i : w_i \in I\}$ to $x$. \\
        Sort $x$.
        }
        Set $\texttt{mid} = (\texttt{lo} + \texttt{hi})/2$ \\
        Run Algorithm \ref{alg:faster-exact} on $x$ with variance parameter $\texttt{mid}$; let $s'$ be the size of the returned set \\
        \lIf{$s' > s$}{set $\texttt{lo} = \texttt{mid}$} \lElse{set $\texttt{hi} = \texttt{mid}$}
    }
    \Return{the set $Q$ constructed in the last call to Algorithm \ref{alg:faster-exact} which updated $\texttt{hi}$.}
\end{algorithm}

\begin{lemma}
    \label{lem:mdv-practical-qual}
    Let $Q$ be the set returned by Algorithm \ref{alg:mdv-practical} when run on vector $w\in \mathbb{R}^{d}$ and $s\in \mathbb{N}$.
    Then,
    \[
        \objo(w,Q) \leq (1+O(\varepsilon))\objo(w,s).
    \]
\end{lemma}
\newcommand{\lo}{\texttt{lo}}
\newcommand{\hi}{\texttt{hi}}
\newcommand{\mi}{\texttt{mid}}
\begin{proof}
    We first claim that throughout the entire iteration of the algorithm, $\lo < \objo(w,s) \leq (1+\varepsilon)\hi$.
    Initially, $\lo = 0$ and so $\objo(w,s) > 0$; in any iteration of the \texttt{while} loop, $\lo$ is increased to a value $v$ if and only if $\objo(x,s) > v$; since $x\subseteq w$, it follows that $\objo(w,s) \geq \objo(x,s) > \lo$ as well.

    For the other direction, by construction $\objo(w,s) \leq \hi$ at the beginning of the algorithm.
    At any other iteration, $\hi$ is updated to $\mi$ if and only if $\objo(x,s) \leq \mid$, where $x$ is the endpoint of the intervals $\calI$.
    By construction, each interval $I \in \calI$ has length at most $\sqrt{\varepsilon \cdot \hi} \leq \sqrt{\varepsilon \cdot 2\mi} < 2\sqrt{\varepsilon\cdot \mi}$, as $\mi = (\lo + \hi)/2 \geq \hi/2$.
    Thus, applying Lemma \ref{lem:approx-by-endpoints},
    \[
        \objo(w,s) \leq \objo(x,s) + \varepsilon\cdot \mi \leq \mi + \varepsilon \cdot \mi = (1+\varepsilon)\mi
    \]
    as we have $\objo(x,s)\leq \mi$.
    Thus, when updating $\hi$ to $\mi$, we maintain that $\objo(w,s)\leq (1+\varepsilon)\hi$.

    At the end of the algorithm, $\lo \geq (1-\varepsilon)\hi$, due to the termination condition of the \texttt{while} loop, and so $(1-\varepsilon)\hi \leq \objo(w,s) \leq (1+\varepsilon)\hi$.
    Moreover, the algorithm returns a set $Q$ such that $\objo(x,Q) \leq \hi$, where $x$ is the endpoint of a collection of intervals of length at most $\sqrt{\varepsilon \cdot \hi}$, so applying Lemma \ref{lem:approx-by-endpoints} again,
    \[
        \objo(w,Q)\leq \objo(x,Q) + \varepsilon\cdot\hi / 4 \leq (1+\varepsilon)\hi \leq (1+\varepsilon)^2\objo(w,s) < (1+2\varepsilon)\objo(w,s)
    \]
    as desired.
\end{proof}

\section{LLM Quantization Setup} \label{sec: llm-setup}
In this section, we describe in detail the LLM experiments presented in Section \ref{sec:llm-expers}.

\paragraph{Quantized weights.}
We quantize all 96 MLP weight matrices ($\mgate$, $\mup$ and $\mdo$ in each of the 32 layers) of Qwen3.5-4B, and leave all other weights at their original precision.
Every method uses $s=32$ quantization values, so each weight is stored as a 5-bit index.
In the per-row setting, each row of each matrix has its own quantization set (stored as 32 fp16 values), for 5.15 bits per weight; in the per-matrix setting, the set is stored per matrix.
Quantized weights are decoded back to the original precision before evaluation.

\paragraph{Input distribution for $\objt$.}
Each row $w$ of a weight matrix is used only through inner products $\la w, x \ra$ with the input $x$ to that layer, so we take $\calX$ to be the distribution of these inputs.
We estimate $\lambda_i = \E_{\bx\sim\calX}[\bx_i^2]$ for each input coordinate $i$ by running the unquantized model on 16 windows of 2048 tokens from the WikiText-2 training set.
In the per-row setting, entry $w_i$ of a row has weight $\lambda_i$; in the per-matrix setting, entry $W_{ji}$ has weight $\lambda_i$.
We use the $\lambda_i$ as real-valued weights in the exact algorithm, rescaling them to have mean 1 (with a minimum weight of $10^{-6}$).

\paragraph{Methods.}
We compare $\objt$ against four baselines.
ASQ optimizes the same objective with every coordinate weighted equally, effectively \emph{ignoring} the input ditribution $\calX$.
$k$-means gives the quantization set with minimal MSE under deterministic nearest rounding, while uniform quantization is the simplest non-adaptive choice.
TurboQuant also targets inner product preservation, but is non-adaptive: it applies a random rotation so that the entries of the weight matrix are approximately Gaussian, and then uses a fixed global quantization set designed for the Gaussian distribution.
In detail, the methods are as follows.
\begin{itemize}
    \item $\objt$: the optimal quantization set for $\objt$, computed with our exact algorithm.
    \item ASQ: the optimal quantization set for (unweighted) ASQ, computed with the same algorithm with all weights equal to 1.
    \item $k$-means: the optimal 1D $k$-means centers, computed with the \texttt{Wilber} algorithm of \cite{GKMNSS17}.
    \item TurboQuant \cite{zandieh2025turboquant}: the weight matrix $W$ is replaced by $WR^\top$, where $R = \frac{1}{\sqrt{b}} B D P$ for a random permutation $P$, a random $\pm 1$ diagonal matrix $D$, and a block-diagonal matrix $B$ of $b\times b$ Hadamard matrices, with $b$ the largest power of 2 dividing the input dimension. Since $R$ is orthogonal, $Wx = (WR^\top)(Rx)$, so the layer is evaluated on the rotated input $Rx$. The quantization set is the $s$-point Lloyd-Max quantizer for $N(0,1)$, shifted and scaled by the mean and standard deviation of the (rotated) row or matrix.
    \item Uniform: $s$ evenly spaced values between the minimum and maximum entry.
\end{itemize}

Although $\objt$ and ASQ are defined in terms of standard stochastic quantization, in Tables \ref{table: llm-row} and \ref{table: llm-matrix} we round each weight to its nearest quantization value for every method.
In the per-row setting, quantization sets are computed from all entries of the row, and in the per-matrix setting, from all entries of the matrix.

Since the model's weights are stored in bfloat16 and thus contain many repeated values, before computing the quantization sets for $\objt$, ASQ and $k$-means we merge all entries with equal values into a single point whose weight is the sum of the weights of the merged entries.  This does not change the objective to be optimized.

\paragraph{Evaluation.}
We split the WikiText-2 test set into 145 non-overlapping windows of 2048 tokens.
At each position, let $p$ and $q$ be the next-token distributions of the unquantized and quantized models.
Storing $p$ in full over the 152k-token vocabulary is infeasible, so we keep only the 512 most likely tokens $S$ under $p$ (on average 98.25\% of its mass) and report
\[
    \sum_{t\in S} p_t \log \frac{p_t}{q_t} + p_{\bar S} \log \frac{p_{\bar S}}{q_{\bar S}},
\]
where $p_{\bar S}$ and $q_{\bar S}$ are the total probabilities outside of $S$.
By the data processing inequality, this is a lower bound on $\text{KL}(p\,\|\,q)$.
We average over all positions.
Top-1 agreement is the fraction of positions at which both models assign the highest probability to the same token, and perplexity is computed as usual (the unquantized model has perplexity 9.58).

For ARC-Challenge, we use all 1172 questions of the test set.
Each question is given to the model as ``\texttt{Question: <question>\textbackslash nAnswer:}'', each candidate answer is scored by its log-likelihood divided by its length in characters, and the highest-scoring answer is chosen.
ARC agree is the fraction of questions on which the quantized model chooses the same answer as the unquantized model, whether or not that answer is correct.

\section{Deferred Proofs and Figures} \label{sec: missing-proofs}

\subsection{Deferred Proofs from Appendix \ref{sec: adv}}

\begin{proof}[Proof of Lemma \ref{lem: endpoints-in-quant-set}]
    First, note that Adaptive Stochastic Quantization distribution is only well-defined if both $\wdown{1}(Q)$ and $\wup{1}(Q)$ exist. We show that if $\{ \wdown{1}(Q), \wup{1}(Q)\} \neq \{w_1, w_d\}$, then $\objt_{\calX}(w, Q) > \objt_{\calX}(w, s)$. First, suppose that $\wdown{1}(Q) \neq w_1$. Then, consider quantization set $Q':= (Q \setminus \{ \wdown{1}(Q) \}) \cup \{ w_1 \}$ and observe that 
    \begin{align*}
        \objt_{\calX}(w, Q') &= \sum_{i= 1}^{d} \lambda_i (\wiup(Q') - w_i)(w_i - \widown(Q')) \\
        &= \sum_{i=2}^{d} \lambda_i (\wiup(Q') - w_i)(w_i - \widown(Q')) \tag{$\wdown{1}(Q') = w_1$} \\
        &\leq \sum_{i=2}^{d} \lambda_i (\wiup(Q) - w_i)(w_i - \widown(Q)) \\
        &\leq \objt_{\calX}(w, Q)
    \end{align*}
    Where the first inequality follows because $\wiup(Q) = \wiup(Q')$ and $\widown(Q') \geq \widown(Q)$. A symmetric argument can be used to show that $\objt_{\calX}(w, Q) > \objt_{\calX}(w, s)$ if $\wdown{1}(Q) \neq w_d$. Thus, any quantization set not containing $w_1$ and $w_d$ must be suboptimal with respect to $\objt_{\calX}$. 
\end{proof}

\begin{proof}[Proof of Lemma \ref{lem: quant-set-inside-w}]
    Consider an optimal quantization set $Q = \{q_1, \ldots, q_s\} \not\subseteq w$ such that $q_1 \leq \ldots \leq q_s$ and $\objt_{\calX}(w, Q) = \objt_{\calX}(w, s)$. The argument proceeds by iteratively transforming $Q$, while leaving its objective value unchanged. 

    Consider $q_i \in Q$ such that $q_i \notin w$. Define $\qidown := \max \{ w_j : w_j \leq q_i \}$ and analogously $\qiup := \min \{ w_j : w_j \geq q_i\}$. We first show that $q_{i-1} < \qidown$ and $q_{i+1} > \qiup$. Suppose that $q_{i-1}\geq \qidown$ and $q_{i+1} > \qiup$. Then, quantization set $Q' := (Q \setminus \{ q_i \}) \cup \{ \qiup \}$ has $\objt_{\calX}(w, Q') < \objt_{\calX}(w, Q)$, contradicting the optimality of $Q$. A symmetric argument contradicts the optimality of $Q$ when $q_{i-1} < \qidown$ and $q_{i+1} \leq \qiup$. When both $q_{i-1} \geq \qidown$ and $q_{i+1} \leq \qiup$, the quantization point $q_i$ is never rounded to by ASQ on $w$, therefore quantization set $Q' := (Q \setminus\{ q_i\}) \cup \{ w_{i'} \}$, where $w_{i'} \notin Q$, has $\objt_{\calX}(w, Q') < \objt_{\calX}(w, Q)$. Thus, for the remainder of the proof, we assume that $q_{i-1} < \qidown$ and $q_{i+1} > \qiup$. 

    We construct quantization set $Q_1 := (Q \setminus \{q_i\}) \cup \{ \qidown\}$. Then, 
    \begin{align*}
        \objt_{\calX}(w, Q_1) - \objt_{\calX}(w, Q) &= \sum_{w_j \in [q_{i-1}, \qidown]} \lambda_j (w_j - q_{i-1}) \cdot \left( \qidown - q_i \right) \\
                                                    &\quad \quad \quad \quad + \sum_{w_k \in [\qidown, q_{i+1}]} \lambda_j (q_{i+1} - w_k) \cdot \left( q_i -\qidown \right) \\
        &= (q_i - \qidown) \cdot \left[ \sum_{w_k \in [\qidown, q_{i+1}]} \lambda_j (q_{i+1} - w_k) - \sum_{w_j \in [q_{i-1}, \qidown]} \lambda_j (w_j - q_{i-1}) \right] \\
        &= (q_i - \qidown) \cdot \gamma
    \end{align*}
    Analogously, we construct $Q_2 := (Q \setminus \{q_i\}) \cup \{ \qiup\}$. Then, 
    \begin{align*}
        \objt_{\calX}(w, Q_2) - \objt_{\calX}(w, Q) &= \sum_{w_j \in [q_{i-1}, \qiup]} \lambda_j (w_j - q_{i-1}) \cdot \left( \qiup - q_i \right) \\
                                                    &\quad \quad \quad \quad + \sum_{w_k \in [\qiup, q_{i+1}]} \lambda_j (q_{i+1} - w_k) \cdot \left( q_i -\qiup \right) \\
        &= (\qiup - q_i) \cdot \left[ \sum_{w_j \in [q_{i-1}, \qiup]} \lambda_j (w_j - q_{i-1}) - \sum_{w_k \in [\qiup, q_{i+1}]} \lambda_j (q_{i+1} - w_k) \right] \\
        &= - (\qiup - q_i) \cdot \gamma
    \end{align*}
    Because we assume that $\objt_{\calX}(w, Q) = \objt_{\calX}(w, s)$, it must be the case that both $\objt_{\calX}(w, Q_1) - \objt_{\calX}(w, Q) \geq 0$ and $\objt_{\calX}(w, Q_2) - \objt_{\calX}(w, Q) \geq 0$. But $(q_i - \qidown) > 0$ and $(\qiup - q_i) > 0$, so it must be the case that $\gamma =0$. Therefore, $\objt_{\calX}(w, Q_1) = \objt_{\calX}(w, Q_2) = \objt_{\calX}(w, Q)$.
    Iteratively replacing each $q_i \notin w$ from $Q$ in this way then proves the claim.\qedhere

\end{proof}

\begin{proof}[Proof of Lemma \ref{lem: C-totally-monotone}]
    We prove that $C$ satisfies the Quadrangle Inequality (Definition \ref{def: quadrangle-inequality}) Note that for $i \in [a, b]$, we have
    \begin{align*}
        (w_c - w_i)(w_i - w_a) \leq (w_f- w_i)(w_i - w_a) \tag{1}
    \end{align*}
    For $i \in [c, f]$, we have
    \begin{align*}
        (w_f - w_i)(w_i - w_b) \leq (w_f - w_i)(w_i - w_a) \tag{2}
    \end{align*}
    For $i \in [b, c]$, we have
    \begin{align*}
        (w_c - w_i)&(w_i - w_a) + (w_f - w_i)(w_i - w_b) \\
        &= (w_c - w_i)(w_i - w_b) + (w_f - w_i)(w_i - w_a) + (w_a - w_b)(w_f - w_c) \\
        &\leq (w_c - w_i)(w_i - w_b) + (w_f - w_i)(w_i - w_a) \tag{3}
    \end{align*}

    Thus, we have that
    \begin{align*}
        C[a,c] + C[b, f] &= \sum_{i \in [a,c]} \lambda_i (w_c - w_i)(w_i - w_a) + \sum_{i \in [b,d]} \lambda_i (w_f - w_i) (w_i - w_b) \\
        &= \sum_{i \in [a,b]} \lambda_i (w_c - w_i)(w_i - w_a) + \sum_{i \in [c,d]} \lambda_i (w_f - w_i) (w_i - w_b) \\
        &\quad \quad \quad \quad + \sum_{i \in [b, c]} \lambda_i [ (w_c - w_i)(w_i - w_b) + (w_f - w_i)(w_i - w_a) ] \\
        &\leq \sum_{i \in [a,b]} \lambda_i (w_f - w_i)(w_i - w_a) + \sum_{i \in [c,d]} \lambda_i (w_f - w_i) (w_i - w_a) \\
        &\quad \quad \quad \quad + \sum_{i \in [b, c]} \lambda_i \left[ (w_c - w_i)(w_i - w_b) + (w_f - w_i)(w_i - w_a) \right] \\
        &= \sum_{i \in [a,d]} \lambda_i (w_f - w_i)(w_i - w_a) + \sum_{i \in [b, c]} \lambda_i (w_c - w_i)(w_i - w_b) = C[a,d] + C[b,c].\qedhere
    \end{align*}
\end{proof}

\begin{proof}[Proof of Lemma \ref{lem: C-sorted}]
    We first show that the rows of $C$ are sorted. Consider a fixed row $i \in [d]$ and let $j, k \in [d]$ such that $i < j < k$. Then,
    \begin{align*}
        C[i, j] = \sum_{\ell=i}^{j} \lambda_{\ell} (w_j - w_{\ell}) (w_{\ell} - w_i) \leq \sum_{\ell=i}^{k} \lambda_{\ell} (w_k - w_{\ell}) (w_{\ell} - w_i) = C[i,k]
    \end{align*}
    Similarly, for a fixed column $j \in [d]$ and rows $i < k$ we have
    \begin{align*}
        C[i, j] = \sum_{\ell=i}^{j} \lambda_{\ell} (w_j - w_{\ell}) (w_{\ell} - w_i) \geq \sum_{\ell=k}^{j} \lambda_{\ell} (w_j - w_{\ell}) (w_{\ell} - w_k) = C[k, j]
    \end{align*}
    This proves that matrix $C$ has sorted rows and columns.
\end{proof}

\subsection{Deferred Table from Section \ref{sec: motivation}}


\begin{table}[H]
    \caption{Recall@100 of maximum inner product and $\ell_2$ search on vectors from GloVe 300D \cite{carlson2025newpairgloves}. The dataset consists of 900,000 randomly selected GloVe vectors, each query is sampled from 100,000 (disjoint from dataset) vectors, and each method is tested on 10,000 sampled queries. We benchmark against the FAISS implementation of Product Quantization (PQ) \cite{Matsui2018}; each method uses an average of 4-bits per coordinate.}
    \label{table: vector-search}
    \centering
    
   \begin{tabular}{llllll}
        \toprule
        \textbf{Task} & \textbf{Method} & {\textbf{Average}} & {\textbf{Worst 1\% }} & {\textbf{Worst 0.1\%}} & {\textbf{Worst}} \\
        \midrule
        Max-IP    & PQ      & 0.8535 & 0.7600 & 0.6600 & 0.4500 \\
                  & $\objt$   & 0.8751 & 0.8000 & 0.7700 & 0.7400 \\
        \addlinespace 
        $\ell_2$ & PQ      & 0.8927 & 0.7800 & 0.7200 & 0.6600 \\
                 & $\objt$ & 0.8907 & 0.7700 & 0.7200 & 0.6400 \\
        \bottomrule
    \end{tabular}

   \medskip

\end{table}


\end{document}